\documentclass[twoside,11pt]{article}

\usepackage{blindtext}

\usepackage{jmlr2e}

\usepackage{color}
\usepackage{amsmath,amssymb}
\usepackage{graphicx}

\newtheorem{assumption}{Assumption}

\usepackage{lastpage}

\allowdisplaybreaks[4]

\def\##1\#{\begin{align}#1\end{align}}
\def\$#1\${\begin{align*}#1\end{align*}}

\ShortHeadings{In-Context Learning as Nonparametric Conditional Probability Estimation}{Liu, Tan, Xie, Zeng and Zhu}
\firstpageno{1}

\begin{document}
\title{In-Context Learning as Nonparametric Conditional Probability Estimation: Risk Bounds and Optimality
}

\author{\name Chenrui Liu \email chenruiliu@mail.bnu.edu.cn \\
\addr Department of Statistics \\
Beijing Normal University at Zhuhai, Zhuhai, China
\AND
\name Falong Tan\thanks{All authors contributed equally to this work. 
Falong Tan and Lixing Zhu are co-corresponding authors 
(Emails: \texttt{falongtan@hnu.edu.cn}, \texttt{lzhu@bnu.edu.cn}).} \email falongtan@hnu.edu.cn \\
\addr Department of Statistics and Data Science \\
Hunan University, Changsha, China
\AND
\name Chuanlong Xie \email clxie@bnu.edu.cn \\
\addr Department of Statistics \\
Beijing Normal University at Zhuhai, Zhuhai, China
\AND
\name Yicheng Zeng \email zengyicheng@mail.sysu.edu.cn \\
\addr School of Science \\
Sun Yat-sen University, Shenzhen, China
\AND
\name Lixing Zhu\footnotemark[\value{footnote}] \email lzhu@bnu.edu.cn \\
\addr Department of Statistics \\
Beijing Normal University at Zhuhai, Zhuhai, China
}

\editor{My editor}

\maketitle

\begin{abstract}%
This paper investigates the expected excess risk of in-context learning (ICL) for multiclass classification. We formalize each task as a sequence of labeled examples followed by a query input; a pretrained model then estimates the query's conditional class probabilities. The expected excess risk is defined as the average truncated Kullback-Leibler (KL) divergence between the predicted and true conditional class distributions over a specified family of tasks.
We establish a new oracle inequality for this risk, based on KL divergence, in multiclass classification. This yields tight upper and lower bounds for transformer-based models, showing that the ICL estimator achieves the minimax optimal rate (up to logarithmic factors) for conditional probability estimation.
From a technical standpoint, our results introduce a novel method for controlling generalization error via uniform empirical entropy. We further demonstrate that multilayer perceptrons (MLPs) can also perform ICL and attain the same optimal rate (up to logarithmic factors) under suitable assumptions, suggesting that effective ICL need not be exclusive to transformer architectures.
\end{abstract}


\section{Introduction}
The transformer architecture~\citep{vaswani2017attention} represents a pivotal milestone in modern machine learning.
It has been extensively applied across domains including natural language processing~\citep{devlin2019bert}, computer vision~\citep{dosovitskiy2020image}, and reinforcement learning~\citep{chen2021decision}, catalyzing significant research advancements in the field of large language models (LLMs).
\citet{brown2020language} observed that, once scaled sufficiently, LLMs exhibit the remarkable capability of \emph{in-context learning} (ICL). Specifically, pretrained LLMs can perform novel tasks with just a few labeled examples in the prompt, without any parameter updates or fine-tuning. This capability enables strong few-shot generalization and has established a new paradigm for few-shot learning~\citep{brown2020language,garcia2023unreasonable}. Consequently, this paradigm accelerates the deployment of deep learning models, enabling a foundational shift toward LLMs as instruction-following systems.

A recent line of research has quantified ICL by studying transformer models pretrained on linear regression tasks with Gaussian priors~\citep{garg2022can, akyurek2022learning, li2023transformers_2, raventos2023pretraining, wu2023many}. Concurrently, a growing body of work has revealed a close connection between the forward pass of multi-layer transformer models and multi-step gradient descent algorithms~\citep{akyurek2022learning, von2023transformers, bai2023transformers, ahn2023transformers, zhang2024trained}, suggesting that transformers may implement ICL by implicitly simulating gradient descent. Recent empirical findings indicate that transformers fail to exhibit ICL unless they are pretrained on a sufficiently diverse set of tasks~\citep{raventos2023pretraining}, underscoring the crucial role of task diversity in pretraining. The pioneering experimental study by \citet{garg2022can} demonstrated that transformer models can in-context learn a variety of fundamental function classes, thereby motivating theoretical analyses of ICL in statistics. \citet{wu2023many} further quantified how the number of pretraining tasks influences the prediction error of ICL in linear regression. \citet{kim2024transformers} analyzed a model consisting of a deep neural network and a single-layer linear transformer, and established the minimax optimality of the in-context learner from the perspective of nonparametric regression.

Although substantial research has analyzed the ICL capabilities of large language models and achieved significant progress in understanding this mechanism, notable limitations and challenges remain in its theoretical analysis.

First, a significant gap exists between the assumptions of theoretical frameworks and the constraints of practical applications.
As noted by \citet{abedsoltan2024context}, there are two different aspects of generalization in ICL: {\it task-scaling}, where model performance improves as the number of pretraining tasks $T$ increases while keeping the context length $N$ fixed, and {\it context-scaling}, where model performance improves as $N$ increases while $T$ remains fixed.
Most existing studies focus on the combined regime of task-scaling and context-scaling, where both $T$ and $N$ are assumed to be large. Some analyses further consider the asymptotic limit as both $T$ and $N$ approach infinity. However, the practical importance of few-shot scenarios is neglected, as inference is typically constrained by limited context length.
For example, GPT-3 was evaluated under few-shot settings~\citep{brown2020language}, where the context length $N$ only allowed one or two labeled examples per class in the prompt.
To our knowledge, \citet{wu2023many} provide the first theoretical analysis of the effect of task-scaling on the prediction risk of ICL under a fixed context length.

Second, existing theoretical analyses are constrained by simplified task settings and model architectures. The literature predominantly focuses on linear regression or binary classification, neglecting the widespread practical use of multiclass classification. Architecturally, to maintain analytical tractability, most studies resort to simplified transformer variants, such as linear transformers or single-head attention mechanisms~\citep{zhang2024trained, huang2023context, kim2024transformers_2, li2024nonlinear}. This stands in stark contrast to empirical evidence, which consistently demonstrates the superior performance of standard multi-head dot-product attention in a wide range of ICL tasks~\citep{cui2024superiority, xing2024benefits}. Although some recent work has begun to address multi-head attention for regression problems~\citep{chen2024training, chen2024how, deora2024on} and binary classification ~\citep{li2024training, li2024one}, the theoretical understanding of multiclass classification remains underdeveloped. 
\citet{shen2024training} studied multiclass classification of Gaussian mixtures in the context of ICL, but their framework was still limited to single-layer linear transformer models.

To address these limitations,  we develop a theoretical framework for multiclass in-context learning that operates under a task-scaling regime with a deep transformer architecture. Specifically, we derive non-asymptotic error bounds on the expected excess risk under the assumption that the number of pretraining tasks $T$ is much larger than the context length $N$. We further analyze the asymptotic behavior as the number of pretraining tasks $T$ tends to infinity while the context length $N$ remains fixed. From a modeling perspective, we frame ICL as a nonparametric conditional probability estimation problem, where the model estimates conditional class probabilities of a query input given labeled prompt examples. To this end, we use a deep transformer model with multi-head dot-product attention and analyze its generalization performance in ICL.

Our contributions are summarized as follows:
\begin{enumerate}
\item Our work extends existing theoretical frameworks by relaxing restrictive assumptions that misalign with practical usage. Unlike prior analyses where the context length $N$ is assumed to be large,
we focus on few-shot settings which are typically operated in real-world deployments. We also relax other constraints common in previous studies, such as focusing on simplified transformer architectures. These generalizations preserve analytical rigor while enabling closer alignment with real-world deployments, thereby bridging critical gaps between theoretical analysis and practical implementation. 

\item We establish a new oracle inequality for the expected excess risk based on Kullback–Leibler (KL) divergence in multiclass classification. To the best of our knowledge, this is the first result that decomposes the expected excess risk under KL divergence into approximation, optimization, and generalization errors, with the generalization error controlled by {\it uniform empirical entropy}. While prior work such as \citet{schmidt2020nonparametric} or 
\citet{bos2022convergence} establishes oracle inequalities for deep neural networks via \(L^\infty\)-metric entropy over the entire covariate space.
Our oracle inequality provides a new methodology for risk analysis and can also be applied to other settings of multiclass classification. 

\item Our risk analysis finds that transformers are not necessarily the uniquely optimal architecture for ICL. By comparing the convergence rates of transformer and MLP architectures, we demonstrate that both can achieve minimax-optimal rates (up to logarithmic factors) under certain conditions. This reveals that transformers may be not the only viable architectural foundation for ICL, which is also corroborated by recent empirical studies~\citep{tong2024mlps, kratsios2025context}.
\end{enumerate}

\section{Related Work}
We provide a brief review of related work on ICL and theoretical analyses of transformer architectures in various settings.

\subsection{Theoretical Analyses of Transformers in ICL}
Recent theoretical studies have rigorously analyzed various aspects of transformer-based models in ICL. For linear function classes, the statistical complexity and representational power of transformers have been precisely characterized~\citep{zhang2024trained, wu2023many, cheng2023transformers}, with extensions to mixture models and broader function classes explored in~\cite{jin2024provable, guo2023transformers, collins2024context}. In terms of optimization, convergence properties and landscape analyses of transformer-based ICL objectives have been investigated primarily under simplified settings, focusing on linear attention and regression tasks~\citep{zhang2024trained, huang2023context, kim2024transformers, nichani2024transformers, kim2024transformers_2}. More recently, \citet{chen2024training} studied single-layer transformers with multi-head attention in nonlinear regimes, although their analysis is restricted to shallow architectures and regression. In contrast, we establish expected excess risk bounds for deep transformers with multi-head attention in the multiclass classification setting, thereby substantially extending existing theoretical results.

\subsection{Mechanistic and Algorithmic Perspectives of Transformers in ICL}
Recent work has shown that transformers can emulate standard learning algorithms such as gradient descent and its variants~\citep{akyurek2022learning, von2023transformers, bai2023transformers, fu2024transformers}. Additional studies have examined global minimizers and critical points arising in adaptive gradient descent methods within transformer architectures \citep{ahn2023transformers, mahankali2023one}. Bayesian analyses further suggest that transformers implicitly approximate Bayesian inference, offering a probabilistic explanation for ICL mechanisms \citep{xie2021explanation, muller2021transformers, zhang2022analysis, panwar2023context, jeon2024information}. Transformers’ ability to perform ICL on data generated by Markov chains has also been investigated \citep{collins2024context, edelman2024evolution, makkuva2024attention}, as has their potential for in-context decision-making within reinforcement learning frameworks \citep{lin2023transformers, sinii2023context}. Broader theoretical perspectives on the generalization behavior, algorithmic interpretations, and meta-learning properties of transformer-based models have been extensively studied by \citet{li2023transformers_2}, \citet{dai2022can}, and \citet{raventos2023pretraining}. Analyses of multi-head attention mechanisms, from both representational and Bayesian viewpoints, have further illuminated the learning dynamics in these models~\citep{an2020repulsive, mahdavi2023memorization}. A detailed discussion of these results lies beyond the scope of this paper.

\subsection{ICL Beyond Transformers}
Recent literature has begun to investigate whether architectures beyond transformers can exhibit in-context learning behavior. \citet{tong2024mlps} empirically demonstrated that multilayer perceptrons (MLPs) can match the ICL performance of transformers on synthetic tasks. Complementing this, \citet{kratsios2025context} showed that MLPs are universal approximators in an ICL setting. Our theoretical results further support these findings by demonstrating that, under appropriate conditions, MLPs attain minimax optimal rates for multiclass classification. These results provide evidence that attention-based architectures may not be strictly necessary for ICL.

\section{Preliminaries}

\paragraph{Notation} 
We write \( [n] := \{1,2,\ldots,n\} \). The set of natural numbers including zero is denoted by \( \mathbb{N} := \mathbb{Z}_{\geq 0} \). For a vector \( \mathbf{x}=(x_1,\cdots,x_p)^\top \in \mathbb{R}^p \) and a matrix \( \mathbf{X}=(X_{ij}) \in \mathbb{R}^{p \times N} \), we define the entrywise max-norms \( \|\mathbf{x}\|_\infty := \max_i |x_i| \), \( \|\mathbf{X}\|_\infty := \max_{i,j} |X_{ij}| \), the Euclidean norm \( \|\mathbf{x}\| := (\sum_i x_i^2)^{1/2} \), and the Frobenius norm \( \|\mathbf{X}\|_F := (\sum_{i,j} X_{ij}^2)^{1/2} \). The counting norm \( \|\mathbf{x}\|_0 \) denotes the number of nonzero entries in \( \mathbf{x} \), and analogously for matrices. For a vector-valued function \( \mathbf{f}=(f_1, \dots, f_K)^{\top}: \Omega \subset \mathbb{R}^d \to \mathbb{R}^K \), we define the $L^{\infty}$-norm as \( \|\mathbf{f} \|_\infty := \sup_{\mathbf{z} \in \Omega} \max_{j \in [K]} |f_j(\mathbf{z})| \), omitting \( \Omega \) when clear. For vectors \( \mathbf{x}, \mathbf{x}' \in \mathbb{R}^K \), define \( \log(\mathbf{x}/\mathbf{x}') := \log(\mathbf{x}) - \log(\mathbf{x}') \), applied coordinatewise. For any \( x \in \mathbb{R} \), \( \lfloor x \rfloor \) and \( \lceil x \rceil \) denote the floor and ceiling, respectively. A standard basis vector in \( \mathbb{R}^K \) has the form \( (0,\ldots,0,1,0,\ldots,0)^\top \). The \((K-1)\)-simplex is \( \mathcal{S}^K := \{ \mathbf{p}=(p_1, \dots, p_K)^{\top} \in \mathbb{R}^K : \sum_k p_k = 1,\; p_k \geq 0 \} \). For probability measures \( P \) and \( Q \), the Kullback--Leibler divergence is \( \mathrm{KL}(P \| Q) := \int \log(dP/dQ)\,dP \) if \( P \ll Q \), and \( \infty \) otherwise. We use the notation \(f(x) \lesssim g(x)\) to signify that \(f(x) \le C g(x)\) for all $x$ and some constant \(C > 0\). Similarly, \(f(x) \gtrsim g(x)\) signifies \(f(x) \ge C g(x)\).
For \( \mathbf{x} \in \mathbb{R}^p \), we write \( \| \mathbf{x} \|_1 = \sum_i|x_i| \).  The hardmax operator is \( \sigma_H(\mathbf{x}) := \lim_{c \to \infty} \exp(c\mathbf{x})/\|\exp(c\mathbf{x})\|_1 \), applied columnwise for matrices. The Hadamard product \( \odot \) denotes entrywise multiplication. The metric entropy of the function class $\mathcal F$ with respect to the metric $\rho$ is defined as \( \mathcal{V}(\delta, \mathcal{F}, \rho) := \log \mathcal{N}(\delta, \mathcal{F}, \rho) \), where \( \mathcal{N}(\delta, \mathcal{F}, \rho) \) is the \( \delta \)-covering number of $\mathcal F$ under the metric \( \rho \).
Throughout this paper, we use $C$ to denote a generic constant, which may differ across appearances.

\subsection{In-Context Learning}
In-context learning (ICL) refers to the phenomenon where a model performs prediction by conditioning on a sequence of input-output examples, known as a \emph{prompt}, without updating its parameters~\citep{brown2020language, garg2022can}. A prompt consists of $N$ labeled examples together with an unlabeled query input, i.e., \( (\mathbf{x}_1, \mathbf{y}_1, \ldots, \mathbf{x}_N, \mathbf{y}_N, \mathbf{x}_{N+1}) \), where each \(\mathbf{y}_i = f(\mathbf{x}_i)\) for some unknown function \(f\). The goal is to output a prediction \(\hat{\mathbf{y}}(\mathbf{x}_{N+1})\) for the query input 
\( \mathbf{x}_{N+1} \) such that \(\hat{\mathbf{y}}(\mathbf{x}_{N+1}) \approx f(\mathbf{x}_{N+1})\).

In ICL for multiclass classification, each task is specified by a joint distribution 
\(P\) over \((\mathbf{x}, \mathbf{y}) \in [0,1]^p \times \{0,1\}^K\), where \(\mathbf{y}\) 
is a one-hot label vector. The distribution \(P\) is drawn from a family of probability measures  \(\mathcal{P}\) and may vary across tasks. 
For a given task, the model observes \(N\) i.i.d.\ examples 
\((\mathbf{x}_1, \mathbf{y}_1), \ldots, (\mathbf{x}_N, \mathbf{y}_N)\) from \(P\) and 
is then required to predict the label \(\mathbf{y}_{N+1}\) for a query input 
\(\mathbf{x}_{N+1}\), where \((\mathbf{x}_{N+1}, \mathbf{y}_{N+1}) \sim P\) 
independently of the first \(N\) labeled examples.
Following \citet{bai2023transformers}, we encode the prompt 
\((\mathbf{x}_1, \mathbf{y}_1, \ldots, \mathbf{x}_N, \mathbf{y}_N, \mathbf{x}_{N+1})\) 
into a matrix \(\mathbf{Z} \in \mathbb{R}^{(p+K)\times(N+1)}\):
\begin{align*}
    \mathbf{Z} = 
    \begin{bmatrix}
        \mathbf{X} & \mathbf{x}_{N+1} \\
        \mathbf{Y} & \mathbf{0}_K
    \end{bmatrix},
\end{align*}
where \(\mathbf{X} = (\mathbf{x}_1, \ldots, \mathbf{x}_N) \in \mathbb{R}^{p \times N}\), \(\mathbf{Y} = (\mathbf{y}_1, \ldots, \mathbf{y}_N) \in \mathbb{R}^{K \times N}\), and \(\mathbf{o}_K \in \mathbb{R}^K\) is a zero vector placeholder for the label of the query input \(\mathbf{x}_{N+1}\), which the model is expected to predict based on the prompt.
We consider a data set of \(T\) tasks, denoted by \(\mathcal{D}_T = \{ (\mathbf{Z}^{(t)}, \mathbf{y}^{(t)}_{N+1}) \}_{t=1}^T\), where each prompt matrix \(\mathbf{Z}^{(t)} \in [0,1]^{(p+K) \times (N+1)}\) is flattened into a vector in \([0,1]^d\) with \(d := (p+K)(N+1)\) for notational simplicity.

\subsection{The Multiclass Classification Model}

The core objective of in-context classification is to predict the label of a query input based on a finite set of observed input-output pairs, collectively referred to as the \emph{prompt}. In the multiclass setting with \(K\) possible classes, this task can be formalized as estimating the conditional probability distribution over labels for the query, given the prompt. Estimating the full conditional distribution is essential for both uncertainty quantification and reliable prediction.

Consider task \(t\) with prompt \(\mathbf{Z}^{(t)} = (\mathbf{X}^{(t)}, \mathbf{Y}^{(t)}, \mathbf{x}_{N+1}^{(t)})\), where \((\mathbf{X}^{(t)}, \mathbf{Y}^{(t)})\) denotes the set of \(N\) in-context examples and \(\mathbf{x}_{N+1}^{(t)}\) is the query input. Let \(\mathbf{y}_{N+1}^{(t)} \in \{0,1\}^K\) be the one-hot label vector of the query, with \(y_{N+1,k}^{(t)}\) denoting its \(k\)-th component. The true conditional class probabilities are defined as
\[
p_k^0(\mathbf{Z}^{(t)}) := \mathbb{P}\left( y_{N+1,k}^{(t)} = 1 \mid \mathbf{Z}^{(t)} \right) = \mathbb{P}\left( y_{N+1,k}^{(t)} = 1 \mid \mathbf{X}^{(t)}, \mathbf{Y}^{(t)}, \mathbf{x}_{N+1}^{(t)} \right), \quad k \in [K],
\]
and the full vector of probabilities is given by
\[
\mathbf{p}_0(\mathbf{Z}^{(t)}) := \left( p_1^0(\mathbf{Z}^{(t)}), \dots, p_K^0(\mathbf{Z}^{(t)}) \right)^{\top} \in \mathcal{S}^K,
\]
where \(\mathcal{S}^K\) denotes the \(K\)-dimensional probability simplex and $\top$ denotes the transpose.

To estimate the conditional class probabilities, a standard approach in deep learning is to minimize the empirical negative log-likelihood via (stochastic) gradient descent. Let \(\mathcal{D}_T = \{ (\mathbf{Z}^{(t)}, \mathbf{y}_{N+1}^{(t)}) \}_{t=1}^{T}\) be a data set of \(T\) tasks.
Given a conditional class probability vector \(\mathbf{p}(\mathbf{Z}) = (p_1(\mathbf{Z}), \dots, p_K(\mathbf{Z}))^{\top} \), the likelihood over \(\mathcal{D}_T\) is defined as \(\mathcal{L}(\mathbf{p} \mid \mathcal{D}_T) = \prod_{t=1}^{T} \prod_{k=1}^{K} \left( p_k(\mathbf{Z}^{(t)}) \right)^{y_{N+1,k}^{(t)}}\).
Taking the negative logarithm yields the empirical negative log-likelihood (or cross-entropy loss): 
\begin{align*}
\ell(\mathbf{p}, \mathcal{D}_T) = -\frac{1}{T} \sum_{t=1}^{T} \sum_{k=1}^{K} y_{N+1,k}^{(t)} \log p_k(\mathbf{Z}^{(t)}) = -\frac{1}{T} \sum_{t=1}^{T} \mathbf{y}_{N+1}^{(t) \, \top} \log \mathbf{p}(\mathbf{Z}^{(t)}),
\end{align*}
where the logarithm is applied element-wise.
This loss measures the average discrepancy between the predicted and true class distributions across tasks. Minimizing it corresponds to empirical risk minimization under the cross-entropy loss. The resulting estimator \(\hat{\mathbf{p}} \in \mathcal{F}\) is obtained by optimizing over a function class \(\mathcal{F}\), typically parameterized by a deep neural network and constrained to produce valid probability vectors.

We aim to investigate a non-asymptotic bound of the expected excess risk of the estimator \( \widehat{\mathbf{p}} \), as measured by the expected cross-entropy loss:
\begin{eqnarray*}
R(\mathbf{p}_0, \widehat{\mathbf{p}})
&:=&\mathbb{E}_{\mathcal{D}_T, (\mathbf{Z}, \mathbf{y}_{N+1})} \left[ \mathbf{y}_{N+1}^{\top} 
    \log\frac{\mathbf{p}_0(\mathbf{Z})}{\widehat{\mathbf{p}}(\mathbf{Z})} \right] \\
&=& \mathbb{E}_{\mathcal{D}_T, \mathbf{Z}} \left[ \mathbf{p}_0(\mathbf{Z})^{\top}
    \log\frac{\mathbf{p}_0(\mathbf{Z})}{\widehat{\mathbf{p}}(\mathbf{Z})} \right] \\
&=& \mathbb{E}_{\mathcal{D}_T, \mathbf{Z}} \left[ {\rm KL} (\mathbf{p}_0(\mathbf{Z}) \| 
    \widehat{\mathbf{p}}(\mathbf{Z}) ) \right],
\end{eqnarray*}
where $(\mathbf{Z}, \mathbf{y}_{N+1}) $ is independent of the data set $\mathcal{D}_T$, \( \mathbf{p}_0(\mathbf{Z}) \) is the true conditional distribution, and \( \widehat{\mathbf{p}}(\mathbf{Z}) \) is the predicted distribution, and ${\rm KL}(\mathbf{p}_0(\mathbf{Z}) \| \widehat{\mathbf{p}}(\mathbf{Z}) )$ denotes the Kullback-Leibler (KL) divergence between the two distributions \( \mathbf{p}_0(\mathbf{Z}) \) and \( \widehat{\mathbf{p}}(\mathbf{Z}) \), conditional on $\mathbf{Z}$ and $(\mathcal{D}_T,\mathbf{Z})$, respectively.

However, the expected KL divergence may diverge when the likelihood ratio becomes unbounded (see Lemma 2.1 in \citet{bos2022convergence}). To address this limitation, following the suggestion of \citet{bos2022convergence}, we consider a truncated version of the KL divergence at level \( B > 0 \):
\begin{align*}
\mathrm{KL}_B (\mathbf{p}_0(\mathbf{Z}) \,\|\, \widehat{\mathbf{p}}(\mathbf{Z}))
= \mathbf{p}_0(\mathbf{Z})^\top \left( B \wedge \log \frac{\mathbf{p}_0(\mathbf{Z})}{\widehat{\mathbf{p}}(\mathbf{Z})} \right),
\end{align*}
where $a \wedge b = \min\{a, b\}$. 
Therefore, the expected excess risk  of the estimator \( \hat{\mathbf{p}} \) is defined as the expected truncated Kullback-Leibler divergence:
\begin{align*}
R_B(\mathbf{p}_0, \widehat{\mathbf{p}}) 
:= \mathbb{E}_{\mathcal{D}_T, \mathbf{Z}}\left[ \mathrm{KL}_B 
   (\mathbf{p}_0(\mathbf{Z}) \,\|\, \widehat{\mathbf{p}}(\mathbf{Z})) \right] 
=  \mathbb{E}_{\mathcal{D}_T, \mathbf{Z}}\left[ \mathbf{p}_0(\mathbf{Z})^\top \left( B \wedge \log \frac{\mathbf{p}_0(\mathbf{Z})}{\widehat{\mathbf{p}}(\mathbf{Z})} \right) \right].
\end{align*}
Hereafter, we refer to \( R_B(\mathbf{p}_0, \hat{\mathbf{p}}) \) as the \emph{ICL expected excess risk}, or simply the \emph{ICL risk}.

The practice of truncating the log-likelihood ratio to ensure boundedness has been well established, see, for example, \citet{wong1995probability} and \citet{bos2022convergence}. Truncation at level \( B > 0 \) regularizes the KL divergence by controlling the influence of near-zero predictions, thereby preventing instability in the risk.

Furthermore, due to the non-convexity of the function class \(\mathcal{F}\), we introduce a parameter to quantify the gap between the expected empirical risk achieved by an estimator \(\hat{\mathbf{p}}\) and the global minimum within \(\mathcal{F}\). Formally, this expected empirical risk gap is defined as:
\begin{align}\label{eq:empirical_risk_gap}
\Delta_T(\mathbf{p}_0, \widehat{\mathbf{p}}) := \mathbb{E}_{\mathcal{D}_T}\left[
	-\frac{1}{T} \sum_{t=1}^{T} \mathbf{y}_{N+1}^{(t)\,\top} \log(\widehat{\mathbf{p}}(\mathbf{Z}^{(t)}))
	- \min_{\mathbf{p} \in \mathcal{F}} \left(
	-\frac{1}{T} \sum_{t=1}^{T} \mathbf{y}_{N+1}^{(t)\,\top} \log(\mathbf{p}(\mathbf{Z}^{(t)}))
	\right)
	\right].
\end{align}
This quantity measures the expected difference between the empirical risk of \(\hat{\mathbf{p}}\) and the optimal empirical risk achievable within the class \(\mathcal{F}\).

\subsection{Model Classes: Transformers and MLPs}
We begin by defining the model classes for our estimator, employing a transformer encoder for prompt representation and a subsequent MLP for processing the encoded output. 
This hybrid architecture--utilizing a transformer as a feature extractor with an MLP classifier atop--constitutes a prevalent design paradigm in practice.  Transformers excel at capturing complex contextual dependencies through self-attention mechanisms, generating enriched latent representations of sequential data. The subsequent MLP leverages these high-level features to perform efficient classification or regression, benefiting from the transformer's representational capacity while preserving computational tractability. Empirical studies across NLP benchmarks (e.g., BERT+MLP for GLUE tasks) demonstrate that such configurations offer an optimal balance: transformers extract transferable linguistic features, while lightweight MLP heads enable rapid task-specific adaptation without exhaustive retraining. The decoupled design further facilitates modular optimization, allowing domain-specific fine-tuning of the transformer alongside dedicated calibration of the classifier.

\paragraph{Transformer architecture.} 
Adapting the architecture proposed in previous theoretical work on transformer-based models~\citep{gurevych2022rate, bai2023transformers, jiao2024convergence}, we adopt a transformer encoder to process the input prompt.

Let $\mathbf{Z} \in \mathbb{R}^{D \times (N+1)}$ denote the input prompt, where $D = p + K$ is the feature dimension and $N$ is the in-context length.
A transformer network $\boldsymbol{\phi}$ maps $\mathbf{Z}$ to an output of the same dimension through a sequence of $M$ transformer blocks, each comprising a self-attention (SA) and a feed-forward (FF) sublayer. The architecture is defined as: 
\begin{align}\label{eq:transformer_architecture}
	\boldsymbol{\phi} = E_{\mathrm{out}} \circ F_M^{(\mathrm{FF})} \circ F_M^{(\mathrm{SA})} \circ \dots \circ F_1^{(\mathrm{FF})} \circ F_1^{(\mathrm{SA})}(\mathbf{Z}),
\end{align}
where \(E_{\mathrm{out}}\) denotes the final output projection.  

Each self-attention (SA) sublayer applies a multi-head attention mechanism with a residual connection:
\[
F^{(\mathrm{SA})}(\mathbf{Z}) = \mathbf{Z} + \sum_{s=1}^{H} \mathbf{W}_{O,s} \left( \mathbf{W}_{V,s} \mathbf{Z} \right) \left( \left( \mathbf{W}_{K,s} \mathbf{Z} \right)^\top \left( \mathbf{W}_{Q,s} \mathbf{Z} \right) \right) \odot \sigma_H \left( \left( \mathbf{W}_{K,s} \mathbf{Z} \right)^\top \left( \mathbf{W}_{Q,s} \mathbf{Z} \right) \right),
\]
where \(\sigma_H\) denotes the hardmax activation, and each attention head is parameterized by learnable projection matrices \(\mathbf{W}_{Q,s}, \mathbf{W}_{K,s}, \mathbf{W}_{V,s}, \mathbf{W}_{O,s} \in \mathbb{R}^{D \times D}\).  
Each feed-forward (FF) sublayer is a position-wise two-layer network with ReLU activation and residual connection:
\[
F^{(\mathrm{FF})}(\mathbf{Z}) = \mathbf{Z} + \mathbf{W}_{F_2} \, \sigma\left( \mathbf{W}_{F_1} \mathbf{Z} \right),
\]
where \(\mathbf{W}_{F_1} \in \mathbb{R}^{D' \times D}\), \(\mathbf{W}_{F_2} \in \mathbb{R}^{D \times D'}\), and \(\sigma(\cdot)\) is the ReLU activation function.

To generate a prediction for the query label \(\mathbf{y}_{N+1}\), we apply an output projection to the final representation of the \((N{+}1)\)-th token:
\[
E_{\mathrm{out}}(\mathbf{Z}) = \mathbf{A}_{\mathrm{out}} \, \mathbf{z}_{N+1} + \mathbf{b}_{\mathrm{out}},
\]
where \(\mathbf{z}_{N+1} \in \mathbb{R}^D\) is the last column of \(\mathbf{Z}\), and \(\mathbf{A}_{\mathrm{out}} \in \mathbb{R}^{K \times D}\), \(\mathbf{b}_{\mathrm{out}} \in \mathbb{R}^K\) are learnable parameters.

Based on the above structure, we define a class of transformer networks with bounded output, Lipschitz continuity, and sparsity constraints:
\begin{align}\label{eq:class_of_transformer}
		\mathcal{T}(M, H, D, W_1, B_1, S_1, \gamma)
		=& \big\{\boldsymbol{\phi}: \mathbb{R}^d \to \mathbb{R}^K \mid
		\boldsymbol{\phi} \text{ follows \eqref{eq:transformer_architecture}}, \nonumber \\
		& \ \sup_{\mathbf{Z}} \|\boldsymbol{\phi}(\mathbf{Z})\| \leq B_1, \nonumber \\
		& \ \|\boldsymbol{\phi}(\mathbf{Z}_1) - \boldsymbol{\phi}(\mathbf{Z}_2)\| \leq \gamma \|\mathbf{Z}_1 - \mathbf{Z}_2\|_F, \; \forall \ \mathbf{Z}_1, \mathbf{Z}_2 \in [0,1]^d, \nonumber \\
		& \ \sum_{r=1}^{M} \sum_{s=1}^{H}
		\left(\|\mathbf{W}_{Q,r,s}\|_0 + \|\mathbf{W}_{K,r,s}\|_0 + \|\mathbf{W}_{V,r,s}\|_0 + \|\mathbf{W}_{O,r,s}\|_0\right) \nonumber \\
		& \ + \sum_{r=1}^{M}\left(\|\mathbf{W}_{r,F_1}\|_0 + \|\mathbf{W}_{r,F_2}\|_0\right)
		+ \|\mathbf{A}_{\mathrm{out}}\|_0 + \|\mathbf{b}_{\mathrm{out}}\|_0 \leq S_1, \nonumber \\
		& \ \max(D, D'_r) \leq W_1, \ \forall \ r=1, \dots, M
		\big\},
\end{align}
where \( d = (p+K)(N+1) \), $D = p + K$,  and \(\|\cdot\|_0\) counts the number of nonzero elements.

\paragraph{MLP architecture.} 
The class \(\mathcal{F}(L, W)\) consists of depth-\(L\) ReLU networks with maximum width \(W\):
\begin{align}\label{eq:MLP_definition}
	\mathcal{F}(L, W) := \left\{ \mathbf{f}: \mathbb{R}^{K} \to \mathbb{R}^{K} \;\middle|\; \mathbf{f}(\mathbf{x}) = \mathbf{W}_L \sigma_{\boldsymbol{\nu}_L} \cdots \mathbf{W}_1 \sigma_{\boldsymbol{\nu}_1} \mathbf{W}_0 \mathbf{x} \right\},
\end{align}
where \(W = \max_j K_j\), \(\mathbf{W}_j \in \mathbb{R}^{K_{j+1} \times K_j}\), \(\boldsymbol{\nu}_j \in \mathbb{R}^{K_j}\), and \(\sigma_{\boldsymbol{\nu}}(\mathbf{x}) := \max(\mathbf{x} - \boldsymbol{\nu}, 0)\) denotes the biased ReLU activation. 

To control the capacity of the MLP class, we define a constrained subclass \(\mathcal{F}_{\mathrm{id}}(L, W, S)\) with bounded parameters and sparsity:
\begin{align}\label{eq:MLP_class}
\mathcal{F}_{\mathrm{id}}(L, W, S) := \bigg\{ \mathbf{f} \in \mathcal{F}(L, W)\Big|\max_{j \in [L]} \left( \|\mathbf{W}_j\|_{\infty} \vee \|\boldsymbol{\nu}_j\|_{\infty} \right) \leq 1, \sum_{j=0}^{L} \left( \|\mathbf{W}_j\|_0 + \|\boldsymbol{\nu}_j\|_0 \right) \leq S \bigg\},
\end{align}
where \(\|\cdot\|_{\infty}\) is the entrywise max-norm.

\begin{assumption}\label{assump:lipschitz_mlp}
The MLP class $\mathcal{F}_{\mathrm{id}}(L, W, S)$ is assumed to be $L_1$-Lipschitz with respect to the entrywise max-norm $\|\cdot\|_\infty$, that is,
\[
\|\mathbf{f}(\mathbf{y}_1) - \mathbf{f}(\mathbf{y}_2)\|_\infty \leq L_1 \|\mathbf{y}_1 - \mathbf{y}_2\|_\infty, 
\quad \forall\, \mathbf{f} \in \mathcal{F}_{\mathrm{id}}(L, W, S),\ \forall\, \mathbf{y}_1, \mathbf{y}_2 \in \mathbb{R}^K.
\]
\end{assumption}

This assumption simplifies our theoretical analysis. Although the Lipschitz constant $L_1$ of an MLP can be bounded in terms of its depth, width, and weight norms, computing its exact value is NP-hard~\citep{virmaux2018lipschitz}. Moreover, obtaining tight upper bounds remains challenging~\citep{fazlyab2019efficient, latorre2020lipschitz, kim2021lipschitz}. Therefore, we do not attempt to estimate it explicitly and instead assume a uniform bound \(L_1\).

\subsection{Function Class for Conditional Probability Estimation}
We now define the full function class used to estimate conditional class probabilities from a given prompt. Each function integrates a transformer encoder for contextual representation learning, followed by an MLP that maps the latent representation to class probabilities via a final softmax layer:
\begin{eqnarray}\label{eq:function_class}
	&& \mathcal{F}(M, H, D, W_1, B_1, S_1, \gamma, L, W, S) \nonumber \\
	&&:= \left\{ \mathbf{p}(\mathbf{Z}) = \boldsymbol{\Phi} \circ \mathbf{f} \circ \boldsymbol{\phi}(\mathbf{Z}) \;\middle|\; 
	\mathbf{f} \in \mathcal{F}_{\mathrm{id}}(L, W, S),\ 
	\boldsymbol{\phi} \in \mathcal{T}(M, H, D, W_1, B_1, S_1, \gamma) 
	\right\},
\end{eqnarray}
where \(\boldsymbol{\Phi}(\cdot)\) is the softmax function. 

\section{The Risk Bounds for In-Context Learning}

In this section, we establish risk bounds for ICL in the multiclass classification setting. We first introduce some key definitions and assumptions used in our analysis.

\begin{definition}[Small Value Bound]
Let $\alpha\ge 0$ and $\mathcal H$ be a function class. A function $\mathbf p=(p_1,\dots,p_K)\in \mathcal H$ is said to be  \(\alpha\)-\emph{small value bounded} (or \(\alpha\)-SVB) if there exists some constant \(C > 0\) such that for any \(k \in [K]\) and \(t \in (0,1]\), the following holds: 
\begin{align*}
\mathbb{P} \left\{ p_k(\mathbf{Z}) \leq t \right\} \leq C t^\alpha.
\end{align*}
Moreover, if all $\mathbf p\in \mathcal H$ are $\alpha$-SVB, we say that $\mathcal H$ is $\alpha$-SVB. 
\end{definition}

The concept of $\alpha$-SVB generalizes the standard assumption in nonparametric classification that true conditional class probabilities are bounded away from zero, allowing them to approach zero. This extension, which follows the conditions outlined in \cite{bos2022convergence}, tightens error bounds and accelerates convergence. The index \(\alpha\) will appear in the convergence rate analysis.

\begin{definition}[Hölder Smoothness]
Let \( \beta > 0 \), \( Q > 0 \), and \( \Omega \subset \mathbb{R}^d \). The ball of \( \beta \)-Hölder continuous functions with radius \( Q \), denoted by \( C^\beta(\Omega, Q) \), is defined as
\begin{align*}
C^\beta(\Omega, Q) := \left\{ f : \Omega \to \mathbb{R}: \sum_{\|\gamma\|_1 < \beta} \|\partial^\gamma f\|_\infty + \sum_{\|\gamma\|_1 = \lfloor \beta \rfloor} \sup_{\substack{\mathbf{x}, \mathbf{y} \in \Omega \\ \mathbf{x} \neq \mathbf{y}}}
\frac{|\partial^\gamma f(\mathbf{x}) - \partial^\gamma f(\mathbf{y})|}{\|\mathbf{x} - \mathbf{y}\|_\infty^{\beta - \lfloor \beta \rfloor}} \leq Q \right\},
\end{align*}
where \( \gamma = (\gamma_1, \dots, \gamma_d) \in \mathbb{N}^d \) is a multi-index and \( \partial^\gamma := \partial^{\gamma_1}_{x_1} \cdots \partial^{\gamma_d}_{x_d} \).
\end{definition}

\begin{definition}[Hölder-Smooth Class of Probability]
For any \( \beta > 0 \) and \( Q > 0 \), the $\beta$-Hölder-smooth class of the conditional probability functions is defined as
\begin{align}\label{Holder_class}
\mathcal{G}(\beta, Q) := \left\{ \mathbf{p}_0 = (p_1^0, \dots, p_K^0)^{\top} : [0,1]^d \to \mathcal{S}^K, \; p_k^0 \in C^\beta([0,1]^d, Q) \text{ for all } k \in [K] \right\},
\end{align}
where \( \mathcal{S}^K \subset \mathbb{R}^K \) denotes the probability simplex.
Furthermore, we write  $\mathcal{G}_\alpha(\beta, Q)$ for all functions in $\mathcal{G}(\beta, Q)$ that are $\alpha$-SVB.
\end{definition}

For a class of conditional class probabilities \(\mathcal{F}\), we define its coordinate-wise logarithmic transform as
\begin{align*}
\log(\mathcal{F}) := \left\{ \log \mathbf{p} = (\log p_1, \dots, \log p_K)^{\top} : \mathbf{p} \in \mathcal{F} \right\}.
\end{align*}

\begin{definition}[$\delta$-Covering Number and Metric Entropy]
Let \((\mathcal{V}, d)\) be a pseudometric space with metric \(d(\cdot, \cdot)\) and \(\mathcal{F} \) be a subset of \( \mathcal{V} \). A finite set \( \mathcal{F}(\delta) \) is called a \(\delta\)-cover (or \(\delta\)-net) of \( \mathcal{F} \) if, for all \(\mathbf{f} \in \mathcal{F}\), there exists \(\mathbf{g} \in \mathcal{F}(\delta)\) such that $d(\mathbf{f}, \mathbf{g}) \leq \delta$. Moreover, the \(\delta\)-covering number of \(\mathcal{F}\) with respect to the metric \(d\) is the smallest cardinality  of such a \(\delta\)-cover \(\mathcal{F}(\delta)\), denoted by \( \mathcal{N}(\delta, \mathcal{F}, d) \). The quantity \(\log \mathcal{N}(\delta, \mathcal{F}, d)\) is referred to as the metric entropy of \(\mathcal{F}\) at scale \(\delta\) with respect to the metric \(d\).
\end{definition}

\begin{definition}[Uniform Empirical Covering Number and Entropy]
For any given \( \mathcal{Z}_T = \{ \mathbf{Z}^{(1)}, \ldots, \mathbf{Z}^{(T)} \} \subset \mathbb{R}^d \), we define the pseudometric \( d(\mathbf{f}, \mathbf{g}) = \|\mathbf{f}-\mathbf{g}\|_{\infty, \mathcal{Z}_T}\) as
\begin{align*}
d(\mathbf{f}, \mathbf{g}) = \|\mathbf{f}-\mathbf{g}\|_{\infty, \mathcal{Z}_T} := \max_{1 \le t \le T} \max_{1 \le k \le K} \left| \mathbf{f}_k(\mathbf{Z}^{(t)}) - \mathbf{g}_k(\mathbf{Z}^{(t)}) \right|.
\end{align*}
The corresponding metric covering number \(\mathcal{N}(\delta, \mathcal{F}, \|\cdot\|_{\infty, \mathcal{Z}_T})\) is referred as the empirical covering number with respect to $\|\cdot\|_{\infty, \mathcal{Z}_T}$. 
For a fixed \(T\), the largest covering number over all \(\mathcal{Z}_T \subset \mathbb{R}^d \) is referred to as the uniform empirical covering number:
\begin{align*}
\mathcal{N}_{\infty}(\delta, \mathcal{F}, T) := \sup_{\mathcal{Z}_T } \mathcal{N}(\delta, \mathcal{F}, \|\cdot\|_{\infty, \mathcal{Z}_T}).
\end{align*}
The uniform empirical entropy is defined as 
\(\mathcal{V}_{\infty}(\delta, \mathcal{F}, T) := 
\log \mathcal{N}_{\infty}(\delta, \mathcal{F}, T)\).
\end{definition}

\begin{assumption}\label{assump:prob-lower-bound}
Let \(\mathcal{F}\) be a class of conditional class probability functions \(\mathbf{p} = (p_1, \dots, p_K)^{\top} \) used to estimate conditional class probabilities.
Suppose that for any \(\mathbf{p} \in \mathcal{F}\), there exist two constants \(C_1 > 0\) and \(0 < C_2 < 1\) such that \(p_k(\mathbf{Z}) \geq C_1 / T^{C_2}\) for all \(k \in [K]\) and \(\mathbf{Z} \in [0,1]^d\).
\end{assumption}

Assumption~\ref{assump:prob-lower-bound} requires all estimated probabilities to be bounded away from zero. This avoids the singularity of the KL divergence and enables the use of empirical covering numbers to control the complexity of \(\log(\mathcal{F}) \) in the oracle inequality.
For the neural network class $\mathcal{F}(M, H, D, W_1, B_1, S_1, \gamma, L, W, S)$ defined in~(\ref{eq:function_class}), this can be achieved by adding additional layers to the MLP architecture of the network. More details on this issue are provided in Remarks~\ref{remark-lowbound-F} and \ref{lowbound-F}.

\begin{remark}
Assumption~\ref{assump:prob-lower-bound} and the truncation level \(B\) in the expected excess risk \(R_B(\mathbf{p}_0, \hat{\mathbf{p}})\) serve complementary roles in ensuring theoretical stability. The lower bound condition prevents the KL loss from becoming singular due to vanishing predicted probabilities. However, since the bound \(C_1 / T^{C_2}\) decreases as the number of tasks \(T\) increases, the predicted probabilities may still become arbitrarily small as \(T\) grows, causing the untruncated risk \(R(\mathbf{p}_0, \hat{\mathbf{p}})\) to diverge. Truncation at level \(B\) is therefore essential to cap the influence of extreme log-ratio values and to ensure meaningful statistical guarantees. 
\end{remark}

\begin{theorem}[Oracle Inequality]\label{thm:Oracle}
Let \(\mathcal{F}\) be any class of conditional class probability functions that satisfies Assumption~\ref{assump:prob-lower-bound}, and let \(\widehat{\mathbf{p}} \in \mathcal{F}\) be an estimator of \( \mathbf{p}_0 \in \mathcal{S}^K \). For any truncation level \(B \geq 2\) and task number $T$, the expected excess risk satisfies
\begin{align}\label{oracle_inequality}
R_B(\mathbf{p}_0, \widehat{\mathbf{p}}) \lesssim \inf_{\mathbf{p} \in \mathcal{F}} R(\mathbf{p}_0, \mathbf{p}) + \Delta_T(\mathbf{p}_0, \widehat{\mathbf{p}}) + \frac{B^2 \log(\mathcal{N}_T)\, \log^2 T}{T},
\end{align}
where \(\mathcal{N}_T = \mathcal{N}_{\infty} \left( 1/T, \log(\mathcal{F}), T \right)\), and \(\Delta_T(\mathbf{p}_0, \widehat{\mathbf{p}})\) is the expected empirical risk gap defined in~\eqref{eq:empirical_risk_gap}.
\end{theorem}

The proof of Theorem~\ref{thm:Oracle} is provided in Appendix~\ref{app:proof_oracle}. It provides a general oracle inequality for any function class \(\mathcal{F}\) satisfying Assumption~\ref{assump:prob-lower-bound}. This result is independent of the specific architecture used to model the conditional class probabilities.

\paragraph{Error Decomposition.} 
The convergence rate of the expected excess risk $R_B(\mathbf{p}_0, \widehat{\mathbf{p}})$ admits a standard decomposition into approximation, optimization, and generalization components; see \citet{van1996weak, wainwright2019high, bartlett2021deep}. This decomposition captures the trade-off among function class expressiveness, empirical optimization, and statistical complexity. Theorem~\ref{thm:Oracle} formalizes this perspective, bounding the truncated risk by the approximation error, the expected empirical risk gap, and a complexity term governed by the uniform empirical entropy of \(\log(\mathcal{F})\). 
While previous work such as \citet{schmidt2020nonparametric} and 
\citet{bos2022convergence} establishes oracle inequalities for deep ReLU 
networks in regression and multiclass classification using \(L^\infty\)-metric entropy over the entire covariate space, our analysis differs in a crucial aspect. To the best of our knowledge, this is the first work to derive a risk bound under KL divergence based on uniform empirical entropy $\log(\mathcal{N}_T)$. 
This data-dependent complexity measure avoids uniform control over the 
entire covariate space, leads to sharp risk bounds, and provides a 
principled analytical foundation for future developments in statistical
learning theory.

Since we focus on multiclass classification with neural networks, the function class \(\mathcal{F}\) is inherently nonconvex. As a result, empirical risk minimization may not achieve the global minimum, and the learned predictor \(\widehat{\mathbf{p}}\) is generally suboptimal. 
To address this, we introduce the expected empirical risk gap \(\Delta_T(\mathbf{p}_0, \widehat{\mathbf{p}})\) in~\eqref{eq:empirical_risk_gap}, 
and treat it as a fixed quantity in our analysis. 
Therefore, to derive the ICL risk, we focus on the remaining two components: the approximation error $\inf_{\mathbf{p} \in \mathcal{F}} R(\mathbf{p}_0, \mathbf{p})$ and the uniform empirical entropy $\log(\mathcal{N}_T)$ in the oracle inequality~\eqref{oracle_inequality}. 
We first establish the approximation capacity of transformer networks for functions in the Hölder class \(\mathcal{G}(\beta, Q)\) in the following lemma. 
For simplicity of notation, we henceforth denote the neural network classes $\mathcal{T}(M, H, D, W_1, B_1, S_1, \gamma)$ defined in \eqref{eq:class_of_transformer}, $\mathcal{F}_{\mathrm{id}}(L, W, S)$ defined in \eqref{eq:MLP_class}, and $ \mathcal{F}(M, H, D, W_1, B_1, S_1, \gamma, L, W, S) $ defined in~\eqref{eq:function_class} as $\mathcal{T}$, $\mathcal{F}_{\mathrm{id}}$, and $\mathcal{F}$, respectively.

\begin{lemma}\label{lem:transformer_approximation}
Let \(0 < \varepsilon < 1\), \(\beta > 0\), and \( Q >0 \). For every function \( \mathbf{p}_0 \in \mathcal{G}(\beta, Q) \), there exists a transformer network \(\boldsymbol{\phi} \in \mathcal{T} \) with
\begin{gather*}
M \lesssim \log\left(\frac{Q}{\varepsilon}\right), \quad
H \lesssim \left(\frac{Q}{\varepsilon}\right)^{d/\beta}, \quad
D \lesssim 1, \\
W_1 \lesssim \left(\frac{Q}{\varepsilon}\right)^{d/\beta}, \quad
B_1 \lesssim \|\mathbf{p}_0\|_\infty, \quad
S_1 \lesssim \left(\frac{Q}{\varepsilon}\right)^{d/\beta} \log\left(\frac{Q}{\varepsilon}\right),
\end{gather*}
such that $\| \boldsymbol{\phi} - \mathbf{p}_0 \|_\infty \leq \varepsilon$. 
Furthermore, if \(\beta > 1\), we may choose \(\gamma \lesssim Q\).
\end{lemma}

Lemma~\ref{lem:transformer_approximation} indicates that transformer networks can uniformly approximate any function in the Hölder class \(\mathcal{G}(\beta, Q)\) with explicit control over architectural complexity. Similar approximation results have been established in related settings, such as Theorem~8 in \citet{jiao2024convergence}. 
Its proof builds on the observation by \citet{gurevych2022rate} that transformer architectures, similar to deep ReLU networks, are capable of representing multiplication operations and therefore can approximate polynomials. Since polynomials can approximate continuous functions with an arbitrarily accuracy, transformer networks can effectively approximate Hölder-smooth targets. Without loss of generality, we assume the output of the approximating transformer lies in \([0,1]\); otherwise, a projection layer can be applied without increasing the error, as the target functions \(p_k^0\) are already supported on this interval. 

Next, we apply Lemma~\ref{lem:transformer_approximation} to approximate the conditional class probabilities in multiclass classification.  
Specifically, we consider the neural network class \(\mathcal{F} \) defined in~\eqref{eq:function_class} to approximate the true conditional class probabilities \(\mathbf{p}_0 \in \mathcal{G}(\beta, Q)\) in the \( \|\cdot\|_{\infty} \)-norm, while ensuring that the output probabilities are uniformly bounded away from zero. This is achieved by composing a transformer encoder with a MLP, followed by a softmax transformation. The detailed construction of the approximating network in $\mathcal{F}$ is provided in the proof of Theorem~\ref{thm:condition_approximation} in Appendix~\ref{app:proof_condition_approx}.

This transformer-MLP hybrid architecture is also motivated by its practical effectiveness. Transformers provide expressive, context-aware representations through self-attention mechanisms, while the MLP acts as an adaptive task-specific output head. Such modular designs are widely used in contemporary deep learning models and have demonstrated strong generalization performance across diverse applications. 

\begin{theorem}\label{thm:condition_approximation}
Fix \(\beta > 0\) and \( Q >0 \). For any \( \mathbf{p}_0 \in \mathcal{G}(\beta, Q) \) and \( \varepsilon \in (0,1) \), there exists a neural network \( \mathbf{\tilde{q}} = (\tilde q_1,\cdots,\tilde q_K)\in \mathcal{F} \) with
\begin{gather*}
M \lesssim \log\left(\frac{Q}{\varepsilon}\right), \quad
H \lesssim \left(\frac{Q}{\varepsilon}\right)^{d/\beta}, \quad
D \lesssim 1, \\
W_1 \lesssim \left(\frac{Q}{\varepsilon}\right)^{d/\beta}, \quad
B_1 \lesssim \|\mathbf{p}\|_\infty, \quad
S_1 \lesssim \left(\frac{Q}{\varepsilon}\right)^{d/\beta}\log\left(\frac{Q}{\varepsilon}\right), \\
L \lesssim \log\left(\frac{1}{\varepsilon}\right), \quad
W \lesssim K\left(\frac{1}{\varepsilon}\right)^{1/\beta},\quad
S \lesssim K\left(\frac{1}{\varepsilon}\right)^{1/\beta} \log\left(\frac{1}{\varepsilon}\right), 
\end{gather*}
such that 
\begin{align*}
\|{\tilde{q}}_k - {p}^0_k\|_\infty \leq  2\varepsilon(K + 1), \quad
{\tilde{q}}_k(\mathbf{Z}) \geq \frac{\varepsilon}{2K} 
\end{align*}
hold for all $k \in [K]$ and $\mathbf{Z} \in [0,1]^d$.   
\end{theorem}

Theorem~\ref{thm:condition_approximation} shows that the conditional probability \(\mathbf{p}_0 \in \mathcal{G}(\beta, Q) \) can be uniformly approximated, with arbitrarily small error, by a neural network of \(\mathcal{F}\). 
Its proof is provided in Appendix~\ref{app:proof_condition_approx}. Based on this result, we derive a non-asymptotic upper bound for the approximation error 
\( \inf_{\mathbf{p} \in \mathcal{F}} R(\mathbf{p}_0, \mathbf{p})\).

\begin{theorem}\label{thm:inf_Risk}
Assume $\mathbf p_0\in \mathcal G_\alpha(\beta,Q)$ for some $\alpha\in [0,1]$.  
Let \(\mathbf{p} = (p_1, \ldots, p_K)^{\top}: [0,1]^d \to \mathcal{S}^K\) be an approximation of $\mathbf {p}_0$ such that 
\begin{align*}
\|\mathbf{p} - \mathbf{p}_0\|_{\infty} \leq 2(K+1)\varepsilon, \quad 
\min_{1 \le k \le K} \inf_{\mathbf{Z} \in [0,1]^d} p_k(\mathbf{Z}) \geq \frac{\varepsilon}{2K}.
\end{align*}
Then we have 
\begin{align*}
R(\mathbf{p}_0, \mathbf{p}) \leq C \varepsilon^{1+ \alpha} K^{4+\alpha} \left(2 + \frac{I(\alpha<1)} {1-\alpha} +\log\left(\frac{1}{\varepsilon} \right)\right).
\end{align*}
If we choose 
$ \varepsilon = C^{-1} T^{-\frac{\beta}{(1+\alpha)\beta + d}} K^{-\frac{(3+\alpha)\beta}{(1+\alpha)\beta + d}} $ for some constant $C > 0$, then 
\begin{align*}
\inf_{\mathbf{p} \in \mathcal{F}} R(\mathbf{p}_0, \mathbf{p}) \lesssim K^{4+\alpha} 
T^{-\frac{(1+\alpha)\beta}{(1+\alpha)\beta + d}} \log(KT).
\end{align*}
\end{theorem}

Theorem~\ref{thm:inf_Risk} establishes a near-optimal convergence rate for in-context classification, recovering the classical nonparametric minimax bound when $\alpha = 1$. The proof leverage $f$-divergence theory, bounding the KL divergence via the $\chi^2$-divergence~\citep{nguyen2009surrogate, cai2022distances}. More details about the proof can be found in Appendix~\ref{app:proof_inf_Risk}.

\begin{remark}\label{remark-lowbound-F}
According to the proof of Theorem~\ref{thm:condition_approximation} and Remark~\ref{lowbound-F} in Appendix~\ref{app:proof_condition_approx}, we can add additional ReLU layers to the MLP architecture of the network in $\mathcal{F}$ to ensure that $ p_k(\mathbf{Z}) \geq \varepsilon/(2K) $ for any network $\mathbf{p} = (p_1, \dots, p_K)^{\top} \in \mathcal{F}$ and $\varepsilon \in (0,1)$.
Consequently, without loss of generality, we require that every network in $\mathcal{F}$ has a uniform, strictly positive lower bound $\varepsilon/(2K)$ across all its coordinates. 
Furthermore, if we choose 
$ \varepsilon = C^{-1} T^{-\frac{\beta}{(1+\alpha)\beta + d}} K^{-\frac{(3+\alpha)\beta}{(1+\alpha)\beta + d}} $, 
it is readily seen that the function class $\mathcal{F}$ satisfies Assumption~\ref{assump:prob-lower-bound}. This allows us to apply the oracle inequality (Theorem~\ref{thm:Oracle}) to derive the expected excess risk for transformer-based classifiers.
\end{remark}

Next, we deal with the uniform empirical entropy $\log(\mathcal{N}_T)$ in the oracle inequality.
Recall that 
\begin{align*}
\log(\mathcal{N}_T)  
= \mathcal{V}_{\infty} \Big(\frac{1}{T}, \log(\mathcal{F}), T \Big) 
= \sup_{\mathcal{Z}_T } \log \mathcal{N} \Big(\frac{1}{T}, \log(\mathcal{F}), \|\cdot\|_{\infty, \mathcal{Z}_T} \Big),
\end{align*}
where \(\mathcal{Z}_T = \{ \mathbf{Z}^{(1)}, \ldots, \mathbf{Z}^{(T)} \} \subset \mathbb{R}^d \). The following lemma provides an upper bound on the uniform empirical entropy $\log(\mathcal{N}_T) $ of the function class \(\log(\mathcal{F})\).

\begin{lemma}\label{Covering_number}
For any \(\delta > 0\), the following inequality holds:
\begin{align*}
\mathcal{V}_\infty\left(\delta, \log(\mathcal{F}), T \right)
\leq 
\mathcal{V}_\infty\left( \tfrac{\delta}{4K L_1},\, \mathcal{T},\, T \right)
+ \mathcal{V}\left( \tfrac{\delta}{4K},\, \mathcal{F}_{\mathrm{id}},\, \| \cdot \|_{\infty} \right),
\end{align*}
where $ \mathcal{V}_\infty\left( \delta,\, \mathcal{T},\, T \right) = \sup_{\mathcal{Z}_T} \log \mathcal{N}(\delta, \mathcal{T}, \|\cdot\|_{\infty, \mathcal{Z}_T}) $, \( \mathcal{V}(\delta, \mathcal{F}_{\rm id}, \|\cdot\|_{\infty}) = \log \mathcal{N}(\delta, \mathcal{F}_{\rm id}, \|\cdot\|_{\infty}) \), and \(L_1\) is the Lipschitz constant specified in Assumption~\ref{assump:lipschitz_mlp}.
\end{lemma}

The proof of Lemma~\ref{Covering_number} is provided in Appendix~\ref{app:proof_covering_number}. The argument decomposes \(\log(\mathcal{F}) \) into a composition of the transformer class \(\mathcal{T}\) and the MLP class \(\mathcal{F}_{\mathrm{id}}\), and then applies standard covering number bounds for Lipschitz compositions to control the resulting entropy. 
The next result provides an upper bound on the uniform empirical entropy of the transformer class.

\begin{lemma}\label{Covering_number_of_Transformer}
For any $\delta_1>0$, the uniform empirical entropy of the transformer class \(\mathcal{T}\) satisfies
\begin{align*}
\mathcal{V}_\infty\left(\delta_1, \mathcal{T}, T \right) 
\lesssim K M^2 S_1 \log\left(\max\{M, H, W_1\}\right) \log\left(\frac{T B_1}{\delta_1}\right), 
\end{align*}
where \(\mathcal{T} = \mathcal{T}(M, H, D, W_1, B_1, S_1, \gamma)\) is given in~\eqref{eq:class_of_transformer}.
\end{lemma} 

The proof of Lemma~\ref{Covering_number_of_Transformer} builds on classical techniques from the neural network literature~\citep{anthony2009neural,bartlett2019nearly}, and it is closely related to recent analyses of transformer architectures~\citep{gurevych2022rate,jiao2024convergence}. 
A complete proof is provided in Appendix~\ref{app:proof_transformer_covering}.

\begin{lemma}\label{Covering_number_of_MLP}
For any $\delta_2>0$, the metric entropy of the function class \( \mathcal{F}_{\mathrm{id}} \) satisfies the following upper bound
\begin{align*}
\mathcal{V}\left(\delta_2,\, \mathcal{F}_{\mathrm{id}},\, \|\cdot\|_{\infty}\right) 
\leq 2 SL \log\left( \frac{L(W+1)}{3\delta_2} \right),
\end{align*}
where $ \mathcal{F}_{\mathrm{id}} =  \mathcal{F}_{\mathrm{id}}(L, W, S)$ is defined in (\ref{eq:MLP_class}).
\end{lemma}

Lemma~\ref{Covering_number_of_MLP} gives an upper bound on the metric entropy of the class \(\mathcal{F}_{\mathrm{id}}(L, W, S)\), which consists of ReLU networks with depth \(L\), width \(W\), and size \(S\). 
Its proof is provided in Appendix~\ref{app:proof_mlp_covering}. 
Combining the oracle inequality (Theorem~\ref{thm:Oracle}), Theorem~\ref{thm:inf_Risk}, and Lemmas~\ref{Covering_number}, \ref{Covering_number_of_Transformer}, and \ref{Covering_number_of_MLP}, we can derive the ICL expected excess risk \( R_B(\mathbf{p}_0, \hat{\mathbf{p}}) \).

\begin{theorem}\label{the main risk bound} 
Assume $\mathbf{p}_0 \in \mathcal{G}_\alpha(\beta, Q)$ for some $\alpha \in [0, 1]$. Let \(\hat{\mathbf{p}}\) be an estimator of \(\mathbf{p}_0\), taking values in the function class \(\mathcal{F}(M, H, D, W_1, B_1, S_1, \gamma, L, W, S)\), where the parameters of this class scale with the precision parameter \(\varepsilon > 0\) as follows:
\begin{gather*}
M \lesssim \log\left(\frac{Q}{\varepsilon}\right),\  H \lesssim \left(\frac{Q}{\varepsilon}\right)^{d/\beta},\  D \lesssim 1, \\
W_1 \lesssim \left(\frac{Q}{\varepsilon}\right)^{d/\beta},\  B_1 \lesssim \|\mathbf{p}_0\|_\infty,\  S_1 \lesssim \left(\frac{Q}{\varepsilon}\right)^{d/\beta} \log\left(\frac{Q}{\varepsilon}\right), \\
L \lesssim \log\left(\frac{1}{\varepsilon}\right),\  W \lesssim K \left(\frac{1}{\varepsilon}\right)^{1/\beta},\  S \lesssim K \left(\frac{1}{\varepsilon}\right)^{1/\beta} \log\left(\frac{1}{\varepsilon}\right).
\end{gather*}
Let \(\Delta_T(\mathbf{p}_0, \widehat{\mathbf{p}})\) denote the expected empirical risk gap defined in \eqref{eq:empirical_risk_gap}. If 
\begin{align*}
\Delta_T(\mathbf{p}_0, \widehat{\mathbf{p}}) \lesssim T^{-\frac{(1+\alpha)\beta}{(1+\alpha)\beta + d}} K^{4+\alpha},
\end{align*}
and we choose 
$
\varepsilon = C^{-1} T^{-\frac{\beta}{(1 + \alpha)\beta + d}} K^{-\frac{(3 + \alpha)\beta}{(1 + \alpha)\beta + d}} 
$  
for some constant \(C > 0\), then the ICL risk is bounded by
\begin{align*}
R_B\left(\mathbf{p}_0, \hat{\mathbf{p}}\right) \lesssim B^2 K^{4+\alpha} T^{-\frac{(1+\alpha)\beta}{(1+\alpha)\beta + d}}\log^7\left(KT\right).
\end{align*}
In particular, for the special case where \( \alpha = 1 \),  the ICL risk achieves the bound
\begin{align*}
R_B\left(\mathbf{p}_0, \hat{\mathbf{p}}\right) \lesssim B^2 K^{5} T^{-\frac{2\beta}{2\beta + d}} \log^7\left(KT\right).
\end{align*}
\end{theorem}

Theorem~\ref{the main risk bound} establishes that the upper bound of the ICL risk attains the minimax optimal  rate $ T^{-2\beta/(2\beta + d)}$ (up to logarithmic factors) when $\alpha=1$. This result highlights the importance of pretraining task diversity, as observed by \citet{raventos2023pretraining}. Our setup corresponds to the \emph{task-scaling} regime formalized by \citet{abedsoltan2024context}, where the  risk decreases as the number of pretraining tasks \(T\) increases, while the context length $N$ remains fixed. In our multiclass setting, the achievable risk is fundamentally determined by the number of pretraining tasks \(T\).  
Since large language models are trained on virtually unlimited tasks, our bound provides a theoretical explanation for their strong in-context learning performance, even with only a small number of in-context examples at inference time. \\

\noindent {\bf Proof of Theorem~\ref{the main risk bound}.} The proof is based on the oracle inequality \eqref{oracle_inequality}, which decomposes the ICL risk $R_B\left(\mathbf{p}_0, \hat{\mathbf{p}}\right)$ into three components: the expected empirical risk gap
\( \Delta_T(\mathbf{p}_0, \widehat{\mathbf{p}}) \), the approximation error \(\inf_{\mathbf{p}\in \mathcal{F}} R(\mathbf{p}_0, \mathbf{p})\), controlled by Theorem~\ref{thm:inf_Risk}, and the complexity term \( (B^2\log(\mathcal{N}_T) \log^2 T)/T \), bounded using Lemmas~\ref{Covering_number}, \ref{Covering_number_of_Transformer}, and \ref{Covering_number_of_MLP}, with \(\delta = 1/T\). The precision parameter \(\varepsilon = C^{-1} T^{-\frac{\beta}{(1+\alpha)\beta + d}} K^{-\frac{(3+\alpha)\beta}{(1+\alpha)\beta + d}}\) is chosen to balance the approximation error and the complexity term, reflecting the trade-off between network expressiveness and statistical complexity. These bounds together yields our result. 
\hfill $\blacksquare$

\section{Lower Bounds}
To establish the optimality of the convergence rate in Theorem~\ref{the main risk bound}, we derive a matching minimax lower bound for the ICL risk. Our argument leverages the classical minimax theory for nonparametric estimation under the squared Hellinger distance and translates these results to the KL divergence setting.
The squared Hellinger distance between two probability measures $P$ and $Q$ is defined as \( H^2(P, Q) = \frac{1}{2} \int ( \sqrt{dP} - \sqrt{dQ} )^2 \), which is a fundamental tool for deriving statistical lower bounds due to its favorable metric properties \citep{van1996weak, geer2000empirical}. The rate \( T^{-2\beta/(2\beta + d)} \) is known to be minimax optimal for nonparametric estimation under the squared Hellinger loss \citep[Chapter~2]{tsybakov2008introduction}. The key to translating this result is the classical inequality \( 2H^2(P, Q) \leq \text{KL}(P \| Q) \), 
see Lemma~4 in \cite{haussler1997mutual} and  (7.5) in \cite{birge1998minimum}.  Conversely, under certain integrability assumptions, the KL divergence can be bounded by the Hellinger distance up to logarithmic factors \citep[Theorem~5]{wong1995probability}. This two-sided relationship enables the direct translation of minimax lower bounds from the Hellinger distance to the KL divergence. The following lemma formalizes this connection for the truncated KL divergence. 

\begin{lemma}\label{lem:hellinger-klb-equivalence}
Let $P$ and $Q$ denote two probability measures defined on the same measurable space. Then for any \( B \geq 2 \), 
\begin{align*}
H^2(P, Q) \leq \tfrac{1}{2} \mathrm{KL}_2(P \| Q) \leq \tfrac{1}{2} \mathrm{KL}_B(P \| Q) \leq 5 e^{B/2} H^2(P, Q).
\end{align*}
\end{lemma}

The proof of Lemma~\ref{lem:hellinger-klb-equivalence} is provided in Appendix~\ref{app:hellinger-klb}. Similar results have been established in \citet{wong1995probability} and \citet{bos2022convergence}. 
This inequality establishes a two-sided relationship between the squared Hellinger distance and the truncated Kullback-Leibler divergence. It enables the transfer of approximation guarantees from \( H^2(P, Q) \) to \( \mathrm{KL}_B(P \| Q) \), up to some constants depending on the truncation level \( B \). Consequently, the ICL risk based on \( \mathrm{KL}_B(P \| Q) \) inherits the same convergence rate as that under \( H^2(P, Q) \), up to a constant factor. Given that the rate \( T^{-2\beta/(2\beta + d)} \) is known to be minimax optimal for nonparametric estimation under the squared Hellinger loss \citep[Chapter~2]{tsybakov2008introduction}, Theorem~\ref{the main risk bound} and Lemma~\ref{lem:hellinger-klb-equivalence} collectively imply that the ICL risk \( R_B(\mathbf{p}_0, \hat{\mathbf{p}}) \) based on the truncated KL divergence achieves the minimax optimal rate (up to logarithmic factors) when $\alpha = 1$. A specific lower bound of the ICL risk is summarized in the following theorem.

\begin{theorem}\label{thm:lower_bound_rate}
Consider the same setting as in Theorem~\ref{the main risk bound}. Then the minimax lower bound of the ICL risk with $\alpha=1$ is given by
\[
\inf_{\hat{\mathbf{p}}} \sup_{\mathbf{p}_0 \in \mathcal{G}_1(\beta, Q)} R_B(\mathbf{p}_0, \hat{\mathbf{p}}) \gtrsim T^{-\frac{2\beta}{2\beta + d}},
\]
where $\mathcal{G}_{\alpha}(\beta, Q) = \{ \mathbf{p}_0: \mathbf{p}_0 \in \mathcal{G}(\beta,    Q) \ {\rm and} \ \mathbf{p}_0 \ is \ \alpha-SVB \}$. 
\end{theorem}

This result aligns with the upper bound in Theorem~\ref{the main risk bound} up to logarithmic factors, thereby establishing the minimax optimality of the convergence rate for the ICL risk under the truncated Kullback-Leibler loss when \(\alpha = 1\). 
For general $\alpha \in (0,1)$, the minimax optimal rate of the expected excess risk under the truncated KL divergence remains unknown. We conjecture that the optimal rate is given by $T^{-\frac{(1+\alpha)\beta}{(1+\alpha)\beta + d}}$, up to logarithmic factors.
However, a rigorous proof is beyond the scope of the current paper and warrants further study.

\section{In-Context Learning Beyond Transformers}
ICL is often regarded as a distinctive capability of transformer architectures,  with explanations frequently grounded in constructions based on attention mechanisms~\citep{akyurek2022learning, von2023transformers, zhang2024trained, reddy2023mechanistic, lu2024asymptotic}. 
However, recent literature suggests that ICL is not exclusive to transformer architectures. \citet{tong2024mlps} empirically demonstrated that MLPs can also learn in-context on standard synthetic ICL tasks. Similarly, \citet{abedsoltan2024context} showed that standard MLPs are capable of task-scaling in ICL settings. Moreover, \citet{kratsios2025context} theoretically established that MLPs possess universal approximation capabilities for ICL tasks.

In this section, we support these insights by analyzing the expected excess risk of MLP-based classifiers in the setting of multiclass ICL. Under appropriate assumptions, we demonstrate that MLPs can achieve the minimax optimal convergence rate for ICL. Our analysis provides theoretical support for recent empirical evidence on the effectiveness of MLPs. This finding suggests that, while attention mechanisms are powerful, they are not strictly necessary for enabling ICL.
Specifically, we consider a class of MLPs with ReLU activation functions, used as probabilistic classifiers. To accommodate input prompts in \(\mathbb{R}^d\), we modify the network architecture introduced in~\eqref{eq:MLP_definition} by increasing its input dimension from \(K\) to \(d\).  Accordingly, we introduce a function class as follows:
\begin{align}\label{eq:prob_class_phi}
\mathcal{F}_{\rm mlp}(L,W,S) := \left\{ \tilde{\mathbf{p}}(\mathbf{Z}) = \boldsymbol{\Phi} \circ \mathbf{f}(\operatorname{vec}(\mathbf{Z})) \;\middle|\;
	\mathbf{f} \in \mathcal{F}_{\mathrm{id}}(L, W, S)\right\},
\end{align}
where $\boldsymbol \Phi$ denotes the softmax function and $\operatorname{vec}(\mathbf{Z})$ is the vectorization of the matrix $\mathbf{Z}$. 

\begin{lemma}\label{MLP approximation}
For any \( \mathbf{p}_0 \in \mathcal{G}(\beta, Q) \) and \( \varepsilon \in (0, \min\{ (K(4+C_{Q, \beta, d}))^{-1}, (\beta+1)^{-\beta}, (Q+1)^{-\beta/d} e^{-\beta}\} ) \) with $C_{Q,\beta,d} = (2Q+1)(1+d^2+\beta^2)6^d + Q3^\beta $, there exist a neural network \( \tilde{\mathbf{p}} \in \mathcal{F}_{\rm mlp}(L, W, S) \) with
\begin{align*}
L \lesssim \log\left(\frac{1}{\varepsilon}\right), \quad 
W \lesssim K\varepsilon^{-d/\beta}, \quad 
S \lesssim K\varepsilon^{-d/\beta} \log\left(\frac{1}{\varepsilon}\right),
\end{align*}
such that 
\begin{align*}
\|\tilde{p}_k - p_k^0\|_\infty \le \frac{1}{2} K(C_{Q,\beta,d} + 4) \varepsilon, 
\quad \tilde{p}_k(\mathbf{Z}) \geq \frac{\varepsilon}{4}
\end{align*}
for all \( k \in [K] \) and $\mathbf{Z} \in [0,1]^d$.
\end{lemma}

This result is a slight modification of Lemma~4.3 in \citet{bos2022convergence}, which builds upon a line of work \citep{yarotsky2017error,schmidt2020nonparametric, 
jiao2023deep, chakraborty2024statistical} showing that ReLU MLPs with depth \( L \lesssim \log(1/\varepsilon) \), width \( W \lesssim \varepsilon^{-d/\beta} \), and size \( S \lesssim \varepsilon^{-d/\beta} \log(1/\varepsilon) \) can efficiently approximate \( \beta \)-Hölder smooth functions. 
Building on Lemma~\ref{MLP approximation}, we then derive a non-asymptotic upper bound for the approximation error $\inf_{\mathbf{p} \in \mathcal{F}} R(\mathbf{p}_0, \mathbf{p}) $. 

\begin{theorem}\label{mlp-appro-error}
Assume $\mathbf{p}_{0} \in \mathcal{G}_{\alpha}(\beta, Q)$ for some $\alpha \in [0,1]$ and $\mathbf{p}=(p_{1},\ldots,p_{K})^{\top}:[0,1]^d\to\mathcal{S}^K$ is an approximation of $\mathbf{p}_{0}$ such that $\|\mathbf{p}-\mathbf{p}_{0}\|_{\infty}\le C_0 \varepsilon/4$ and
$\min_{k}\inf_{\mathbf{Z}\in[0,1]^d} p_k(\mathbf{Z})\ge \varepsilon/4$ with $C_0$ being a positive constant.
Then we have
\begin{align*}
R(\mathbf{p}_{0},\mathbf{p}) \lesssim K (C_0 + 1)^{2 + \alpha}
\varepsilon^{1 + \alpha}
\left(1+\frac{\mathbf{1}(\alpha<1)}{1-\alpha} +\log \left( \frac{4}{\varepsilon} \right) \right).
\end{align*}
Furthermore, if we choose $C_0 = 2K(C_{Q, \beta, d} + 4)$ and 
$ \varepsilon = C^{-1} T^{-\frac{\beta}{(1+\alpha)\beta + d}} K^{-\frac{(3+\alpha)\beta}{(1+\alpha)\beta + d}} $ for some constant $C > 0$, then 
\begin{align*}
\inf_{\mathbf{p} \in \mathcal{F}} R(\mathbf{p}_0, \mathbf{p}) \lesssim 
K^{3+\alpha} T^{-\frac{(1+\alpha)\beta}{(1+\alpha)\beta + d}} \log(KT).
\end{align*}
\end{theorem}

The proof of this theorem follows the same arguments as in Theorem~\ref{thm:inf_Risk}. A similar result can be found in \citet[Theorem~3.2]{bos2022convergence}.

\begin{remark}\label{mlp-lowboud}
Similar to the arguments in Remark~\ref{lowbound-F}, by adding additional ReLU layers to the network in $\mathcal{F}_{\rm id}(L,W,S)$, we can assume that every network $\mathbf{G} = (G_1, \dots, G_p)^{\top} \in \mathcal{F}_{\rm id}(L,W,S)$ satisfies $ \log(\varepsilon) \leq G_i \leq \log(2)$ for all $1 \leq i \leq p$. Since $ \mathcal{F}_{\rm mlp}(L,W,S) = \mathbf{\Phi}( \mathcal{F}_{\rm id}(L,W,S)) $, it follows that every network in $\mathcal{F}_{\rm mlp}(L,W,S)$ has a uniform, strictly positive lower bound of $\varepsilon/2K$ for all its coordinates. Given that the neural network $\tilde{p}_k \geq \varepsilon/4 \geq \varepsilon/2K$, this implies that $\tilde{\mathbf{p}} \in \mathcal{F}_{\rm mlp}(L,W,S)$. Therefore, with a suitable chosen precision parameter $\varepsilon$, the function class $\mathcal{F}_{\rm mlp}(L,W,S)$ satisfies Assumption~\ref{assump:prob-lower-bound}. This allows us to apply the oracle inequality to derive the expected excess risk for MLP-based classifiers in the ICL setting. 
\end{remark}

The next lemma establishes an explicit upper bound on the uniform empirical entropy of \( \mathcal{F}_{\rm mlp}(L,W,S) \), a key quantity in deriving the generalization error bound in the oracle inequality. We write \( \mathcal{F}_{\rm mlp}(L,W,S) = \mathbf{\Phi}(\mathcal{F}_{\mathrm{id}}(L,W,S)) \).

\begin{lemma}\label{lem:covering_number_bound}
Suppose $\| \mathbf{G} \|_{\infty} \leq B_2 < \infty$ for any $\mathbf{G} \in \mathcal{F}_{\mathrm{id}}(L,W,S)$. Then
\begin{equation*}
\mathcal{V}_\infty \left( \frac{1}{T}, \log(\mathcal{F}_{\rm mlp}(L,W,S)), T \right)
\leq C K S L \log(S) \log(B_2KT),
\end{equation*}
where $\mathcal{V}_\infty\left(1/T, \log(\mathcal{F}_{\rm mlp}(L,W,S)), T \right) = \sup_{\mathcal{Z}_T} \log \mathcal{N}(1/T, \log(\mathcal{F}_{\rm mlp}(L,W,S)), \|\cdot\|_{\infty, \mathcal{Z}_T})$, and $C >0$ is an universal constant.
\end{lemma}

The proof of Lemma~\ref{lem:covering_number_bound} is provided in Appendix~\ref{app:proof_covering_number_bound}. Building on the oracle inequality (Theorem \ref{thm:Oracle}), Theorem~\ref{mlp-appro-error}, and Lemma~\ref{lem:covering_number_bound}, we derive an upper bound on the expected excess risk for MLP-based classifiers in the ICL setting, as stated below. 

\begin{theorem}\label{the MLP main risk bound} 
Assume $\mathbf{p}_0 \in \mathcal{G}_\alpha(\beta, Q)$ for some $\alpha \in [0, 1]$.  
Let \( \hat{\mathbf{p}}_{\rm mlp} \) be an estimator of \(\mathbf{p}_0\), taking values in the function class \( \mathcal{F}_{\rm mlp}(L,W,S) \), where the parameters of this class scale with the precision parameter \(\varepsilon > 0\) as follows:
	\begin{gather*}
		L \lesssim \log\left(\frac{1}{\varepsilon}\right),\  W \lesssim K \varepsilon^{-d/\beta},\  S \lesssim K\varepsilon^{-d/\beta}\log\left(\frac{1}{\varepsilon}\right).
	\end{gather*}
	If we choose the precision parameter
	\( \varepsilon = C^{-1} T^{-\frac{\beta}{(1 + \alpha)\beta + d}} K^{-\frac{(3 + \alpha)\beta}{(1 + \alpha)\beta + d}} \) for some positive constant \( C \), and the expected empirical risk gap satisfies 
	\[
	\Delta_T(\mathbf{p}_0, \widehat{\mathbf{p}}_{\rm mlp}) \lesssim K^{5+\alpha}
    T^{-\frac{(1+\alpha)\beta}{(1+\alpha)\beta + d}},
	\]
	then the ICL risk is bounded by
	\[
	R_B\left(\mathbf{p}_0, \hat{\mathbf{p}}_{\rm mlp} \right) \lesssim B^2  K^{5+\alpha} T^{-\frac{(1+\alpha)\beta}{(1+\alpha)\beta + d}} \log^6\left(KT\right).
	\]
	In particular, for the special case where \( \alpha = 1 \),  the ICL risk achieves the bound
	\[
	R_B\left(\mathbf{p}_0, \hat{\mathbf{p}}_{\rm mlp} \right) \lesssim B^2  K^6 T^{-\frac{2\beta}{2\beta + d}} \log^6(KT).
	\]
\end{theorem}

\begin{remark}
We have shown that the ICL risk bounds in Theorems~\ref{the main risk bound} and \ref{the MLP main risk bound} both achieve the optimal rate \( T^{-2\beta/(2\beta + d)} \), up to logarithmic factors. Recent works have empirically shown that MLPs can also exhibit ICL behavior~\citep{tong2024mlps, kratsios2025context, abedsoltan2024context}. Building on this, we prove that in the task-scaling regime, where the number of tasks \(T\) increases while the context length remains fixed, the convergence rate of MLPs matches that of transformer-based models. This provides a theoretical justification for the comparable performance of MLP-based ICL under realistic pretraining conditions.
\end{remark}

\noindent {\bf Proof of Theorem~\ref{the MLP main risk bound}.}  
Recall that $ \varepsilon = C^{-1} T^{-\frac{\beta}{(1 + \alpha)\beta + d}} K^{-\frac{(3 + \alpha)\beta}{(1 + \alpha)\beta + d}} $ and $\| \mathbf{G} \|_{\infty} \leq |\log(\varepsilon)|$ for any $\mathbf{G} \in \mathcal{F}_{\mathrm{id}}(L,W,S)$, it follows that $\| \mathbf{G} \|_{\infty} \leq C \log(KT)$ for some positive constant $C$. 
Applying Theorems~\ref{thm:Oracle},~\ref{mlp-appro-error}, Lemmas~\ref{MLP approximation} and \ref{lem:covering_number_bound}, and expressing the network parameters as functions of \(\varepsilon\), we obtain the claimed in-context risk bound.
\hfill $\blacksquare$

\section{Discussions}
The upper bounds established in Theorem~\ref{the main risk bound} and ~\ref{the MLP main risk bound} reflect the curse of dimensionality, as the input prompt dimension $d$ explicitly appears in the convergence rate exponent. This phenomenon is well-known in nonparametric estimation under smoothness constraints, where the optimal minimax rate for estimating a $\beta$-Hölder function scales as $T^{-2\beta/(2\beta + d)}$~\citep{donoho1998minimax}. Our result demonstrates that this dependence persists in transformer-based ICL under general smoothness assumptions.
Nevertheless, the practical impact of dimensionality may be mitigated, as the number of tasks $T$ is typically extremely large in realistic ICL scenarios. Furthermore, transformer architectures are believed to possess an inherent ability to adapt to latent low-dimensional structure within the input, potentially reducing their effective dependence on the  prompt dimension.

Recent work has explored multiple strategies to overcome the curse of dimensionality. One strategy is to explore alternative function classes with more favorable approximation properties, such as mixed-Besov spaces~\citep{suzuki2018adaptivity} and Korobov spaces~\citep{montanelli2019new}. Another approach is to impose structural assumptions on the target function, including additive ridge functions~\citep{fang2022optimal}, hierarchical compositions~\citep{schmidt2020nonparametric, han2020depth}, generalized single-index models~\citep{bauer2019deep}, and multivariate adaptive regression splines (MARS)~\citep{kohler2022estimation}.
Beyond structural assumptions on function classes, recent literature highlights the importance of leveraging the geometric structure of input data. The \emph{manifold hypothesis}~\citep{pope2021intrinsic} suggests that although data appear in high-dimensional spaces, their intrinsic structure is governed by a low-dimensional manifold~\citep{tenenbaum2000global, roweis2000nonlinear}. 
Under this assumption, deep neural networks have been shown to achieve significantly faster convergence rates by effectively adapting to the intrinsic dimensionality of data~\citep{schmidt2019deep, chen2019efficient, chen2022nonparametric, nakada2020adaptive, cloninger2021deep}. Recent studies further characterize the statistical properties of deep learning and generative models under low intrinsic dimensionality, providing refined generalization guarantees in these settings~\citep{chakraborty2024statistical, chakraborty2024gan}.

Despite these advances, the theoretical understanding of transformers in low-dimensional geometric settings remains limited. A recent contribution by~\citet{havrilla2024understanding} analyzes transformer approximation for a regression problem on low-dimensional manifolds, whose setting is quite different from ours. Extending Lemma~\ref{lem:transformer_approximation} to manifold-supported functions warrants further study, as this would allow risk bounds to scale with the intrinsic dimension rather than the ambient dimension.

\acks{We gratefully acknowledge the support of the National Natural Science Foundation of China (NSFC12471276, NSFC12131006, NSFC12201048, NSFC12301383, NSFC12471276) 
and the Scientific and Technological Innovation Project of China Academy of 
Chinese Medical Science (CI2023C063YLL).}

\appendix

\section{Proof of Theorem \ref{thm:Oracle}}\label{app:proof_oracle}
\begin{proof}
We begin with some notations that will be used in the proof. 
	\begin{itemize}
		\item The expected excess risk defined as:
		\begin{equation*}
			R_B(\mathbf{p}_0, \hat{\mathbf{p}}) = \mathbb{E}_{\mathcal{D}_T, \mathbf{Z}} \left[ \boldsymbol{p}_0(\mathbf{Z})^\top \left( B \wedge \log\left( \frac{\boldsymbol{p}_0(\mathbf{Z})}{\hat{\boldsymbol{p}}(\mathbf{Z})} \right) \right) \right].
		\end{equation*}
		\item The expected empirical risk is given by:
		\begin{equation*}
			R_{B,T}(\mathbf{p}_0, \hat{\mathbf{p}}) := \mathbb{E}_{\mathcal{D}_T} \left[ \frac{1}{T} \sum_{t=1}^T \mathbf{y}^{(t)\top} \left( B \wedge \log\left( \frac{\mathbf{p}_0(\mathbf{Z}^{(t)})}{\hat{\mathbf{p}}(\mathbf{Z}^{(t)})} \right) \right) \right].
		\end{equation*}
		
		\item The data set is denoted by:
		\begin{equation*}
			\mathcal{D}_T = \{ (\mathbf{Z}^{(t)}, \mathbf{y}^{(t)}) \}_{t=1}^{T}.
		\end{equation*}
		Here, we omit the subscript $N+1$ from $\mathbf{y}_{N+1}^{(t)}$ for simplicity.
		
		\item The expected empirical risk gap is defined as:
		\begin{equation*}
			\Delta_T(\boldsymbol{p}_0, \hat{\boldsymbol{p}}) := \mathbb{E}_{\mathcal{D}_T}\left[
			-\frac{1}{T} \sum_{t=1}^{T} \mathbf{y}^{(t)\top} \log(\hat{\boldsymbol{p}}(\mathbf{Z}^{(t)}))
			- \min_{\boldsymbol{p} \in \mathcal{F}} \left(
			-\frac{1}{T} \sum_{t=1}^{T} \mathbf{y}^{(t)\top} \log(\boldsymbol{p}(\mathbf{Z}^{(t)}))
			\right)
			\right].
		\end{equation*}
	\end{itemize}
	
	To establish the oracle inequality on $R_B(\mathbf{p}_0, \widehat{\mathbf{p}})$, we first derive an upper bound for the difference $ R_B(\mathbf{p}_0, \hat{\mathbf{p}}) - 2 R_{B,T}(\mathbf{p}_0, \hat{\mathbf{p}}) $. Lemma \ref{P: Upper Bound Risk Delta} below provides an upper bound of $R_{B,T}(\mathbf{p}_0, \hat{\mathbf{p}})$, from which the theorem follows by applying the triangle inequality.
	
	The rest of this proof gives an upper bound for 
    $ R_B(\mathbf{p}_0, \hat{\mathbf{p}}) - 2 R_{B,T}(\mathbf{p}_0, \hat{\mathbf{p}}) $. 
    We decompose the expected excess risk \(R_B(\mathbf{p}_0, \hat{\mathbf{p}})\) 
    into two parts, according to whether the true probability 
    \(p_k^0(\mathbf{Z})\) is above a truncation threshold. 
    This yields the following form:
	\begin{align*}
		R_B(\mathbf{p}_0, \hat{\mathbf{p}}) 
		&= \mathbb{E}_{\mathcal{D}_T} \mathbb{E}_{\mathbf{Z}} \sum_{k=1}^K p^0_k(\mathbf{Z}) \left[ B \wedge \log\left( \frac{p^0_k(\mathbf{Z})}{\hat{p}_k(\mathbf{Z})} \right) \right]  \mathbb{I}\left( p^0_k(\mathbf{Z}) \geq \frac{C_T}{T} \right) \\
		&\quad + \mathbb{E}_{\mathcal{D}_T} \mathbb{E}_{\mathbf{Z}} \sum_{k=1}^K p^0_k(\mathbf{Z}) \left[ B \wedge \log\left( \frac{p^0_k(\mathbf{Z})}{\hat{p}_k(\mathbf{Z})} \right) \right]  \mathbb{I}\left( p^0_k(\mathbf{Z}) < \frac{C_T}{T} \right) \\
		&=:\ R_B^{(1)}(\mathbf{p}_0, \hat{\mathbf{p}}) + R_B^{(2)}(\mathbf{p}_0, \hat{\mathbf{p}}).
	\end{align*}
	In parallel, we similarly decompose the expected empirical risk \(R_{B,T}(\mathbf{p}_0, \hat{\mathbf{p}})\) into contributions from the large and small-probability regions:
	\begin{align*}
		R_{B,T}(\mathbf{p}_0, \hat{\mathbf{p}}) 
		&= \mathbb{E}_{\mathcal{D}_T} \frac{1}{T} \sum_{t=1}^T \sum_{k=1}^K y_k^{(t)} \left[ B \wedge \log\left( \frac{p_k^0(\mathbf{Z}^{(t)})}{\hat{p}_k(\mathbf{Z}^{(t)})} \right) \right] \mathbb{I}\left( p_k^0(\mathbf{Z}^{(t)}) \geq \frac{C_T}{T} \right) \\
		&\quad + \mathbb{E}_{\mathcal{D}_T} \frac{1}{T} \sum_{t=1}^T \sum_{k=1}^K y_k^{(t)} \left[ B \wedge \log\left( \frac{p_k^0(\mathbf{Z}^{(t)})}{\hat{p}_k(\mathbf{Z}^{(t)})} \right) \right] \mathbb{I}\left( p_k^0(\mathbf{Z}^{(t)}) < \frac{C_T}{T} \right) \\
		&=:\ R_{B,T}^{(1)}(\mathbf{p}_0, \hat{\mathbf{p}}) + R_{B,T}^{(2)}(\mathbf{p}_0, \hat{\mathbf{p}}).
	\end{align*}
	We then decompose the difference $ R_B(\mathbf{p}_0, \hat{\mathbf{p}}) - 2 R_{B,T}(\mathbf{p}_0, \hat{\mathbf{p}}) $ into three components: 
	\begin{align}\label{eq:excess-risk-decomp}
		R_B(\mathbf{p}_0, \hat{\mathbf{p}}) - 2 R_{B,T}(\mathbf{p}_0, \hat{\mathbf{p}})
		= W_1 + R_B^{(2)}(\mathbf{p}_0, \hat{\mathbf{p}}) 
		- 2 R_{B,T}^{(2)}(\mathbf{p}_0, \hat{\mathbf{p}}),
	\end{align}
	where $$ W_1 = R_B^{(1)}(\mathbf{p}_0, \hat{\mathbf{p}}) - 2 R_{B,T}^{(1)}(\mathbf{p}_0, \hat{\mathbf{p}}). $$
	Each of these three terms will be treated separately: $R_B^{(2)}(\mathbf{p}_0, \hat{\mathbf{p}})$ and $2 R_{B,T}^{(2)}(\mathbf{p}_0, \hat{\mathbf{p}})$ involving the low-probability region are bounded via Lemma~\ref{Bound for Small Probabilities}, and the term \(W_1\) is analyzed using empirical process techniques.
	
	For $W_1$, we introduce the truncated loss function
	\[
	g_{\mathbf{p}}(\mathbf{Z}, \mathbf{y}) 
	:= \mathbf{y}^\top \left[ \left( B \wedge \log \frac{\mathbf{p}_0(\mathbf{Z})}{\mathbf{p}(\mathbf{Z})} \right)_{\geq \frac{C_T}{T}} \right]
	= \sum_{k=1}^K y_k \left[ B \wedge \log \left( \frac{p_k^0(\mathbf{Z})}{p_k(\mathbf{Z})} \right) \right] 
	\mathbb{I} \left( p_k^0(\mathbf{Z}) \geq \frac{C_T}{T} \right).
	\]
	With this notation, the quantity \(W_1\) can be expressed as
	\begin{align}\label{eq:W1-decomp}
		W_1 
		&=\  R_B^{(1)}(\mathbf{p}_0, \hat{\mathbf{p}}) - 2 R_{B,T}^{(1)}(\mathbf{p}_0, \hat{\mathbf{p}}) \notag \\
		&=\  \mathbb{E}_{\mathcal{D}_T} \, \mathbb{E}_{(\mathbf{Z}, \mathbf{y})} \left[ g_{\hat{\mathbf{p}}}(\mathbf{Z}, \mathbf{y}) \right]
		- 2\, \mathbb{E}_{\mathcal{D}_T} \left[ \frac{1}{T} \sum_{i=1}^T g_{\hat{\mathbf{p}}}(\mathbf{Z}^{(i)}, \mathbf{y}^{(i)}) \right] \notag \\
		&\leq\  \int_0^\infty \mathbb{P} \left( \mathbb{E}_{(\mathbf{Z}, \mathbf{y})} \, g_{\hat{\mathbf{p}}}(\mathbf{Z}, \mathbf{y}) 
		- \frac{2}{T} \sum_{i=1}^T g_{\hat{\mathbf{p}}}(\mathbf{Z}^{(i)}, \mathbf{y}^{(i)}) > t \right) dt.
	\end{align}
	Let $C_T$ be small enough such that  $C_T \leq C_1 T^{1-C_2}$ where $C_1 >0$ and $ 0 < C_2 < 1$. 
	Then the application of Lemma~\ref{lem:excess-risk-bound} with $\epsilon = 1/2$ and $a = b = t/2$ yields  
	\begin{align*}
		&\   \mathbb{P} \left( \mathbb{E}_{(\mathbf{Z}, \mathbf{y})}
		g_{\hat{\mathbf{p}}}(\mathbf{Z}, \mathbf{y}) - \frac{2}{T} \sum_{i=1}^T g_{\hat{\mathbf{p}}}(\mathbf{Z}^{(i)}, \mathbf{y}^{(i)}) > t \right) \\
		\leq&\  \mathbb{P} \left( \exists \ \mathbf{p} \in \mathcal{F}: 
		\mathbb{E}_{(\mathbf{Z}, \mathbf{y})} g_{\mathbf{p}}(\mathbf{Z}, \mathbf{y}) - \frac{2}{T} \sum_{i=1}^T g_{\mathbf{p}}(\mathbf{Z}^{(i)}, \mathbf{y}^{(i)}) > t \right) \\
		=&\  \mathbb{P} \left( \exists \ \mathbf{p} \in \mathcal{F}: 
		\mathbb{E}_{(\mathbf{Z}, \mathbf{y})} g_{\mathbf{p}}(\mathbf{Z}, \mathbf{y}) - \frac{1}{T} \sum_{i=1}^T g_{\mathbf{p}}(\mathbf{Z}^{(i)}, \mathbf{y}^{(i)})
		> \frac{1}{2} \{ t + \mathbb{E}_{(\mathbf{Z}, \mathbf{y})} g_{\mathbf{p}}(\mathbf{Z}, \mathbf{y}) \}  \right) \\  
		\leq&\  10 \sup_{\mathcal{Z}_T} \mathcal{N}\left( \frac{t}{32}, 
		\log(\mathcal{F}), \|\cdot\|_{\infty, \mathcal{Z}_T} \right) 
		\exp\left( -\frac{t T}{480 B_T^2} \right),
	\end{align*}
	where $B_T= B \vee \left| \log \left( C_T/T \right) \right|$ and $ \|f - g\|_{\infty, \mathcal{Z}_T} = \max_{1 \leq t \leq T} \max_{1 \leq k \leq K} \left| f_k(\mathbf{Z}^{(t)}) - g_k(\mathbf{Z}^{(t)}) \right| $.
	Consequently, 
	\begin{align*}
		W_1
		\leq&\  \int_0^\infty \mathbb{P} \left( \mathbb{E}_{(\mathbf{Z}, \mathbf{y})} 
		g_{\hat{\mathbf{p}}}(\mathbf{Z}, \mathbf{y}) - \frac{2}{T} \sum_{i=1}^T g_{\hat{\mathbf{p}}}(\mathbf{Z}^{(i)}, \mathbf{y}^{(i)}) > t \right) dt \\
		\leq&\  \alpha_n + \int_{\alpha_n}^\infty \mathbb{P}\left( \mathbb{E}_{(\mathbf{Z}, 
			\mathbf{y})} g_{\hat{\mathbf{p}}}(\mathbf{Z}, \mathbf{y}) - \frac{2}{T} \sum_{i=1}^T g_{\hat{\mathbf{p}}}(\mathbf{Z}^{(i)}, \mathbf{y}^{(i)}) > t \right) dt   \\
		\leq&\  \alpha_n + \int_{\alpha_n}^\infty 10 \sup_{\mathcal{Z}_T} \mathcal{N}\left( 
		\frac{t}{32}, \log(\mathcal{F}), \|\cdot\|_{\infty, \mathcal{Z}_T} \right) 
		\exp\left( -\frac{t T}{480 B_T^2} \right) dt \\
		\leq&\  \alpha_n + 10 \sup_{\mathcal{Z}_T} \mathcal{N}\left( \frac{\alpha_n}{32}, 
		\log(\mathcal{F}), \|\cdot\|_{\infty, \mathcal{Z}_T} \right) \int_{\alpha_n}^\infty \exp\left( -\frac{t T}{480 B_T^2} \right) dt \\
		=&\  \alpha_n + 10 \sup_{\mathcal{Z}_T} \mathcal{N}\left( \frac{\alpha_n}{32}, 
		\log(\mathcal{F}), \|\cdot\|_{\infty, \mathcal{Z}_T} \right) \frac{480 B_T^2}{T} \exp\left( -\frac{\alpha_n T}{480 B_T^2} \right).
	\end{align*}
	Now let $\alpha_n = (\log(10 \mathcal{N}_T))480 B_T^2/T$ with 
	\[ \mathcal{N}_T = \mathcal{N}_{\infty} \left( \frac{1}{T}, \log (\mathcal{F}), T \right) =\sup_{\mathcal{Z}_T} \mathcal{N} \left( \frac{1}{T}, \log(\mathcal{F}), \|\cdot\|_{\infty, \mathcal{Z}_T} \right). \]
	It is easy to see that $1/T \leq \alpha_n/32$ and then 
	\[
	\sup_{\mathcal{Z}_T} \mathcal{N}\left( \frac{\alpha_n}{32}, \log(\mathcal{F}), \|\cdot\|_{\infty, \mathcal{Z}_T} \right)
	\leq \sup_{\mathcal{Z}_T} \mathcal{N} \left( \frac{1}{T}, 
	\log(\mathcal{F}), \|\cdot\|_{\infty, \mathcal{Z}_T} \right) = \mathcal{N}_T.
	\]
	Consequently,
	\begin{eqnarray}\label{W_1}
		W_1\leq \log(10 \mathcal{N}_T) \frac{480 B_T^2}{T}  + \frac{480 B_T^2}{T}
		= \{ \log(10 \mathcal{N}_T) + 1\} \frac{480 B_T^2}{T}.
	\end{eqnarray}
	
	Now we deal with $R_B^{(2)}(\mathbf{p}_0, \widehat{\mathbf{p}}) $
	and $2 R_{B,T}^{(2)}(\mathbf{p}_0, \widehat{\mathbf{p}}) $ in 
	the decomposition (\ref{eq:excess-risk-decomp}). 
	Let $C_T$ small enough such that $0 < C_T \leq T/e$. Applying Lemma ~\ref{Bound for Small Probabilities} in the following, we have 
	\begin{align}\label{R_B2}
		\left| R_B^{(2)}(\mathbf{p}_0, \widehat{\mathbf{p}}) \right| 
		=& \ \mathbb{E}_{\mathcal{D}_T} \mathbb{E}_{\mathbf{Z}} \sum_{k=1}^K  
		p^0_k(\mathbf{Z})   \left[ B \wedge \log \frac{p^0_k(\mathbf{Z})} {\hat{p}_k(\mathbf{Z})} \right]  
		\mathbb{I}\left( p^0_k(\mathbf{Z}) < \frac{C_T}{T} \right) \nonumber \\
		\leq&\  \frac{K C_T \{ \log(T/C_T) + B \} }{T},
	\end{align}
	and 
	\begin{align}\label{R_BT2}
		\left| R_{B,T}^{(2)}(\mathbf{p}_0, \widehat{\mathbf{p}}) \right|
		=&\  \left| \mathbb{E}_{\mathcal{D}_T} \frac{1}{T} \sum_{t=1}^T \sum_{k=1}^K 
		y_k^{(t)} \left[ B \wedge \log\left(\frac{p_k^0(\mathbf{Z}^{(t)})} {\hat{p}_k(\mathbf{Z}^{(t)})} \right) \right] \mathbb{I}\left( p_k^0(\mathbf{Z}^{(t)}) < \frac{C_T}{T} \right) \right| \nonumber \\
		\leq&\  \frac{K C_T \{ \log(T/C_T) + B \} }{T}.
	\end{align}
	
	Combining (\ref{eq:excess-risk-decomp}), (\ref{W_1}), (\ref{R_B2}), and (\ref{R_BT2}), we have 
	\begin{align}\label{eq:RB-decomp}
		R_B(\mathbf{p}_0, \widehat{\mathbf{p}})
		\leq&\  W_1 + 2R_{B,T}(\mathbf{p}_0, \widehat{\mathbf{p}})
		+ \left| R_B^{(2)}(\mathbf{p}_0, \widehat{\mathbf{p}}) \right|
		+ 2 \left| R_{B,T}^{(2)}(\mathbf{p}_0, \widehat{\mathbf{p}}) \right| \nonumber \\
		\leq&\  2R_{B,T}(\mathbf{p}_0, \widehat{\mathbf{p}}) + \{\log(10 
		\mathcal{N}_T) + 1\} \frac{480 B_T^2}{T} + \frac{3 K C_T \{ 
			\log(T/C_T) + B \}}{T}.  
	\end{align}
	Recall that $B_T= B \vee \left| \log \left( \frac{C_T}{T} \right) \right|$, then the application of Lemma~\ref{P: Upper Bound Risk Delta} yields the oracle inequality
	\begin{align*}
		R_B(\mathbf{p}_0, \widehat{\mathbf{p}})
		&\leq\  2\inf_{\mathbf{p} \in \mathcal{F}} R(\mathbf{p}_0, \mathbf{p}) 
		+ 2\Delta_T(\mathbf{p}_0, \widehat{\mathbf{p}}) \\
		&\quad + \frac{480 B_T^2 \{\log(10 \mathcal{N}_T) + 1\}}{T} 
		+ \frac{3 K C_T \{\log(T/C_T) + B\}}{T} \\[8pt]
		&\leq\  C \left( \inf_{\mathbf{p} \in \mathcal{F}} R(\mathbf{p}_0, \mathbf{p}) 
		+ \Delta_T(\mathbf{p}_0, \widehat{\mathbf{p}}) 
		+ \frac{B^2 \log(\mathcal{N}_T)\log^2(T)}{T} \right).
	\end{align*}
	where $C >0$ is an universal constant. This completes the proof of Theorem~\ref{thm:Oracle}.
\end{proof}

\subsection{Supporting Lemmas}
\begin{lemma}[Bound for Small Probabilities]\label{Bound for Small Probabilities}
	Let $\mathcal{F}$ be a class of conditional class probabilities, and let $\hat{\mathbf{p}}$ be an estimator of $\mathbf{p}_0$ taking values in $\mathcal{F}$. Suppose we have a sample $\mathcal{D}_T = \{(\mathbf{Z}^{(1)}, \mathbf{y}^{(1)}), \cdots, (\mathbf{Z}^{(T)}, \mathbf{y}^{(T)}) \}$ with the same distribution as $(\mathbf{Z}, \mathbf{y})$ and $(\mathbf{Z}, \mathbf{y})$ is independent of $\mathcal{D}_T$. Given a constant $C_T$ satisfying $0 < C_T \leq T/e$, the following bound holds:
	\begin{eqnarray*}
		&& \bigg| \mathbb{E}_{\mathcal{D}_T,(\mathbf{Z},\mathbf{y})} 
		y_k \bigg[ B \wedge \log \frac{p_k^0(\mathbf{Z})} {\widehat{p}_k(\mathbf{Z})} \bigg] \mathbb{I}\left( p_{k}^0(\mathbf{Z}) < \frac{C_T}{T} \right) \bigg|
		\leq \frac{C_T \{ \log( T/C_T ) + B \} }{T} \\
		&& \bigg| \mathbb{E}_{\mathcal{D}_T} \left[ y_k^{(t)} \left( B \wedge \log  
		\frac{p_k^0(\mathbf{Z}^{(t)})}{\widehat{p}_k(\mathbf{Z}^{(t)})}  \right) \mathbb{I}\left( p_k^0(\mathbf{Z}^{(t)}) < \frac{C_T}{T} \right) \right] \bigg| \leq  \frac{C_T \{ \log( T/C_T ) + B \} }{T}, 
	\end{eqnarray*}
	for any $t \in \{ 1 \dots T\}$ and $ k \in \{1, \dots, K\}$.
\end{lemma}

\begin{proof}
	Because $p^0_k, \widehat{p}_k \in [0,1]$, we first observe that
	$ \log(p_{k}^0(\mathbf{Z})) \leq B \wedge \log\frac{p_{k}^0(\mathbf{Z})}{\widehat{p}_k(\mathbf{Z})} \leq B
	$.
	Thus, taking absolute values, we have
	\begin{eqnarray*}
		&&  \bigg| \mathbb{E}_{\mathcal{D}_T,(\mathbf{Z},\mathbf{y})}
		y_k \bigg[ B \wedge \log \frac{p_k^0(\mathbf{Z})} {\widehat{p}_k(\mathbf{Z})} \bigg] \mathbb{I}\left( p_{k}^0(\mathbf{Z}) < \frac{C_T}{T} \right) \bigg| \\
		&\leq& \mathbb{E}_{\mathcal{D}_T,(\mathbf{Z},\mathbf{y})} 
		y_k \{ |\log( p_{k}^0(\mathbf{Z}))| + B \} \mathbb{I} \left( p_{k}^0(\mathbf{Z}) < \frac{C_T}{T} \right)  \\
		&\leq& \mathbb{E}_{\mathbf{Z}} \left[ p_{k}^0(\mathbf{Z}) 
		\{|\log( p_{k}^0(\mathbf{Z}))| + B\} \mathbb{I} \left( p_{k}^0(\mathbf{Z}) < \frac{C_T}{T} \right) \right] \\
		&=& \mathbb{E}_{\mathbf{Z}} \left[ p_{k}^0(\mathbf{Z}) 
		|\log( p_{k}^0(\mathbf{Z}))| \mathbb{I} \left( p_{k}^0(\mathbf{Z}) < \frac{C_T}{T} \right) \right] 
		+ B \mathbb{E}_{\mathbf{Z}} \left[ p_{k}^0(\mathbf{Z}) 
		\mathbb{I} \left( p_{k}^0(\mathbf{Z}) \leq \frac{C_T}{T} \right) \right] \\
		&\leq& \mathbb{E}_{\mathbf{Z}} \left[ p_{k}^0(\mathbf{Z}) 
		|\log( p_{k}^0(\mathbf{Z}))| \mathbb{I} \left( p_{k}^0(\mathbf{Z}) < \frac{C_T}{T} \right) \right] + B \frac{C_T}{T}.
	\end{eqnarray*}
	Since the function $u \mapsto u|\log(u)|$ is increasing on $(0,e^{-1})$ and $C_T/T \leq e^{-1}$, it follows that
	\begin{equation*}
		p_{k}^0(\mathbf{Z}) |\log( p_{k}^0(\mathbf{Z}))| \mathbb{I} \left( p_{k}^0(\mathbf{Z}) < \frac{C_T}{T} \right)  
		\leq \frac{C_T}{T}\log\frac{T}{C_T}.
	\end{equation*}
	Consequently, 
	\begin{equation*}
		\bigg| \mathbb{E}_{\mathcal{D}_T,(\mathbf{Z},\mathbf{y})} y_k \bigg[ 
		B \wedge \log \frac{p_k^0(\mathbf{Z})} {\widehat{p}_k(\mathbf{Z})} \bigg] \mathbb{I}\left( p_{k}^0(\mathbf{Z}) < \frac{C_T}{T} \right) \bigg| 
		\leq \frac{C_T \{ \log( T/C_T ) + B \} }{T}.
	\end{equation*}
	Similarly, 
	\begin{eqnarray*}
		&& \bigg| \mathbb{E}_{\mathcal{D}_T} \left[ y_k^{(t)} \left( B \wedge \log  
		\frac{p_k^0(\mathbf{Z}^{(t)})}{\widehat{p}_k(\mathbf{Z}^{(t)})}  \right) \mathbb{I}\left( p_k^0(\mathbf{Z}^{(t)}) < \frac{C_T}{T} \right) \right] \bigg|  \\
		&\leq& \mathbb{E}_{\mathcal{D}_T} \left[ y_k^{(t)} \{ 
		|\log(p_k^0(\mathbf{Z}^{(t)}))|  + B \} \mathbb{I}\left( p_k^0(\mathbf{Z}^{(t)}) < \frac{C_T}{T} \right) \right] \\
		&=& \mathbb{E}_{\mathbf{Z}^{(t)}} \left[ p_k^0(\mathbf{Z}^{(t)}) \{ 
		|\log(p_k^0(\mathbf{Z}^{(t)}))|  + B \} \mathbb{I}\left( p_k^0(\mathbf{Z}^{(t)}) < \frac{C_T}{T} \right) \right] \\
		&\leq& \frac{C_T \{ \log( T/C_T ) + B \} }{T}.
	\end{eqnarray*}
	This completes the proof of Lemma~\ref{Bound for Small Probabilities}. 
\end{proof}

\begin{lemma}\label{P: Upper Bound Risk Delta}
	For any estimator \( \widehat{\mathbf{p}} \in \mathcal{F} \), the following inequality holds:
	\begin{equation*}
		R_{B,T}(\mathbf{p}_0, \widehat{\mathbf{p}}) \leq R_{\infty,T}(\mathbf{p}_0, \widehat{\mathbf{p}}) \leq \inf_{\mathbf{p} \in \mathcal{F}} R(\mathbf{p}_0, \mathbf{p}) + \Delta_T(\mathbf{p}_0, \widehat{\mathbf{p}}).
	\end{equation*}
\end{lemma}
\begin{proof}
	The first inequality follows directly from the definition, since for any $a,b \in \mathbb{R}$, we have $a \geq \min(a,b)$.
	
	To show the second inequality, fix any arbitrary $\mathbf{p}^* \in \mathcal{F}$. Since $\Delta_T(\mathbf{p}_0,\mathbf{p}^*) \geq 0$, it follows that
	\begin{align*}
		\mathbb{E}_{\mathcal{D}_T}\left[-\frac{1}{T}\sum_{t=1}^T \mathbf{y}^{(t)\top} \log(\widehat{\mathbf{p}}(\mathbf{Z}^{(t)}))\right] 
		&\leq\  \mathbb{E}_{\mathcal{D}_T}\left[-\frac{1}{T}\sum_{t=1}^T \mathbf{y}^{(t)\top} \log(\mathbf{p}^*(\mathbf{Z}^{(t)}))\right] + \Delta_T(\mathbf{p}_0,\widehat{\mathbf{p}}) \\
		&=\  \mathbb{E}_{\mathbf{Z}}\left[-\mathbf{p}_0(\mathbf{Z})^\top \log(\mathbf{p}^*(\mathbf{Z}))\right] + \Delta_T(\mathbf{p}_0,\widehat{\mathbf{p}}),
	\end{align*}
	where the equality follows from the law of total expectation and the fact that $\mathbb{E}[\mathbf{y}^{(t)}|\mathbf{Z}^{(t)}] = \mathbf{p}_0(\mathbf{Z}^{(t)})$.
	Note that 
	$$ \mathbb{E}_{\mathcal{D}_T}\left[ \frac{1}{T}\sum_{t=1}^T \mathbf{y}^{(t)\top} \log(\mathbf{p}_0(\mathbf{Z}^{(t)})) \right]
	= \mathbb{E}_{\mathbf{Z}} \left[ \mathbf{p}_0(\mathbf{Z})^\top \log(\mathbf{p}_0(\mathbf{Z})) \right]. $$
	Adding $\mathbb{E}_{\mathbf{Z}}[\mathbf{p}_0(\mathbf{Z})^\top\log(\mathbf{p}_0(\mathbf{Z}))]$ to both sides yields the desired inequality.
\end{proof}

Recall that \( \mathbf{p}(\mathbf{Z}) = ( p_1(\mathbf{Z}), \dots, p_K(\mathbf{Z}))^{\top} \)and $ g_{\mathbf{p}}(\mathbf{Z}, \mathbf{y}) = \mathbf{y}^\top \left( B \wedge \log \frac{\mathbf{p}_0(\mathbf{Z})}{\mathbf{p}(\mathbf{Z})} \right)_{\geq C_T / T} $, where \(B \geq 2\). 
We then have the following result, which is used to derive the oracle inequality.

\begin{lemma}\label{lem:excess-risk-bound}
Suppose that for $\mathbf{p} \in \mathcal{F}$, $p_k(\mathbf{Z})
\geq C_1/T^{C_2}$ for some constants $C_1 > 0$ and $0 < C_2 < 1$,
and $C_T \leq C_1 T^{1-C_2}$. Then for each $\epsilon \in (0, \tfrac{1}{2}]$ and $a, b > 0$, we have
	\begin{eqnarray*}
		&& \mathbb{P} \Big\{ \exists \ \mathbf{p} \in \mathcal{F}: \
		\mathbb{E}g_{\bf p}(\mathbf{Z, y}) - \frac{1}{T} \sum_{t=1}^T g_{\bf p} (\mathbf{Z}^{(t)}, \mathbf{y}^{(t)}) \geq \epsilon \{ a + b + \mathbb{E}g_{\bf p}(\mathbf{Z, y}) \} \Big\} \\
		&\leq& 10 \sup_{\mathcal{Z}_T} \mathcal{N}\left( \frac{\epsilon b}{8}, 
		\log(\mathcal{F}), \|\cdot\|_{\infty, \mathcal{Z}_T} \right) 
		\exp\left( -\frac{\epsilon^2 (1-\epsilon) a T}{15 (1+\epsilon) B_T^2} \right).
	\end{eqnarray*}
\end{lemma}

\begin{proof}
The proof of this Lemma follows the same line as Theorem 11.4 of \citet{gyorfi2006distribution} with some extra complications arising from dealing with the truncated loss function $g_{\mathbf{p}}(\mathbf{Z}, \mathbf{y})$ . We here provide a detailed proof for the sake of completeness. Recall that
$$ g_{\mathbf{p}}(\mathbf{Z}, \mathbf{y}) 
   = \mathbf{y}^\top \left( B \wedge \log   
   \frac{\mathbf{p}_0(\mathbf{Z})}{\mathbf{p}(\mathbf{Z})} \right)_{\geq \frac{C_T}{T}}  
   = \sum_{k=1}^K y_k \left[ B \wedge \log \left(   
   \frac{p_k^0(\mathbf{Z})}{p_k(\mathbf{Z})} \right) \right] 
   \mathbb{I} \left( p_k^0(\mathbf{Z}) \geq \frac{C_T}{T} \right).
$$
It is easy to see that
	\begin{align*}
		g_{\mathbf{p}}(\mathbf{Z}, \mathbf{y}) 
		&\leq \ \sum_{k=1}^K y_k B = B, \\[1ex]
		g_{\mathbf{p}}(\mathbf{Z}, \mathbf{y}) 
		&\geq\  \sum_{k=1}^K y_k \left[ B \wedge \log \left( \frac{p_k^0(\mathbf{Z})}{p_k(\mathbf{Z})} \right) \right] \mathbb{I}\left( p_k^0(\mathbf{Z}) \geq \frac{C_T}{T} \right) \\[1ex]
		&\geq \ \sum_{k=1}^K y_k \left[ B \wedge \log \left( \frac{C_T}{T} \right) \right] = \log \left( \frac{C_T}{T} \right).
	\end{align*}
	Write $B_T = B \vee \left| \log ( C_T/T ) \right|$, and it follows that $|g_{\mathbf{p}}(\mathbf{Z}, \mathbf{y})| \leq B_T$.
	The rest of the proof will be divided into six steps to enhance the readability.
	
	{\bf Step 1.} Symmetrization. 
	
	Let $\tilde{\mathcal{D}}_T = \{ (\mathbf{\tilde{Z}}^{(t)}, \mathbf {\tilde{y}}^{(t)}) \}_{t=1}^T$ be i.i.d. random variables with the same distribution as $(\mathbf{Z}, \mathbf{y})$ and independent of $\mathcal{D}_T$.
	Consider the event 
	$$ \bigg\{ \exists \ \mathbf{p} \in \mathcal{F} : \ \mathbb{E}g_{\bf p}(\mathbf{Z, y}) - \frac{1}{T} \sum_{t=1}^T g_{\bf p}(\mathbf{Z}^{(t)}, \mathbf{y}^{(t)}) \geq \epsilon[a + b + \mathbb{E}g_{\bf p}(\mathbf{Z, y}) ] \bigg\}. $$
	Let $\mathbf{p}^T \in \mathcal{F}$ be a function depending on $\mathcal{D}_T = \{ (\mathbf{Z}^{(t)}, \mathbf{y}^{(t)}) \}_{t=1}^T$ such that 
	$$ \bigg\{ \mathbb{E}[g_{\mathbf{p}^T}(\mathbf{Z, y})| \mathcal{D}_T] - \frac{1}{T} \sum_{t=1}^T g_{\mathbf{p}^T}(\mathbf{Z}^{(t)}, \mathbf{y}^{(t)}) \geq \epsilon \{a + b + \mathbb{E}[g_{\mathbf{p}^T}(\mathbf{Z, y})|\mathcal{D}_T] \} \bigg\}. $$
	Note that
	\begin{eqnarray*}
		\mathbb{E}[g_{\mathbf{p}}^2(\mathbf{Z}, \mathbf{y})]
		&=& \mathbb{E} \left( \sum_{k=1}^K y_k \left[ B    
		\wedge  \log \left( \frac{p_k^0(\mathbf{Z})}{p_k(\mathbf{Z})} \right) \right] \mathbb{I} \left( p_k^0(\mathbf{Z}) \geq \frac{C_T}{T} \right) \right)^2  \\
		&=& \mathbb{E} \left( \sum_{k=1}^K y_k \left[ B    
		\wedge  \log \left( \frac{p_k^0(\mathbf{Z})}{p_k(\mathbf{Z})} \right) \right]^2 \mathbb{I} \left( p_k^0(\mathbf{Z}) \geq \frac{C_T}{T} \right) \right)  \\
		&=& \mathbb{E} \left( \sum_{k=1}^K p_k^0(\mathbf{Z}) \left[ B    
		\wedge  \log \left( \frac{p_k^0(\mathbf{Z})}{p_k(\mathbf{Z})} \right) \right]^2 \mathbb{I} \left( p_k^0(\mathbf{Z}) \geq \frac{C_T}{T} \right) \right)  \\ 
		&\leq& \mathbb{E} \left( \sum_{k=1}^K p_k^0(\mathbf{Z}) \left[ B    
		\wedge  \log \left( \frac{p_k^0(\mathbf{Z})}{p_k(\mathbf{Z})} \right) \right]^2 \right)  \\
		&\leq& B \mathbb{E} \left( \sum_{k=1}^K p_k^0(\mathbf{Z}) \left[ B    
		\wedge  \log \left( \frac{p_k^0(\mathbf{Z})}{p_k(\mathbf{Z})} \right) \right] \right),  
	\end{eqnarray*}
	where the second equation is due to the fact that \(\mathbf{y}\) is one-hot label vector, and the second inequality follows from Lemma \ref{lem:trunc-ineq} in the following.
	Since for any $\mathbf{p} \in \mathcal{F}$, we have  $p_k(\mathbf{Z}) \geq C_1/T^{C_2}$ for some constants $C_1 >0$ and $0 < C_2 < 1$, and $C_T \leq  C_1 T^{1-C_2}$, it follows that
	\begin{eqnarray*}
	&& \mathbb{E} \left( \sum_{k=1}^K p_k^0(\mathbf{Z}) \left[ B \wedge 
	\log\left( \frac{p_k^0(\mathbf{Z})}{p_k(\mathbf{Z})} \right) \right] \mathbb{I} \left( p_k^0(\mathbf{Z}) \leq \frac{C_T}{T} \right) \right) \\
	&\leq&  \mathbb{E} \left( \sum_{k=1}^K p_k^0(\mathbf{Z}) \left[ B \wedge \log\left( \frac{C_T/T}{C_1/T^{C_2}} \right) \right] \mathbb{I} \left( p_k^0(\mathbf{Z}) \leq \frac{C_T}{T} \right) \right) \\
	&=& \mathbb{E} \left( \sum_{k=1}^K p_k^0(\mathbf{Z}) \left[ B \wedge \log\left( \frac{C_T}{C_1 T^{1-C_2}} \right) \right] \mathbb{I} \left( p_k^0(\mathbf{Z}) \leq \frac{C_T}{T} \right) \right) \leq  0, 
	\end{eqnarray*}
	where the last inequality follows from $\log(C_T/C_1T^{1-C_2} ) \leq B$.
	Consequently, 
	\begin{eqnarray*}
		\mathbb{E}[g_{\mathbf{p}}^2(\mathbf{Z, y})]
		&\leq& B \mathbb{E} \left( \sum_{k=1}^K p_k^0(\mathbf{Z}) \left[ B    
		\wedge  \log \left( \frac{p_k^0(\mathbf{Z})}{p_k(\mathbf{Z})} \right) \right] \right) \\
		&\leq& B \mathbb{E} \left( \sum_{k=1}^K p_k^0(\mathbf{Z}) \left[ B    
		\wedge  \log \left( \frac{p_k^0(\mathbf{Z})}{p_k(\mathbf{Z})} \right) \right]
		\mathbb{I}(p_k^0(\mathbf{Z}) \geq C_T/T ) \right) \\
		&\leq&  B_T \mathbb{E}[g_{\mathbf{p}}(\mathbf{Z, y})].
	\end{eqnarray*}
	Combining this and Chebyshev's inequality, we have
	\begin{eqnarray*}
		&& \mathbb{P} \left( \mathbb{E}[g_{\mathbf{p}^T}(\mathbf{Z, y})| \mathcal{D}_T] 
		- \frac{1}{T} \sum_{t=1}^T g_{\mathbf{p}^T}(\tilde{\mathbf{Z}}^{(t)}, \tilde{\mathbf{y}}^{(t)}) \geq \frac{\epsilon}{2} \{a + b + \mathbb{E}[g_{\mathbf{p}^T}(\mathbf{Z, y})|\mathcal{D}_T] \} \Big| \mathcal{D}_T \right) \\
		&\leq& \frac{Var\{g_{\mathbf{p}^T}(\mathbf{Z, y})| \mathcal{D}_T \}}{  
			\frac{T \epsilon^2}{4} \{a + b + \mathbb{E}[g_{\mathbf{p}^T}(\mathbf{Z, y})|\mathcal{D}_T]\}^2 } \\
		&\leq& \frac{4 \mathbb{E}[ g^2_{\mathbf{p}^T}(\mathbf{Z, y})| \mathcal{D}_T \}}{  
			T \epsilon^2 \{a + b + \mathbb{E}[g_{\mathbf{p}^T}(\mathbf{Z, y}) |\mathcal{D}_T]\}^2 } \\
		&\leq& \frac{4 B_T \mathbb{E}[ g_{\mathbf{p}^T}(\mathbf{Z, y})| \mathcal{D}_T \}}{  
			T \epsilon^2 \{a + b + \mathbb{E}[g_{\mathbf{p}^T}(\mathbf{Z, y}) |\mathcal{D}_T]\}^2 } \\
		&\leq& \frac{4 B_T}{T \epsilon^2} \frac{1}{4(a + b)} 
		=  \frac{B_T}{T(a+b) \epsilon^2},
	\end{eqnarray*}
	where the last inequality follows from the fact that $ \frac{t}{(a+t)^2} \leq \frac{1}{4a} $ for all $t \geq 0$ and $a >0$.
	Therefore, for $T \geq \frac{8 B_T}{(a + b)\epsilon^2}$,
	\begin{equation}\label{symmetry-inequ}
		\mathbb{P} \left( \mathbb{E}[g_{\mathbf{p}^T}(\mathbf{Z, y})| \mathcal{D}_T] 
		- \frac{1}{T} \sum_{t=1}^T g_{\mathbf{p}^T}(\tilde{\mathbf{Z}}^{(t)}, \tilde{\mathbf{y}}^{(t)}) \leq 
		\frac{\epsilon}{2} \{a + b + \mathbb{E}[g_{\mathbf{p}^T}(\mathbf{Z, y})|\mathcal{D}_T] \} \Big| \mathcal{D}_T \right)  \geq \frac{7}{8}.
	\end{equation}
	Consequently, 
	\begin{align*}
		& \ \mathbb{P} \left( \exists \ \mathbf{p} \in \mathcal{F}: \frac{1}{T} \sum_{t=1}^T 
		g_{\mathbf{p}}(\tilde{\mathbf{Z}}^{(t)}, \tilde{\mathbf{y}}^{(t)}) -  \frac{1}{T} \sum_{t=1}^T g_{\mathbf{p}}(\mathbf{Z}^{(t)}, \mathbf{y}^{(t)}) \geq \frac{\epsilon}{2} \{a + b + \mathbb{E} g_{\mathbf{p}}(\mathbf{Z, y}) \} \right) \\
		\geq & \  \mathbb{P} \left(\frac{1}{T} \sum_{t=1}^T g_{\mathbf{p}^T}  
		(\tilde{\mathbf{Z}}^{(t)}, \tilde{\mathbf{y}}^{(t)}) -  \frac{1}{T} \sum_{t=1}^T g_{\mathbf{p}^T}(\mathbf{Z}^{(t)}, \mathbf{y}^{(t)}) \geq \frac{\epsilon}{2} \{a + b + \mathbb{E}[g_{\mathbf{p}^T}(\mathbf{Z, y})| \mathcal{D}_T] \} \right) \\
		\geq&\   \mathbb{P} \Big( \mathbb{E}[g_{\mathbf{p}^T}(\mathbf{Z, y})| \mathcal{D}_T] -   
		\frac{1}{T} \sum_{t=1}^T g_{\mathbf{p}^T}(\mathbf{Z}^{(t)}, \mathbf{y}^{(t)})  
		\geq \epsilon \{a + b + \mathbb{E}[g_{\mathbf{p}^T}(\mathbf{Z, y})| \mathcal{D}_T] \}, \\
		&\  \mathbb{E}[g_{\mathbf{p}^T}(\mathbf{Z, y})| \mathcal{D}_T] 
		- \frac{1}{T} \sum_{t=1}^T g_{\mathbf{p}^T}(\tilde{\mathbf{Z}}^{(t)}, \tilde{\mathbf{y}}^{(t)})  \leq \frac{\epsilon}{2} \{a + b + \mathbb{E}[g_{\mathbf{p}^T}(\mathbf{Z, y})| \mathcal{D}_T] \}
		\Big) \\
		=&\ \mathbb{E} \Big\{ \mathbb{I}\Big( \mathbb{E}[g_{\mathbf{p}^T}(\mathbf{Z, y})| 
		\mathcal{D}_T] - \frac{1}{T} \sum_{t=1}^T g_{\mathbf{p}^T}(\mathbf{Z}^{(t)}, \mathbf{y}^{(t)})  
		\geq \epsilon \{a + b + \mathbb{E}[g_{\mathbf{p}^T}(\mathbf{Z, y})| \mathcal{D}_T] \} \Big) \\
		&\  \times \mathbb{I} \Big( \mathbb{E}[g_{\mathbf{p}^T}(\mathbf{Z, y})| \mathcal{D}] 
		- \frac{1}{T} \sum_{t=1}^T g_{\mathbf{p}^T}(\tilde{\mathbf{Z}}^{(t)}, \tilde{\mathbf{y}}^{(t)})  \leq \frac{\epsilon}{2} \{a + b + \mathbb{E}[g_{\mathbf{p}^T}(\mathbf{Z, y})| \mathcal{D}_T] \}
		\Big) \Big\} \\
		=&\  \mathbb{E} \Bigg\{ \mathbb{I}\Big( \mathbb{E}[g_{\mathbf{p}^T}(\mathbf{Z, y})| 
		\mathcal{D}_T] - \frac{1}{T} \sum_{t=1}^T g_{\mathbf{p}^T}(\mathbf{Z}^{(t)}, \mathbf{y}^{(t)})  
		\geq \epsilon \{a + b + \mathbb{E}[g_{\mathbf{p}^T}(\mathbf{Z, y})| \mathcal{D}_T] \} \Big) \\
		&\  \times \mathbb{E} \Big\{\mathbb{I} \Big( \mathbb{E}[g_{\mathbf{p}^T}(\mathbf{Z, y})|
		\mathcal{D}_T] 
		- \frac{1}{T} \sum_{t=1}^T g_{\mathbf{p}^T}(\tilde{\mathbf{Z}}^{(t)}, \tilde{\mathbf{y}}^{(t)})  \leq \frac{\epsilon}{2} \{a + b + \mathbb{E}[g_{\mathbf{p}^T}(\mathbf{Z, y})| \mathcal{D}_T] \}
		\Big) \Big| \mathcal{D}_T \Big\} \Bigg\} \\
		\geq&\  \frac{7}{8} \mathbb{P} \Big\{ \mathbb{E}[g_{\mathbf{p}^T}(\mathbf{Z, y})| 
		\mathcal{D}_T] - \frac{1}{T} \sum_{t=1}^T g_{\mathbf{p}^T}(\mathbf{Z}^{(t)}, \mathbf{y}^{(t)})  \geq \epsilon \{a + b + \mathbb{E}[g_{\mathbf{p}^T}(\mathbf{Z, y})| \mathcal{D}_T] \} \Big\} \\
		=&\  \frac{7}{8} \mathbb{P} \Big\{ \exists \ \mathbf{p} \in \mathcal{F} : \ 
		\mathbb{E}g_{\bf p}(\mathbf{Z, y}) 
		- \frac{1}{T} \sum_{t=1}^T g_{\bf p}(\mathbf{Z}^{(t)}, \mathbf{y}^{(t)}) \geq \epsilon \{ a + b + \mathbb{E}g_{\bf p}(\mathbf{Z, y}) \} \Big\},
	\end{align*}
	where the last inequality is due to (\ref{symmetry-inequ}). Thus, we obtain the following symmetrization inequality
	\begin{align}\label{symmetric-inequality}
		&\  \mathbb{P} \Big\{ \exists \ \mathbf{p} \in \mathcal{F} : \ 
		\mathbb{E}g_{\bf p}(\mathbf{Z, y}) - \frac{1}{T} \sum_{t=1}^T g_{\bf p} (\mathbf{Z}^{(t)}, \mathbf{y}^{(t)}) \geq \epsilon \{ a + b + \mathbb{E}g_{\bf p}(\mathbf{Z, y}) \} \Big\} \nonumber \\
		\leq&\ \frac{8}{7} \mathbb{P} \left( \exists \ \mathbf{p} \in \mathcal{F}:
		\frac{1}{T} \sum_{t=1}^T \{ g_{\mathbf{p}}(\tilde{\mathbf{Z}}^{(t)}, \tilde{\mathbf{y}}^{(t)}) - g_{\mathbf{p}}(\mathbf{Z}^{(t)}, \mathbf{y}^{(t)}) \} \geq \frac{\epsilon}{2} \{a + b + \mathbb{E} g_{\mathbf{p}}(\mathbf{Z, y}) \} \right) 
	\end{align}
	for all $T \geq \frac{8 B_T}{(a + b)\epsilon^2}$.

	{\bf Step 2.} Replacement of the expectation $\mathbb{E} g_{\mathbf{p}}(\mathbf{Z, y})$ in (\ref{symmetric-inequality}) by a sample mean of $\tilde{\mathcal{D}}_T$.  
	By (\ref{symmetric-inequality}), it follows that 
	\begin{eqnarray}\label{symmetric-decomp}
		&&  \mathbb{P} \left( \exists \ \mathbf{p} \in \mathcal{F}:  \frac{1}{T} 
		\sum_{t=1}^T g_{\mathbf{p}}(\tilde{\mathbf{Z}}^{(t)}, \tilde{\mathbf{y}}^{(t)}) -  \frac{1}{T} \sum_{t=1}^T g_{\mathbf{p}}(\mathbf{Z}^{(t)}, \mathbf{y}^{(t)}) \geq \frac{\epsilon}{2} \{a + b + \mathbb{E} g_{\mathbf{p}}(\mathbf{Z, y}) \} \right) \nonumber \\
		&\leq& \mathbb{P} \Bigg( \exists \ \mathbf{p} \in \mathcal{F}:  \frac{1}{T} 
		\sum_{t=1}^T g_{\mathbf{p}}(\tilde{\mathbf{Z}}^{(t)}, \tilde{\mathbf{y}}^{(t)}) -  \frac{1}{T} \sum_{t=1}^T g_{\mathbf{p}}(\mathbf{Z}^{(t)}, \mathbf{y}^{(t)}) \geq \frac{\epsilon}{2} \{a + b + \mathbb{E} g_{\mathbf{p}}(\mathbf{Z, y}) \},  \nonumber \\
		&&   \frac{1}{T} \sum_{t=1}^T g^2_{\mathbf{p}}(\mathbf{Z}^{(t)},  
		\mathbf{y}^{(t)}) - \mathbb{E} g^2_{\mathbf{p}}(\mathbf{Z, y}) 
		\leq \epsilon \{a + b + \frac{1}{T} \sum_{t=1}^T g^2_{\mathbf{p}}(\mathbf{Z}^{(t)},  \mathbf{y}^{(t)}) + \mathbb{E} g^2_{\mathbf{p}}(\mathbf{Z, y}) \},\nonumber  \\
		&&   \frac{1}{T} \sum_{t=1}^T g^2_{\mathbf{p}}(\tilde{\mathbf{Z}}^{(t)},  
		\tilde{\mathbf{y}}^{(t)}) - \mathbb{E} g^2_{\mathbf{p}}(\mathbf{Z, y}) 
		\leq \epsilon \{a + b + \frac{1}{T} \sum_{t=1}^T g^2_{\mathbf{p}}(\tilde{\mathbf{Z}}^{(t)}, \tilde{\mathbf{y}}^{(t)}) 
		+ \mathbb{E} g^2_{\mathbf{p}}(\mathbf{Z, y}) \}  \Bigg)  \nonumber \\
		&&   + 2 \mathbb{P} \Bigg( \exists \ \mathbf{p} \in \mathcal{F}: 
		\frac{\frac{1}{T} \sum_{t=1}^T g^2_{\mathbf{p}}(\mathbf{Z}^{(t)},  
			\mathbf{y}^{(t)}) - \mathbb{E} g^2_{\mathbf{p}}(\mathbf{Z, y}) }{a + b + \frac{1}{T} \sum_{t=1}^T g^2_{\mathbf{p}}(\mathbf{Z}^{(t)},  \mathbf{y}^{(t)}) + \mathbb{E} g^2_{\mathbf{p}}(\mathbf{Z, y})}
		> \epsilon \Bigg)  \nonumber \\
		&=:& I_1 + 2 I_2.
	\end{eqnarray}
	
	For \( I_2 \), it follows from Theorem 11.6 of \citet{gyorfi2006distribution} that
	\begin{eqnarray}\label{inequality-I2}
		I_2 
		&=& \mathbb{P} \Bigg( \exists \ \mathbf{p} \in \mathcal{F}: 
		\frac{\frac{1}{T} \sum_{t=1}^T g^2_{\mathbf{p}}(\mathbf{Z}^{(t)},  
			\mathbf{y}^{(t)}) - \mathbb{E} g^2_{\mathbf{p}}(\mathbf{Z, y}) }{a + b + \frac{1}{T} \sum_{t=1}^T g^2_{\mathbf{p}}(\mathbf{Z}^{(t)},  \mathbf{y}^{(t)}) + \mathbb{E} g^2_{\mathbf{p}}(\mathbf{Z, y})} > \epsilon \Bigg) \nonumber \\
		&\leq& 4 \mathbb{E} \Big( \mathcal{N}( \frac{(a + b)\epsilon}{5},    
		\mathcal{F}_1, \|\cdot\|_{L_1, \mathcal{D}_T}) \exp(-\frac{3 \epsilon^2 (a + b) T}{40 B_T^2 }) \Big),
	\end{eqnarray}
	where $\mathcal{F}_1 = \{g_{\mathbf{p}}(\mathbf{Z}, \mathbf{y}):  \mathbf{p} \in \mathcal{F} \}$ and 
	$\|f-g\|_{L_1, \mathcal{D}_T} = \frac{1}{T} \sum_{t=1}^T | f(\mathbf{Z}^{(t)}, \mathbf{y}^{(t)}) - g(\mathbf{Z}^{(t)}, \mathbf{y}^{(t)})|$.

	Next we discuss the first probability $I_1$ in the decomposition (\ref{symmetric-decomp}).
	Note that the second and third inequalities in $I_1$ imply that 
	\begin{eqnarray*}
		(1-\epsilon) \frac{1}{T} \sum_{t=1}^T g^2_{\mathbf{p}}(\mathbf{Z}^{(t)}, 
		\mathbf{y}^{(t)}) - \epsilon(a+b) 
		\leq (1+\epsilon) \mathbb{E} g^2_{\mathbf{p}}(\mathbf{Z, y}),  \\
		(1-\epsilon) \frac{1}{T} \sum_{t=1}^T g^2_{\mathbf{p}} (\tilde{\mathbf{Z}}^{(t)}, \tilde{\mathbf{y}}^{(t)}) - \epsilon(a+b) 
		\leq (1+\epsilon) \mathbb{E} g^2_{\mathbf{p}}(\mathbf{Z, y}).
	\end{eqnarray*}
	Combining these with the fact $ \mathbb{E}[g_{\mathbf{p}}^2(\mathbf{Z, y})]
	\leq  B_T \mathbb{E}[g_{\mathbf{p}}(\mathbf{Z, y})]$, we have
	\begin{align}\label{symmetric-I1}
		I_1 
		=&\  \mathbb{P} \Bigg( \exists \ \mathbf{p} \in \mathcal{F}: 
		\frac{1}{T} \sum_{t=1}^T g_{\mathbf{p}}(\tilde{\mathbf{Z}}^{(t)}, \tilde{\mathbf{y}}^{(t)})
		- \frac{1}{T} \sum_{t=1}^T g_{\mathbf{p}}(\mathbf{Z}^{(t)}, \mathbf{y}^{(t)})
		\geq \frac{\epsilon}{2} \left\{ a + b + \mathbb{E} g_{\mathbf{p}}(\mathbf{Z}, \mathbf{y}) \right\}, \nonumber \\
		& \quad \frac{1}{T} \sum_{t=1}^T g_{\mathbf{p}}^2(\mathbf{Z}^{(t)}, \mathbf{y}^{(t)})
		- \mathbb{E} g_{\mathbf{p}}^2(\mathbf{Z}, \mathbf{y})
		\leq \epsilon \left\{ a + b + \frac{1}{T} \sum_{t=1}^T g_{\mathbf{p}}^2(\mathbf{Z}^{(t)}, \mathbf{y}^{(t)}) 
		+ \mathbb{E} g_{\mathbf{p}}^2(\mathbf{Z}, \mathbf{y}) \right\}, \nonumber \\
		& \quad \frac{1}{T} \sum_{t=1}^T g_{\mathbf{p}}^2(\tilde{\mathbf{Z}}^{(t)}, \tilde{\mathbf{y}}^{(t)})
		- \mathbb{E} g_{\mathbf{p}}^2(\mathbf{Z}, \mathbf{y})
		\leq \epsilon \left\{ a + b + \frac{1}{T} \sum_{t=1}^T g_{\mathbf{p}}^2(\tilde{\mathbf{Z}}^{(t)}, \tilde{\mathbf{y}}^{(t)}) 
		+ \mathbb{E} g_{\mathbf{p}}^2(\mathbf{Z}, \mathbf{y}) \right\} \Bigg) \nonumber \\
		\leq&\  \mathbb{P} \Bigg\{ \exists \ \mathbf{p} \in \mathcal{F}:
		\frac{1}{T} \sum_{t=1}^T \left[ g_{\mathbf{p}}(\tilde{\mathbf{Z}}^{(t)}, \tilde{\mathbf{y}}^{(t)}) 
		- g_{\mathbf{p}}(\mathbf{Z}^{(t)}, \mathbf{y}^{(t)}) \right]
		\geq \frac{\epsilon}{2} (a + b) \nonumber \\
		& \quad + \frac{\epsilon}{2} \Bigg\{
		\frac{1 - \epsilon}{2(1 + \epsilon) B_T} \left( 
		\frac{1}{T} \sum_{t=1}^T g_{\mathbf{p}}^2(\mathbf{Z}^{(t)}, \mathbf{y}^{(t)}) 
		+ \frac{1}{T} \sum_{t=1}^T g_{\mathbf{p}}^2(\tilde{\mathbf{Z}}^{(t)}, \tilde{\mathbf{y}}^{(t)}) 
		\right)
		- \frac{2 \epsilon (a + b)}{2(1 + \epsilon) B_T}
		\Bigg\} \Bigg\} \nonumber \\
		=&\  \mathbb{P} \Bigg\{ \exists \ \mathbf{p} \in \mathcal{F}:
		\frac{1}{T} \sum_{t=1}^T \left[ g_{\mathbf{p}}(\tilde{\mathbf{Z}}^{(t)}, \tilde{\mathbf{y}}^{(t)}) 
		- g_{\mathbf{p}}(\mathbf{Z}^{(t)}, \mathbf{y}^{(t)}) \right]
		\geq \frac{\epsilon}{2} (a + b) \nonumber \\
		& \quad + \frac{\epsilon (1 - \epsilon)}{4(1 + \epsilon) B_T} \cdot 
		\frac{1}{T} \sum_{t=1}^T \left[ g_{\mathbf{p}}^2(\mathbf{Z}^{(t)}, \mathbf{y}^{(t)}) 
		+ g_{\mathbf{p}}^2(\tilde{\mathbf{Z}}^{(t)}, \tilde{\mathbf{y}}^{(t)}) \right]
		- \frac{\epsilon^2 (a + b)}{2(1 + \epsilon) B_T} \Bigg\}.
	\end{align}

	{\bf Step 3.} Second symmetrization.
	
	Now we discuss the symmetrization for the last probability in (\ref{symmetric-I1}). Let $\{U_t: t =1, \cdots, T\}$ be an i.i.d sample following the uniform distribution on $\{-1, 1\}$.
	It is easy to see that 
	$g_{\mathbf{p}}(\tilde{\mathbf{Z}}^{(t)}, \tilde{\mathbf{y}}^{(t)}) -  g_{\mathbf{p}}(\mathbf{Z}^{(t)}, \mathbf{y}^{(t)})$ 
	has the same distribution with 
	$ U_t \{ g_{\mathbf{p}}(\tilde{\mathbf{Z}}^{(t)}, \tilde{\mathbf{y}}^{(t)}) - g_{\mathbf{p}}(\mathbf{Z}^{(t)}, \mathbf{y}^{(t)}) \} $.
	Therefore,
	\begin{eqnarray}\label{sec-symmetric-inequ}
		&&  \mathbb{P} \Bigg\{ \exists \ \mathbf{p} \in \mathcal{F}:  \frac{1}{T} 
		\sum_{t=1}^T \{ g_{\mathbf{p}}(\tilde{\mathbf{Z}}^{(t)}, \tilde{\mathbf{y}}^{(t)}) - g_{\mathbf{p}}(\mathbf{Z}^{(t)}, \mathbf{y}^{(t)}) \} \geq \frac{\epsilon}{2} (a + b) \nonumber \\
		&&   + \frac{\epsilon(1-\epsilon)}{4(1+\epsilon)B_T} \frac{1}{T} \sum_{t=1}^T  
		\{g^2_{\mathbf{p}}(\mathbf{Z}^{(t)}, \mathbf{y}^{(t)}) + g^2_{\mathbf{p}}(\tilde{\mathbf{Z}}^{(t)}, \tilde{\mathbf{y}}^{(t)}) \} - \frac{\epsilon^2 (a+b)}{2(1+\epsilon)B_T}  \Bigg\} \nonumber \\
		&=&  \mathbb{P} \Bigg\{ \exists \ \mathbf{p} \in \mathcal{F}:  \frac{1}{T} 
		\sum_{t=1}^T U_t \{ g_{\mathbf{p}}(\tilde{\mathbf{Z}}^{(t)}, \tilde{\mathbf{y}}^{(t)}) - g_{\mathbf{p}}(\mathbf{Z}^{(t)}, \mathbf{y}^{(t)}) \} \geq \frac{\epsilon}{2} (a + b)  \nonumber \\
		&&   + \frac{\epsilon(1-\epsilon)}{4(1+\epsilon)B_T} \frac{1}{T} \sum_{t=1}^T  
		\{g^2_{\mathbf{p}}(\mathbf{Z}^{(t)}, \mathbf{y}^{(t)}) + g^2_{\mathbf{p}}(\tilde{\mathbf{Z}}^{(t)}, \tilde{\mathbf{y}}^{(t)}) \} - \frac{\epsilon^2 (a+b)}{2(1+\epsilon)B_T}  \Bigg\} \nonumber \\
		&\leq& 2 \mathbb{P} \Bigg\{ \exists \ \mathbf{p} \in \mathcal{F}:  \Big| \frac{1}{T} 
		\sum_{t=1}^T U_t g_{\mathbf{p}}(\mathbf{Z}^{(t)}, \mathbf{y}^{(t)}) \Big| \geq \frac{\epsilon}{4} (a + b) \nonumber \\
		&&   + \frac{\epsilon(1-\epsilon)}{4(1+\epsilon)B_T} \frac{1}{T} \sum_{t=1}^T  
		g^2_{\mathbf{p}}(\mathbf{Z}^{(t)}, \mathbf{y}^{(t)}) - \frac{\epsilon^2 (a+b)}{4(1+\epsilon)B_T}  \Bigg\}. 
	\end{eqnarray}
	
	Now we study the last probability in (\ref{sec-symmetric-inequ}), conditioned on $\mathcal{D}_T$. 
    Fix $\mathcal{D}_T = \{ (\mathbf{Z}^{(t)}, \mathbf{y}^{(t)}): t = 1, \dots, T \}$, 
    let $\delta > 0$, and define $\mathcal{G}_{\delta} = \{ g_i : i = 1, \dots, M \}$ 
    as a $\delta$-cover of $\mathcal{F}_1 = \{ g_{\mathbf{p}} : \mathbf{p} \in \mathcal{F} \}$ 
    with respect to $\|\cdot\|_{L_1,\mathcal{D}_T}$, 
    where $M = N(\delta, \mathcal{F}_1, \|\cdot\|_{L_1,\mathcal{D}_T})$.

	Without loss of generalization, we assume that $\mathcal{G}_{\delta} \subset \mathcal{F}_1$. 
	For a given $\mathbf{p} \in \mathcal{F}$, there exist $g_i \in \mathcal{F}_1$ such that 
	$$ \frac{1}{T} \sum_{t=1}^T \big| g_i(\mathbf{Z}^{(t)}, \mathbf{y}^{(t)}) - g_{\mathbf{p}}(\mathbf{Z}^{(t)}, \mathbf{y}^{(t)}) \big| \leq \delta.
	$$
	This yields 
	\begin{eqnarray*}
		\Big| \frac{1}{T} \sum_{t=1}^T U_t g_{\mathbf{p}}(\mathbf{Z}^{(t)}, 
		\mathbf{y}^{(t)}) \Big|  
		&\leq& \Big| \frac{1}{T} \sum_{t=1}^T U_t g_i(\mathbf{Z}^{(t)}, \mathbf{y}^{(t)}) \Big|
		+ \frac{1}{T} \sum_{t=1}^T \big| g_i(\mathbf{Z}^{(t)}, \mathbf{y}^{(t)}) - g_{\mathbf{p}}(\mathbf{Z}^{(t)}, \mathbf{y}^{(t)}) \big| \\
		&\leq& \Big| \frac{1}{T} \sum_{t=1}^T U_t g_i(\mathbf{Z}^{(t)}, \mathbf{y}^{(t)}) \Big|
		+ \delta.
	\end{eqnarray*}
	Note that $|g_i(\mathbf{Z}^{(t)}, \mathbf{y}^{(t)})| \leq B_T$ and $|g_{\mathbf{p}}(\mathbf{Z}^{(t)}, \mathbf{y}^{(t)})| \leq B_T$. 
	Consequently, 
	\begin{eqnarray*}
		\frac{1}{T} \sum_{t=1}^T g^2_{\mathbf{p}}(\mathbf{Z}^{(t)}, \mathbf{y}^{(t)})
		&=& \frac{1}{T} \sum_{t=1}^T g^2_i(\mathbf{Z}^{(t)}, \mathbf{y}^{(t)})
		+  \frac{1}{T} \sum_{t=1}^T \{ g^2_{\mathbf{p}}(\mathbf{Z}^{(t)}, \mathbf{y}^{(t)}) - g^2_i(\mathbf{Z}^{(t)}, \mathbf{y}^{(t)}) \}  \\
		&\geq& \frac{1}{T} \sum_{t=1}^T g^2_i(\mathbf{Z}^{(t)}, \mathbf{y}^{(t)}) 
		- 2B_T \frac{1}{T} \sum_{t=1}^T |g_{\mathbf{p}}(\mathbf{Z}^{(t)}, \mathbf{y}^{(t)}) - g_i(\mathbf{Z}^{(t)}, \mathbf{y}^{(t)})| \\
		&\geq& \frac{1}{T} \sum_{t=1}^T g^2_i(\mathbf{Z}^{(t)}, \mathbf{y}^{(t)}) - 2 \delta 
		B_T.
	\end{eqnarray*}
	It follows that 
	\begin{eqnarray*}
		&& \mathbb{P} \Bigg\{ \exists \ \mathbf{p} \in \mathcal{F}:  
		\left| \frac{1}{T} \sum_{t=1}^T U_t g_{\mathbf{p}}(\mathbf{Z}^{(t)}, \mathbf{y}^{(t)}) \right| 
		\geq \frac{\epsilon}{4} (a + b) 
		+ \frac{\epsilon(1 - \epsilon)}{4(1 + \epsilon) B_T} 
		\cdot \frac{1}{T} \sum_{t=1}^T g_{\mathbf{p}}^2(\mathbf{Z}^{(t)}, \mathbf{y}^{(t)}) \\
		&& \qquad\qquad\qquad\qquad
		- \frac{\epsilon^2 (a + b)}{4(1 + \epsilon) B_T} 
		\ \Big| \ \mathcal{D}_T \Bigg\} \\[1ex]
		&& \leq \mathbb{P} \Bigg\{ \exists \ g_i \in \mathcal{G}_\delta: 
		\left| \frac{1}{T} \sum_{t=1}^T U_t g_i(\mathbf{Z}^{(t)}, \mathbf{y}^{(t)}) \right| 
		\geq -\delta + \frac{\epsilon}{4}(a + b) 
		- \frac{\epsilon^2(a + b)}{4(1 + \epsilon) B_T} \\
		&& \qquad\qquad\qquad\qquad
		+ \frac{\epsilon(1 - \epsilon)}{4(1 + \epsilon) B_T} \cdot 
		\left( \frac{1}{T} \sum_{t=1}^T g_i^2(\mathbf{Z}^{(t)}, \mathbf{y}^{(t)}) - 2 \delta B_T \right)
		\ \Big| \ \mathcal{D}_T \Bigg\} \\[1ex]
		&& \leq N(\delta, \mathcal{F}_1, \|\cdot\|_{L_1, \mathcal{D}_T}) \cdot 
		\max_{g_i \in \mathcal{G}_\delta} 
		\mathbb{P} \Bigg\{ 
		\left| \frac{1}{T} \sum_{t=1}^T U_t g_i(\mathbf{Z}^{(t)}, \mathbf{y}^{(t)}) \right| 
		\geq -\delta + \frac{\epsilon}{4}(a + b) 
		- \frac{\epsilon^2(a + b)}{4(1 + \epsilon) B_T} \\
		&& \qquad\qquad\qquad\qquad
		+ \frac{\epsilon(1 - \epsilon)}{4(1 + \epsilon) B_T} \cdot 
		\left( \frac{1}{T} \sum_{t=1}^T g_i^2(\mathbf{Z}^{(t)}, \mathbf{y}^{(t)}) - 2 \delta B_T \right) 
		\ \Big| \ \mathcal{D}_T \Bigg\}.
	\end{eqnarray*}
	
	Now set $\delta = \epsilon b/8$. Since $B_T \geq 2$ and $0 < \epsilon \leq 1/2$, it follows that
	$$ -\delta + \frac{\epsilon b}{4}- \frac{\epsilon^2 b} {4(1+\epsilon)B_T} - \frac{\epsilon(1-\epsilon) \delta}{2(1+\epsilon)} > 0.  $$
	
	Consequently, 
	\begin{eqnarray}\label{max-inequ}
		&& \mathbb{P} \Bigg\{ \exists \ \mathbf{p} \in \mathcal{F}:
		\left| \frac{1}{T} \sum_{t=1}^T U_t g_{\mathbf{p}}(\mathbf{Z}^{(t)}, \mathbf{y}^{(t)}) \right| 
		\geq \frac{\epsilon}{4} (a + b) 
		+ \frac{\epsilon(1 - \epsilon)}{4(1 + \epsilon) B_T} \cdot 
		\frac{1}{T} \sum_{t=1}^T g_{\mathbf{p}}^2(\mathbf{Z}^{(t)}, \mathbf{y}^{(t)}) \nonumber \\
		&& \qquad\qquad\qquad\qquad
		- \frac{\epsilon^2 (a + b)}{4(1 + \epsilon) B_T} 
		\ \Big| \ \mathcal{D}_T \Bigg\} \nonumber \\
		&& \leq N\left( \frac{\epsilon b}{8}, \mathcal{F}_1, \|\cdot\|_{L_1, \mathcal{D}_T} \right) 
		\cdot \max_{g_i \in \mathcal{G}_{\epsilon b / 8}} 
		\mathbb{P} \Bigg\{ 
		\left| \frac{1}{T} \sum_{t=1}^T U_t g_i(\mathbf{Z}^{(t)}, \mathbf{y}^{(t)}) \right| 
		\geq \frac{\epsilon a}{4} 
		- \frac{\epsilon^2 a}{4(1 + \epsilon) B_T} \nonumber \\
		&& \qquad\qquad\qquad\qquad
		+ \frac{\epsilon(1 - \epsilon)}{4(1 + \epsilon) B_T} \cdot 
		\frac{1}{T} \sum_{t=1}^T g_i^2(\mathbf{Z}^{(t)}, \mathbf{y}^{(t)})
		\ \Big| \ \mathcal{D}_T \Bigg\}.
	\end{eqnarray}

	{\bf Step 4.}  Applying Bernstein's inequality.
	
	We write 
	\begin{eqnarray*}
		&& W_t = U_t g_i(\mathbf{Z}^{(t)}, \mathbf{y}^{(t)}), \quad
		\sigma^2 = \frac{1}{T} \sum_{t=1}^T var\{W_t| \mathbf{Z}^{(t)}, 
		\mathbf{y}^{(t)}\},  \\
		&& A_1 = \frac{\epsilon a}{4} - \frac{\epsilon^2 a} 
		{4(1+\epsilon)B_T}, \quad 
		A_2 = \frac{\epsilon(1-\epsilon)}{4(1+\epsilon)B_T}
	\end{eqnarray*}
	Since $g_i \in \mathcal{F}_1$, $B_T \geq 2$, and $0 < \epsilon \leq 1/2$, it is easy to see that  $|W_t| \leq B_T$ and $A_1, A_2 > 0$. 
	Applying Bernstein's inequality, we have
	\begin{eqnarray}\label{condit-inequ}
		&&  \mathbb{P} \Bigg\{ \Big| \frac{1}{T} \sum_{t=1}^T U_t g_i(\mathbf{Z}^{(t)},      \mathbf{y}^{(t)}) \Big| \geq \frac{\epsilon a}{4} 
		- \frac{\epsilon^2 a}{4(1+\epsilon)B_T} 
		+ \frac{\epsilon(1-\epsilon)}{4(1+\epsilon)B_T} \frac{1}{T} 
		\sum_{t=1}^T g^2_i(\mathbf{Z}^{(t)}, \mathbf{y}^{(t)}) \Big| \mathcal{D}_T \Bigg\} \nonumber \\
		&=& \mathbb{P} \Bigg\{ \Big| \frac{1}{T} \sum_{t=1}^T W_t \Big| \geq A_1 + A_2 
		\sigma^2 \Big| \mathcal{D}_T \Bigg\} \nonumber \\
		&\leq& 2 \exp\Bigg( - \frac{T (A_1 + A_2 \sigma^2)^2}{2 \sigma^2 + \frac{2}{3} 
			B_T (A_1 + A_2 \sigma^2) } \Bigg) \nonumber \\
		&=&    2 \exp\Bigg(- \frac{TA_2^2}{2 B_T A_2/3} \cdot \frac{(A_1/A_2 + 
			\sigma^2)^2}{A_1/A_2 + (1 + 3/B_T A_2) \sigma^2} \Bigg).
	\end{eqnarray}
	Some elementary calculations show that for any $c_1, x >0$ and $c_2 \geq 2$, 
	\begin{eqnarray*}
		\frac{(c_1+x)^2}{c_1 + c_2x} 
		\geq \frac{[c_1+ c_1(c_2-2)/c_2]^2}{c_1 + c_2 \cdot c_1(c_2-2)/c_2} 
		= 4 c_1 \frac{c_2-1}{c_2^2}. 
	\end{eqnarray*}
	Now let $c_1 = A_1/A_2$, $x = \sigma^2$, and $c_2 = 1 + 3/B_T A_2$. Since $0 < \epsilon \leq 1/2$, it follows that
	$$ 1+ 3/B_T A_2 = 1+ \frac{12(1+\epsilon)}{\epsilon (1-\epsilon)}
	\geq 1 + 48 > 2. $$
	Consequently, 
	\begin{eqnarray*}
		\frac{TA_2^2}{2 B_T A_2/3} \cdot \frac{(A_1/A_2 + \sigma^2)^2}{A_1/A_2 + (1+ 3/B_T A_2) \sigma^2} 
		&\geq& \frac{TA_2^2}{2 B_T A_2/3} \cdot 4 \frac{A_1}{A_2} \frac{3/B_T A_2}{(1 + 3/B_T A_2)^2} \\
		&=&    \frac{18 T A_1A_2}{(3 + B_T A_2)^2}. 
	\end{eqnarray*}
	Also note that $ A_1 = \frac{\epsilon a}{4} - \frac{\epsilon^2 a}
	{4(1+\epsilon)B_T} $ 
	and $A_2 = \frac{\epsilon(1-\epsilon)}{4(1+\epsilon)B_T} $. 
	As $B_T \geq 2$ and $0 < \epsilon \leq 1/2$, it follows that
	$$ A_1 = \frac{\epsilon a}{4} \Big(1 - \frac{\epsilon}
	{(1+\epsilon)B_T} \Big) \geq \frac{\epsilon a}{4} \frac{3}{4} 
	= \frac{3\epsilon a}{16}. $$
	Then we obtain that
	\begin{eqnarray*}
		\frac{18 T A_1A_2}{(3 + B_T A_2)^2}
		&\geq& 18T \cdot \frac{3\epsilon a}{16} \cdot \frac{\epsilon(1-  
			\epsilon)}{4(1+\epsilon)B_T} \cdot \frac{1}{(3 + \frac{\epsilon(1-\epsilon)}{4(1+\epsilon)})^2 } \\
		&\geq& \frac{27T \epsilon^2(1-\epsilon) a}{32(1+\epsilon)B_T}
		\cdot \frac{1}{(3+1/16)^2}  \\
		&=&    \frac{27 \cdot 16^2}{32 \cdot 49^2} \frac{T \epsilon^2(1-\epsilon) 
			a}{(1+\epsilon)B_T} \\
		&\geq& \frac{\epsilon^2(1-\epsilon)}{12 (1+\epsilon)B_T} a T.
	\end{eqnarray*}
	Combining this with (\ref{condit-inequ}), we obtain
	\begin{eqnarray*}
		&&  \mathbb{P} \Bigg\{ \Big| \frac{1}{T} \sum_{t=1}^T U_t g_i(\mathbf{Z}^{(t)},      \mathbf{y}^{(t)}) \Big| \geq \frac{\epsilon a}{4} 
		- \frac{\epsilon^2 a}{4(1+\epsilon)B_T} 
		+ \frac{\epsilon(1-\epsilon)}{4(1+\epsilon)B_T} \frac{1}{T} 
		\sum_{t=1}^T g^2_i(\mathbf{Z}^{(t)}, \mathbf{y}^{(t)}) \Big| \mathcal{D}_T \Bigg\} \\
		&\leq&  2 \exp\Bigg(- \frac{\epsilon^2(1-\epsilon)}{12 (1+\epsilon)B_T} 
		a T \Bigg).
	\end{eqnarray*}
	Applying this result with (\ref{symmetric-I1}), (\ref{sec-symmetric-inequ}), and (\ref{max-inequ}), we obtain
	\begin{equation}\label{inequality-I1}
		I_1 \leq 4 \mathbb{E} \Bigg( N( \epsilon b/8, \mathcal{F}_1, 
		\|\cdot\|_{L_1, \mathcal{D}_T}) \exp\Big(- \frac{\epsilon^2(1-\epsilon)}{12 (1+\epsilon)B_T} 
		a T \Big) \Bigg).
	\end{equation}

	{\bf Step 5.} Bounding the covering numbers in (\ref{inequality-I2}) and (\ref{inequality-I1}). 
	
	Recall that $\mathcal{F}_1 = \{g_{\mathbf{p}}(\mathbf{Z}, \mathbf{y}) : \mathbf{p} \in \mathcal{F} \}$, where 
	$$ 
	g_{\mathbf{p}}(\mathbf{Z}, \mathbf{y}) = \mathbf{y}^\top \left( B \wedge \log   
	\frac{\mathbf{p}_0(\mathbf{Z})}{\mathbf{p}(\mathbf{Z})} \right)_{\geq \frac{C_T}{T}} 
	= \sum_{k=1}^K y_k \left[ B \wedge \log \left(   
	\frac{p_k^0(\mathbf{Z})}{p_k(\mathbf{Z})} \right) \right] 
	\mathbb{I} \left( p_k^0(\mathbf{Z}) \geq \frac{C_T}{T} \right). 
	$$
	Given $\mathcal{Z}_T = \{ \mathcal{Z}_1, \dots, \mathcal{Z}_T \}$, let $\log(\mathbf{p}^1), \dots, \log(\mathbf{p}^M)$ with 
	$ M= \mathcal{N}(\delta, \log({\mathcal{F}}), \|\cdot\|_{\infty, \mathcal{Z}_T})$ 
	be a $\delta$-covering of $\log(\mathcal{F})$ with respect to the norm $\|\cdot\|_{\infty, \mathcal{Z}_T}$. Without loss of generality, we assume that $\mathbf{p}^i \in \mathcal{F}$.
	For any $\mathbf{p} \in \mathcal{F}$, there exists an $\mathbf{p}^i$ such that 
	$$ \| \log(\mathbf{p}) - \log(\mathbf{p}^i) \|_{\infty, \mathcal{Z}_T} \leq \delta, $$
	where 
	$$ \| \log(\mathbf{p}) - \log(\mathbf{p}^i) \|_{\infty, \mathcal{Z}_T} = \max_{1 \leq t \leq T} \max_{1 \leq k \leq K} |\log(\mathbf{p}_k(Z^{(t)})) - \log(\mathbf{p}^i_k(Z^{(t)}))|. $$
	Consequently, 	
	\begin{eqnarray*}
		\| g_{\mathbf{p}} - g_{\mathbf{p}^i} \|_{L_1, \mathcal{D}_T} 
		&=&  \frac{1}{T} \sum_{t=1}^T |g_{\mathbf{p}}(\mathbf{Z}^{(t)}, \mathbf{y}^{(t)}) 
		-g_{\mathbf{p}^i}(\mathbf{Z}^{(t)}, \mathbf{y}^{(t)})|    \\
		&=&  \frac{1}{T} \sum_{t=1}^T \Bigg| \sum_{k=1}^K y_k^{(t)} \left[B 
		\wedge \log\left(\frac{p_k^0(\mathbf{Z}^{(t)})}{p_k(\mathbf{Z}^{(t)})} \right) \right] \mathbb{I} \left( p_k^0(\mathbf{Z}^{(t)}) \geq \frac{C_T}{T} \right) \\
		&&   - \sum_{k=1}^K y_k^{(t)} \left[B \wedge \log\left( 
		\frac{p_k^0(\mathbf{Z}^{(t)})} {p^i_k(\mathbf{Z}^{(t)})} \right) \right] \mathbb{I} \left( p_k^0(\mathbf{Z}^{(t)}) \geq \frac{C_T}{T} \right) \Bigg|  \\
		&\leq& \frac{1}{T} \sum_{t=1}^T \sum_{k=1}^K y_k^{(t)} 
		\Bigg| B\wedge \log\left(\frac{p_k^0(\mathbf{Z}^{(t)})}{p_k(\mathbf{Z}^{(t)})} \right) - B \wedge \log\left( \frac{p_k^0(\mathbf{Z}^{(t)})} {p^i_k(\mathbf{Z}^{(t)})} \right) \Bigg| \\
		&\leq& \frac{1}{T} \sum_{t=1}^T \sum_{k=1}^K y_k^{(t)} 
		\Big| \log\left( p_k(\mathbf{Z}^{(t)}) \right) - \log\left( p^i_k(\mathbf{Z}^{(t)}) \right) \Big| \\
		&\leq& \| \log(\mathbf{p}) - \log(\mathbf{p}^i) \|_{\infty, \mathcal{Z}_T} 
		\leq \delta.
	\end{eqnarray*}
	This implies that 
	$$ \mathcal{N}( \delta, \mathcal{F}_1, \|\cdot\|_{L_1, \mathcal{D}_T}) 
	\leq \mathcal{N}(\delta, \log({\mathcal{F}}), \|\cdot\|_{\infty, \mathcal{Z}_T}).
	$$
	
	{\bf Step 6.} Conclusion.
	
	The assertions (\ref{symmetric-inequality}), (\ref{inequality-I2}) and (\ref{inequality-I1}) imply, for $T \geq \frac{8 B_T}{(a + b)\epsilon^2}$, 
	\begin{eqnarray*}
		&&  \mathbb{P} \Big\{ \exists \ \mathbf{p} \in \mathcal{F} : \ 
		\mathbb{E}g_{\bf p}(\mathbf{Z, y}) - \frac{1}{T} \sum_{t=1}^T g_{\bf p} (\mathbf{Z}^{(t)}, \mathbf{y}^{(t)}) \geq \epsilon \{ a + b + \mathbb{E}g_{\bf p}(\mathbf{Z, y}) \} \Big\} \nonumber \\
		&\leq& \frac{32}{7} \mathbb{E} \Bigg( \mathcal{N}( \epsilon b/8, \mathcal{F}_1, 
		\|\cdot\|_{L_1, \mathcal{D}_T}) \exp\Big(- \frac{\epsilon^2(1-\epsilon)}{12 (1+\epsilon)B_T} a T \Big) \Bigg) \nonumber \\
		&&  + \frac{32}{7} \mathbb{E} \Bigg( \mathcal{N}( \frac{(a+b)\epsilon}{5},
		\mathcal{F}_1, \|\cdot\|_{L_1, \mathcal{D}_T}) \exp(-\frac{3 \epsilon^2 (a+b) T}{40 B_T^2}) \Bigg)  \nonumber \\
		&\leq& \frac{32}{7} \mathbb{E} \Bigg( \mathcal{N}( \epsilon b/8, 
		\log(\mathcal{F}), \|\cdot\|_{\infty, \mathcal{Z}_T}) \exp\Big(- \frac{\epsilon^2(1-\epsilon)}{12 (1+\epsilon)B_T} a T \Big) \Bigg) \nonumber \\
		&&  + \frac{32}{7} \mathbb{E} \Bigg( \mathcal{N}( \frac{(a+b)\epsilon}{5},
		\log(\mathcal{F}), \|\cdot\|_{\infty, \mathcal{Z}_T}) \exp(-\frac{3 \epsilon^2 (a+b) T}{40 B_T^2}) \Bigg)  \nonumber \\
		&\leq& 2 \cdot \frac{32}{7} \sup_{\mathcal{Z}_T} \mathcal{N}( \epsilon b/8, 
		\log(\mathcal{F}), \|\cdot\|_{\infty, \mathcal{Z}_T}) \exp \Bigg( -\frac{\epsilon^2 (1-\epsilon) a T}{15 (1+\epsilon) B_T^2} \Bigg)  \nonumber \\
		&\leq& 10 \sup_{\mathcal{Z}_T} \mathcal{N}( \frac{\epsilon b}{8}, 
		\log(\mathcal{F}), \|\cdot\|_{\infty, \mathcal{Z}_T}) \exp \Bigg( -\frac{\epsilon^2 (1-\epsilon) a T}{15 (1+\epsilon) B_T^2} \Bigg). 
	\end{eqnarray*}
	For $T < \frac{8 B_T}{(a + b)\epsilon^2}$, it is easy to see that 
	$$
	\exp \Bigg( -\frac{\epsilon^2 (1-\epsilon) a T}{15 (1+\epsilon) B_T^2} \Bigg) \geq 
	\exp \Bigg( -\frac{\epsilon^2 (1-\epsilon) a}{15 (1+\epsilon) B_T^2} 
	\cdot \frac{8 B_T}{(a + b)\epsilon^2}
	\Bigg) \geq 
	\exp(-\frac{8}{30}) \geq \frac{1}{10}.
	$$
	Then for all $T$, 
	\begin{eqnarray*}
		&&  \mathbb{P} \Big\{ \exists \ \mathbf{p} \in \mathcal{F} : \ 
		\mathbb{E}g_{\bf p}(\mathbf{Z, y}) - \frac{1}{T} \sum_{t=1}^T g_{\bf p} (\mathbf{Z}^{(t)}, \mathbf{y}^{(t)}) \geq \epsilon \{ a + b + \mathbb{E}g_{\bf p}(\mathbf{Z, y}) \} \Big\} \\
		&\leq& 10 \sup_{\mathcal{Z}_T} \mathcal{N}( \frac{\epsilon b}{8}, 
		\log(\mathcal{F}), \|\cdot\|_{\infty, \mathcal{Z}_T}) \exp \Bigg( -\frac{\epsilon^2 (1-\epsilon) a T}{15 (1+\epsilon) B_T^2} \Bigg). 
	\end{eqnarray*}
	This completes the proof of Lemma \ref{lem:excess-risk-bound}.
\end{proof}

The following lemma, adapted from Lemma 3.7 in \cite{bos2022convergence}, is used in Step 1 of the proof of Lemma \ref{lem:excess-risk-bound} with $m=2$. We call $(p_1,\dots,p_K) \in \mathbb{R}^K$ a probability vector if $\sum_{i=1}^K p_i = 1$ and $p_i \geq 0$ for all $1 \leq i \leq K$. 

\begin{lemma}\label{lem:trunc-ineq}
Given $B > 2$ and the integer $m \geq 2$. Then, for any two probability vectors $(p_1,\dots,p_K)$ and $(q_1,\dots,q_K)$, we have
\begin{equation*}
\sum_{k=1}^K p_k \left| B \wedge \log\frac{p_k}{q_k} \right|^m  \leq 
\max\left\{m!, \frac{B^m}{B - 1}\right\} \sum_{k=1}^K p_k \left(B \wedge \log\frac{p_k}{q_k}\right).
\end{equation*}
\end{lemma}

\section{Proof of Theorem~\ref{thm:condition_approximation}} \label{app:proof_condition_approx}
The following lemma is a key component in the proof of Theorem~\ref{thm:condition_approximation}. It provides an efficient approximation of the logarithm function by a scalar-valued neural network with controlled complexity. This result is adapted from Theorem~4.1 of \citet{bos2022convergence}, with some modifications to the notations.

\begin{lemma}[Theorem~4.1 of \citet{bos2022convergence}]\label{lem:scalar-log}
For any \(\varepsilon \in (0,1)\) and \(\beta > 0\), there exists a scalar-valued ReLU neural network \(G : [0,1] \to \mathbb{R}\) with the depth \(L\), width \(W\), and size \(S\) (the total number of nonzero parameters) satisfying
\[
L \lesssim \log\left(\frac{1}{\varepsilon}\right), \quad
W \lesssim \left(\frac{1}{\varepsilon}\right)^{\frac{1}{\beta}}, \quad
S \lesssim \left(\frac{1}{\varepsilon}\right)^{\frac{1}{\beta}} \log\left(\frac{1}{\varepsilon}\right),
\]
such that
\[
\left| e^{G(x)} - x \right| \leq \varepsilon, \quad \text{and} \quad G(x) \geq \log(\varepsilon), \quad \text{for all } x \in [0,1].
\]
\end{lemma}

\begin{remark}\label{upp-bound-mlp}
Note that the neural network $G$ constructed in Lemma~\ref{lem:scalar-log} also satisfies $G(x) \leq \log(x+\varepsilon) \leq \log(2)$. This upper bound of $G(x)$ will be used in the proof of Theorem~\ref{thm:condition_approximation}.
\end{remark}

\begin{proof}[Proof of Theorem~\ref{thm:condition_approximation}]
By composing the neural networks established in  Lemma~\ref{lem:transformer_approximation} and Lemma~\ref{lem:scalar-log}, we construct a neural network \(\mathbf{G}(\boldsymbol{\phi}) = \big(G(\phi_1), \dots, G(\phi_K)\big)\) such that, for any $ \varepsilon \in (0,1)$, 
\begin{eqnarray*}
	  \| \boldsymbol{\phi} - \mathbf{p}_0 \|_{\infty} \leq \varepsilon 
	  \quad {\rm and} \quad 
	  \| e^{\mathbf{G}(\boldsymbol{\phi})} - \boldsymbol{\phi} \|_{\infty} \leq \varepsilon,
\end{eqnarray*}
where \(\boldsymbol{\phi} = (\phi_1, \phi_2, \dots, \phi_K)^\top \in \mathcal{T}(M, D, H, W_1, B_1, S_1, \gamma ) \) and $G \in \mathcal{F}_{\mathrm{id}}(L, W, S)$ are provided in Lemma~\ref{lem:transformer_approximation} and Lemma~\ref{lem:scalar-log}, respectively. Here we temporarily allow the notation $\mathcal{F}_{\mathrm{id}}(L, W, S)$ to refer to scalar-valued networks with depth \(L\), width \(W\), and size \(S\).
It is readily seen that $\mathbf{G} \in \mathcal{F}_{\mathrm{id}}(L, KW, KS)$. 
Consequently, for any \(k = 1, \dots, K\), we have
\[
\left\|e^{G(\phi_k)} - p_k^0\right\|_{\infty} \leq \left\|e^{G(\phi_k)} - \phi_k \right\|_{\infty} + \left\|\phi_k - p_k^0\right\|_{\infty} \leq 2\varepsilon.
\]
Now, define the vector-valued function \(\widetilde{\mathbf{q}} = \big(\widetilde{q}_1, \cdots, \widetilde{q}_K\big)^\top\) component-wisely as
\[
\widetilde{q}_k(\mathbf{Z}) = \frac{e^{G(\phi_k(\mathbf{Z}))}}{\sum_{j=1}^{K} e^{G(\phi_j(\mathbf{Z}))}}, \quad k = 1, \dots, K.
\]
To bound \(\left\|\widetilde{q}_k - p_k^0\right\|_{\infty}\), note that \(\mathbf{p}_0 = \big(p_1^0, \dots, p_K^0\big)\) is a probability vector, and apply the triangle inequality to obtain
\begin{eqnarray*}
	 \left\|\widetilde{q}_k - p_k^0\right\|_{\infty}  
	 &\leq& \left\|e^{G(\phi_k)} \left(\frac{1}{\sum_{j=1}^K e^{G(\phi_j)}} - 1 \right)
	        \right\|_{\infty} + \left\|e^{G(\phi_k)} - p_k^0\right\|_{\infty} \\
	 &=& \left\|e^{G(\phi_k)} \left(\frac{\sum_{\ell=1}^K p_\ell^0}{\sum_{j=1}^K 
	     e^{G(\phi_j)}} - 1\right)\right\|_{\infty} + \left\|e^{G(\phi_k)} - p_k^0\right\|_{\infty} \\
	 &\leq& \sum_{\ell=1}^K \left\|p_\ell^0 - e^{G(\phi_\ell)} 
	     \right\|_{\infty}  \left\|\frac{e^{G(\phi_k)}}{\sum_{j=1}^K e^{G(\phi_j)}} \right\|_{\infty} + \left\|e^{G(\phi_k)} - p_k^0\right\|_{\infty} \\
	 &\leq& 2K\varepsilon + 2\varepsilon \\
     &=& 2\varepsilon(K + 1).
\end{eqnarray*}
	 
Next, we investigate the low bound of $ \widetilde{q}_k(\mathbf{Z}) $. 
According to Lemma ~\ref{lem:scalar-log} and Remark~\ref{upp-bound-mlp}, we have 
$\log(\varepsilon) \leq G(\phi_k) \leq \log(2)$. 
Consequently,
\[
\widetilde{q}_k(\mathbf{Z}) 
= \frac{e^{G(\phi_k(\mathbf{Z}))}}{\sum_{j=1}^{K} e^{G(\phi_j(\mathbf{Z}))}}
\geq \frac{\varepsilon}{2K}.
\]
This completes the proof of Theorem~\ref{thm:condition_approximation}.
\end{proof}

\begin{remark}\label{lowbound-F}
According to the proof of Theorem~4.1 of \citet{bos2022convergence}, for any ReLU neural network $G_0$ and $\varepsilon \in (0,1)$, we can always construct a network  $G(x)$ by adding additional ReLU layers to $G_0$ such that 
$ G(x) = \{ G_0(x) \vee \log(\varepsilon) \} \wedge \log(2) $, 
where $a \vee b = \max\{a, b\}$ and $a \wedge b = \min\{a, b\}$. 
If $|e^{G_0(x)} - x| \leq \varepsilon$ for all $x \in [0,1]$, then $|e^{G(x)} - x| \leq \varepsilon$ for all $x \in [0,1]$. This construction also ensures $ \log(\varepsilon) \leq G(x) \leq \log(2) $. Therefore, by applying this construction to the MLP architecture in the network of $\mathcal{F}$ and adapting the arguments in proof of Theorem~\ref{thm:condition_approximation}, we can ensure that for any network $\mathbf{p} = (p_1, \dots, p_K) \in \mathcal{F}$,
$$ p_k \geq \frac{\varepsilon}{2K}, \quad k=1, \dots, K, $$
where $\mathcal{F} = \mathcal{F}(M, H, D, W_1, B_1, S_1, \gamma, L, W, S)$ is defined in~(\ref{eq:function_class}).
\end{remark}

\section{Proof of Theorem~\ref{thm:inf_Risk}}\label{app:proof_inf_Risk}
To establish an upper bound for the approximation error $\inf_{\mathbf{p} \in \mathcal{F}} R(\mathbf{p}_0, \mathbf{p})$ in terms of the Kullback-Leibler (KL) divergence, we first provide two Lemmas that will be used in the proof of Theorem \ref{thm:inf_Risk}.
\begin{lemma}\label{thm:pk}
	Let $\mathbf{p}(\mathbf{Z}) = (p_1(\mathbf{Z}), \dots, p_K(\mathbf{Z}))^{\top}$ be a conditional class probability function. Suppose there exist constants $\alpha \ge 0$ and $C < \infty$ such that
	\[
	\mathbb{P}_{\mathbf{Z}}\bigl(p_k(\mathbf{Z})\le t\bigr) \le C\,t^\alpha,\quad \forall t\in[0,1],\, k\in\{1,\dots,K\}.
	\]
	Then for any $h \in (0,1]$, it holds that
	\[
	\int_{\{p_k(\mathbf{z})\ge h\}}\frac{1}{p_k(\mathbf{z})}\, d\mathbb{P}_{\mathbf{Z}}(\mathbf{z}) \;\le\;
	\begin{cases}
		\dfrac{C\,h^{\alpha-1}}{1-\alpha}, & 0\le \alpha <1,\\[6pt]
		C\bigl(1-\log h \bigr), & \alpha \ge 1.
	\end{cases}
	\]
\end{lemma}

\begin{proof}
Using the standard layer-cake representation, we have for a nonnegative measurable function $g$:
\[
\int g(\mathbf{z})\,d\mathbb{P}_{\mathbf{Z}}(\mathbf{z}) = \int_{0}^{\infty}\mathbb{P}_{\mathbf{Z}}\bigl(g(\mathbf{Z}) \ge u\bigr)\,du.
\]
Write \( g(\mathbf{z}) = \frac{1}{p_k(\mathbf{z})}\mathbb{I}( p_k(\mathbf{z}) \ge h) \), it follows that
\[
\int_{\{p_k(\mathbf{z})\ge h\}}\frac{1}{p_k(\mathbf{z})}\, d\mathbb{P}_{\mathbf{Z}}(\mathbf{z})
=\int_{0}^{\infty}\mathbb{P}_{\mathbf{Z}}\left(\frac{1}{p_k(\mathbf{Z})}\mathbb{I}( p_k(\mathbf{z}) \ge h) \ge u\right)\,du.
\]
Since
	\[
	\frac{1}{p_k(\mathbf{Z})} \mathbb{I}( p_k(\mathbf{z}) \ge h) \ge u >0
	\quad \Longrightarrow\quad
	h \leq p_k(\mathbf{Z})\le\frac{1}{u},
	\]
it follows that $u \in (0, 1/h]$ and 
	\[
	\int_{\{p_k(\mathbf{z})\ge h\}}\frac{1}{p_k(\mathbf{z})}\, d\mathbb{P}_{\mathbf{Z}}(\mathbf{z}) 
	\le \int_{0}^{1/h}\mathbb{P}_{\mathbf{Z}}\bigl(p_k(\mathbf{Z})\le\tfrac{1}{u}\bigr)\,du.
	\]
To complete the proof, we analyze two cases separately:
	
	\medskip
	\textbf{Case 1:} $0\le \alpha<1$. Note that $C \geq 1$. Then, for all $u \in (0, 1/h]$, we have
	\[
	\mathbb{P}_{\mathbf{Z}}\bigl(p_k(\mathbf{Z})\le\tfrac{1}{u}\bigr)\le C\,u^{-\alpha},
	\]
	and thus
	\[
	 \int_{\{p_k(\mathbf{z})\ge h\}}\frac{1}{p_k(\mathbf{z})}\, d\mathbb{P}_{\mathbf{Z}}(\mathbf{z}) \leq \int_{0}^{1/h}C\,u^{-\alpha}\,du = \frac{C\,h^{\alpha-1}}{1-\alpha}.
	\]
	
	\textbf{Case 2:} $\alpha \geq 1$. It follows that $\mathbb{P}_{\mathbf{Z}}(p_k(\mathbf{Z}) \leq t) \leq C t^{\alpha} \leq C t$ for all $t \in [0,1]$ and $C \geq 1$. Then we can bound the integral as follows:
	\begin{align*}
		\int_{\{p_k(\mathbf{z}) \geq h\}} \frac{d\mathbb{P}_{\mathbf{Z}}(\mathbf{z})}{p_k(\mathbf{z})}
		&\leq \int_0^{1/h} \mathbb{P}_{\mathbf{Z}}\left( p_k(\mathbf{Z}) \leq \frac{1}{u} \right) du \\
		&\leq \int_0^{1/h} \min\left( 1, \frac{C}{u} \right) du \\
		&= \int_0^C 1 \, du + \int_C^{1/h} \frac{C}{u} \, du \\
		&= C + C \left( \log \left( \frac{1}{h} \right) - \log C \right) \\
		&= C \left( 1 - \log h - \log C \right) \\
		&\leq  C (1 - \log h ),
	\end{align*}
	where the last inequality is due to $C \geq 1$.
	Combining both cases completes the proof.
\end{proof}

The next lemma provides the relationship between the KL divergence and the Chi-square divergence.
	
\begin{lemma}[KL-Chi-square Divergence Inequality]\label{KL-Chi-square}
Let $P$ and $Q$ be two probability measures satisfying $P \ll Q$ with the Radon-Nikodym derivative \( g(z) = \frac{dP}{dQ}(z) \).
Then, the following inequality holds:
		\[
		\mathrm{KL}(P \| Q) \leq \chi^2(P \| Q),
		\]
		where $ \mathrm{KL}(P \| Q) = \mathbb{E}_Q[g(z) \log g(z)]$ and $\chi^2(P \| Q) = \mathbb{E}_Q[(g(z)-1)^2]$.
	\end{lemma}
	
	\begin{proof}
		Recall the elementary inequality
		\[
		\log(x) \leq x - 1, \quad \text{for all } x > 0.
		\]
		Multiplying both sides by the non-negative function $g(z)$, we obtain
		\[
		g(z) \log(g(z)) \leq g(z)(g(z)-1).
		\]
		Expanding the right-hand side, we have
		\[
		g(z)(g(z)-1) = (g(z)-1)^2 + (g(z)-1).
		\]
		Taking expectation with respect to $Q$ on both sides yields
		\[
		\mathbb{E}_Q[g(z)\log(g(z))] \leq \mathbb{E}_Q[(g(z)-1)^2] + \mathbb{E}_Q[g(z)-1].
		\]
		Since $\mathbb{E}_Q[g(z)] = 1$, the second expectation vanishes. Consequently, 
		\[
		\mathbb{E}_Q[g(z)\log(g(z))] \leq \mathbb{E}_Q[(g(z)-1)^2].
		\]
		This completes the proof of Lemma~\ref{KL-Chi-square}.
	\end{proof}
	
	\begin{proof}[Proof of Theorem~\ref{thm:inf_Risk}] Recall that 
	$ R(\mathbf{p}_0, \mathbf{p}) = \mathbb{E}_{\mathbf{Z}}\left[\mathbf{p}_0(\mathbf{Z})^\top \log\frac{\mathbf{p}_0(\mathbf{Z})}{\mathbf{p}(\mathbf{Z})} \right]. $
	Applying Lemma \ref{KL-Chi-square} to our setting yields
	\begin{equation*}
		R(\mathbf{p}_0, \mathbf{p}) 
		= \mathbb{E}_{\mathbf{Z}} \left[KL(\mathbf{p}_0(\mathbf{Z}), \mathbf{p}(\mathbf{Z})) \right]
		\leq \mathbb{E}_{\mathbf{Z}}\left[\sum_{k=1}^{K}\frac{(p_k^0(\mathbf{Z}) - p_k(\mathbf{Z}))^2}{p_k(\mathbf{Z})}\right].
	\end{equation*}
     Recall that 
	\[
	\|\mathbf{p} - \mathbf{p}_0\|_{\infty} \leq 2(K+1)\varepsilon,
	\quad \text{and} \quad
	\min_{1 \le k \le K} \inf_{\mathbf{Z} \in [0,1]^d} p_k(\mathbf{Z}) \geq \frac{\varepsilon}{2K}.
	\]
	Therefore, for any $k \in \{1,\dots,K\}$ and any $\mathbf{Z}\in [0,1]^d$, it holds that
	\[
	p_k(\mathbf{Z}) \geq \max\left\{p_k^0(\mathbf{Z}) - 2(K+1)\varepsilon,\;\frac{\varepsilon}{2K}\right\}.
	\]
	Furthermore, if
	\[
	 p_k^0(\mathbf{Z}) - 2(K+1)\varepsilon \geq \frac{\varepsilon}{2K},
	\]
	then we have
	\[
	p_k^0(\mathbf{Z}) - 2(K+1)\varepsilon \geq p_k^0(\mathbf{Z})\left(1 - \frac{2(K+1)}{2(K+1) + \frac{1}{2K}}\right)  =\frac{p_k^0(\mathbf{Z})}{1 + 4K(K+1)}.
	\]
	This gives rise to the following upper bound:
	\begin{equation*}
		\begin{split}
			\frac{\left(p_k^0(\mathbf{Z}) - p_k(\mathbf{Z})\right)^2}{p_k(\mathbf{Z})}
        & \leq 8K(K+1)^2\varepsilon \cdot \mathbb{I}\left\{p_k^0(\mathbf{Z}) \leq 2(K+1)\varepsilon + \frac{\varepsilon}{2K} \right\} \\
        & \quad + \frac{4(K+1)^2\left[1+4K(K+1)\right]\varepsilon^2}{p_k^0(\mathbf{Z})}\cdot \mathbb{I}\left\{p_k^0(\mathbf{Z}) > 2(K+1)\varepsilon + \frac{\varepsilon}{2K} \right\}.
        \end{split}
	\end{equation*}
	Taking the expectation on both sides with respect to $\mathbf{Z}$, we obtain
    \begin{gather*}
    \mathbb{E}_{\mathbf{Z}}\left[\frac{\left(p_k^0(\mathbf{Z}) - p_k(\mathbf{Z})\right)^2}{p_k(\mathbf{Z})}\right]
    \leq 8K(K+1)^2\varepsilon\,\mathbb{P}\left\{p_k^0(\mathbf{Z}) \leq 2(K+1)\varepsilon + \frac{\varepsilon}{2K} \right\} \\
    \quad + 4(K+1)^2\left[1+4K(K+1)\right]\varepsilon^2\,\mathbb{E}_{\mathbf{Z}}\left[\frac{\mathbb{I}\left\{p_k^0(\mathbf{Z}) > 2(K+1)\varepsilon + \frac{\varepsilon}{2K} \right\}}{p_k^0(\mathbf{Z})}\right].
    \end{gather*}

By applying the $\alpha$-SVB condition to bound the first term and using Lemma~\ref{thm:pk} to bound the second term, we obtain
\begin{align*}
&\mathbb{E}_{\mathbf{Z}}\left[\frac{\left(p_k^0(\mathbf{Z}) - p_k(\mathbf{Z})\right)^2}{p_k(\mathbf{Z})}\right] \\
\leq & C \, \varepsilon^{1+(\alpha\wedge 1)} \, 8K(K+1)^2 \left(\frac{4K(K+1)+1}{2K}\right)^{\alpha\wedge 1} \left(2 + \frac{I(\alpha<1)}{1-\alpha} + \log\left(\frac{1}{\varepsilon}\right)\right),
\end{align*}
where $\alpha \wedge \beta = \min\{\alpha, \beta\}$. Also note that $0 \leq \alpha \leq 1$ and $ \frac{4K(K+1)+1}{2K} \leq 3(K+1)$.
	Consequently,
	\begin{align*} \label{eq:sum-bound}
    R(\mathbf{p}_0, \mathbf{p}) 
    &= \mathbb{E}_{\mathbf{Z}}\left[\sum_{k=1}^{K}\frac{\left(p_k^0(\mathbf{Z}) 
    -p_k(\mathbf{Z})\right)^2}{p_k(\mathbf{Z})}\right] \\
    & \leq C \varepsilon^{1+(\alpha\wedge 1)} K^{2}(K+1)^{2} 
    \left(\frac{4K(K+1)+1}{2K}\right)^{\alpha\wedge 1} 
    \left(2 + \frac{I(\alpha<1)}{1-\alpha} + \log\!\left(\frac{1}{\varepsilon}\right)\right)\\
    & \leq C \varepsilon^{1+(\alpha\wedge 1)} K^{2}(K+1)^{2 + (\alpha \wedge 1)}  
    \left(2 + \frac{I(\alpha<1)}{1-\alpha} +\log\left(\frac{1}{\varepsilon} \right)\right) \\
    &= C \varepsilon^{1+ \alpha} K^{2}(K+1)^{2 + \alpha} 
    \left(2 + \frac{I(\alpha<1)}{1-\alpha} +\log\left(\frac{1}{\varepsilon} \right)\right)\\
    &\leq C \varepsilon^{1+ \alpha} K^{4+\alpha} 
    \left(2 + \frac{I(\alpha<1)} {1-\alpha} +\log\left(\frac{1}{\varepsilon} \right)\right).
    \end{align*}

	Now choose 
	$ \varepsilon = C^{-1} T^{-\frac{\beta}{(1+\alpha)\beta + d}} \, K^{-\frac{(3+\alpha)\beta}{(1+\alpha)\beta + d}} $
	with $C > 0$ being a sufficiently small constant.
	This yields
    $$
    \inf_{\mathbf{p} \in \mathcal{F}} R(\mathbf{p}_0, \mathbf{p}) \lesssim T^{-\frac{(1+\alpha)\beta}{(1+\alpha)\beta + d}} K^{4+\alpha} \log(KT).
    $$
	This concludes the proof of Theorem~\ref{thm:inf_Risk}.
	\end{proof}

\section{Proof of Lemma~\ref{Covering_number}}\label{app:proof_covering_number}
\begin{proof}[Proof of Lemma~\ref{Covering_number}]
We first establish a useful Lipschitz continuity property of the log-softmax function. This intermediate result is critical for bounding covering numbers in Lemma~\ref{Covering_number}.
		
\begin{lemma}[Lipschitz Continuity of Log-Softmax]\label{lemma:log_soft}
Let \(\mathbf{y} = (y_1, \dots, y_K)^{\top} \in \mathbb{R}^K\), and define the softmax function \(\mathbf{\Phi} : \mathbb{R}^K \to \mathcal{S}^K \) by
\[
\mathbf{\Phi}(\mathbf{y}) := \frac{1}{\sum_{j=1}^K \exp(y_j)}
\begin{bmatrix}
 \exp(y_1) \\
 \vdots \\
 \exp(y_K)
\end{bmatrix},
\]
where \(\mathcal{S}^K\) is the probability simplex in \(\mathbb{R}^K\). Then \(\log \circ \mathbf{\Phi}\) is \(K\)-Lipschitz with respect to the $\| \cdot \|_{\infty}$-norm, i.e.,
\[
  \|\log(\mathbf{\Phi}(\mathbf{y}_1)) - \log(\mathbf{\Phi}(\mathbf{y}_2))\|_{\infty}
  \leq K \|\mathbf{y}_1 - \mathbf{y}_2\|_{\infty}, \quad \forall\, \mathbf{y}_1, \mathbf{y}_2 \in \mathbb{R}^K.
\]
\end{lemma}

\begin{proof}
Consider the log-softmax function
\( \log(\mathbf{\Phi}(\mathbf{y}))_i = y_i - \log \left(\sum_{j=1}^{K} e^{y_j} \right) \). 
Taking derivatives, we have for $i, k \in \{1,\dots,K\}$,
\[
\frac{\partial}{\partial y_k}\log(\mathbf{\Phi}(\mathbf{y}))_i = \delta_{ik} - \mathbf{\Phi}(\mathbf{y})_k,
\]
where $\delta_{ik}$ is the Kronecker delta. As $0 \leq \mathbf{\Phi}(\mathbf{y})_k \leq 1$, the Jacobian's entries $\delta_{ik} - \mathbf{\Phi}(\mathbf{y})_k $ are bounded by 1 in absolute value. Thus, the Lipschitz constant is at most $K$.
\end{proof}
		
		We now define an auxiliary function class excluding the softmax layer: 
		\begin{equation*}
			\mathcal{F}' = \{\mathbf{f}\circ\boldsymbol{\phi}(\mathbf{Z}) : \mathbf{f}\in\mathcal{F}_{id}(L,W,S),\; \boldsymbol{\phi}\in\mathcal{T}(M,D,H,W,S,B_1,B_2,\gamma)\}.
		\end{equation*}
		This allows us to represent the original function class conveniently as
		\[
		\log(\mathcal{F}) = \log(\mathbf{\Phi}(\mathcal{F}')).
		\]
		The next lemma establish the covering number bound of $\log \mathcal{F}$ with respect to the norm $\|\cdot\|_{\infty, \mathcal{Z}_T}$.
		
		\begin{lemma}\label{lem:deltaCover}
			Given a sample \(\mathcal{Z}_T = \{ \mathbf{Z}^{(1)}, \dots, \mathbf{Z}^{(T)} \} \subset [0,1]^d\), then for any \(\delta > 0\), we have
			\[
			\mathcal{N}\left(\delta, \log(\mathcal{F}), \|\cdot\|_{\infty, \mathcal{Z}_T}\right)
			\leq \mathcal{N}\left(\tfrac{\delta}{2K}, \mathcal{F}', \|\cdot\|_{\infty, \mathcal{Z}_T}\right).
			\]
		\end{lemma}
		
		\begin{proof}
			Fix \(\delta > 0\). Let \(\{\mathbf{f}_j\}_{j=1}^J\) be the centers of a minimal \(\delta/(2K)\)-cover of \(\mathcal{F}'\) with respect to \(\|\cdot\|_{\mathcal{Z}_T}\). That is, for every \(\mathbf{f} \in \mathcal{F}'\), there exists \(j \in \{1, \dots, J\}\) such that
			\[
			\|\mathbf{f} - \mathbf{f}_j\|_{\infty, \mathcal{Z}_T} \le \tfrac{\delta}{2K}.
			\]
			For each center \(\mathbf{f}_j\), choose \(\hat{\mathbf{f}}_j \in \mathcal{F}'\) such that
			\[
			\|\mathbf{f}_j - \hat{\mathbf{f}}_j\|_{\infty, \mathcal{Z}_T} \le \tfrac{\delta}{2K}.
			\]
			By the triangle inequality, for any \(\mathbf{f} \in \mathcal{F}'\), there exists \(\hat{\mathbf{f}}_j \in \mathcal{F}'\) such that
			\[
			\|\mathbf{f} - \hat{\mathbf{f}}_j\|_{\infty, \mathcal{Z}_T} \le \|\mathbf{f} - \mathbf{f}_j\|_{\infty, \mathcal{Z}_T} + \|\mathbf{f}_j - \hat{\mathbf{f}}_j\|_{\infty, \mathcal{Z}_T} \le \tfrac{\delta}{K}.
			\]
			Now consider an arbitrary \(\mathbf{g} \in \log(\mathcal{F})\), so that \(\mathbf{g} = \log\circ(\mathbf{\Phi} \circ \mathbf{f})\) for some \(\mathbf{f} \in \mathcal{F}'\). Then for the associated \(\hat{\mathbf{f}}_j\), we have
			\[
			\|\mathbf{g} - \log\circ(\mathbf{\Phi} \circ \hat{\mathbf{f}}_j)\|_{\infty, \mathcal{Z}_T}
			= \|\log\circ(\mathbf{\Phi} \circ \mathbf{f})-\log\circ(\mathbf{\Phi} \circ \hat{\mathbf{f}}_j)\|_{\infty, \mathcal{Z}_T}
			\le K \|\mathbf{f} - \hat{\mathbf{f}}_j\|_{\infty, \mathcal{Z}_T}
			\le \delta,
			\]
			where we used the \(K\)-Lipschitz continuity of \(\log \circ \mathbf{\Phi}\) (Lemma~\ref{lemma:log_soft}) with respect to \( \|\cdot\|_{\infty} \)-norm.
			Since \(\mathbf{g} \in \log(\mathcal{F})\) was arbitrary and each \(\hat{\mathbf{f}}_j \in \mathcal{F}'\), the set \(\{\log\circ(\mathbf{\Phi} \circ \hat{\mathbf{f}}_j)\}_{j=1}^J \subset \log(\mathcal{F})\) forms a \(\delta\)-cover of \(\log(\mathcal{F})\) under \(\|\cdot\|_{\infty, \mathcal{Z}_T}\). The claim of Lemma~\ref{lem:deltaCover} follows.
		\end{proof}
		
		\begin{lemma}[Empirical Covering Bound for Composed Class]\label{lem:composition_covering}
			Let $\mathcal{T}$ be a class of functions from $[0,1]^d$ to $\mathbb{R}^K$, and let $\mathcal{F}_{\mathrm{id}}$ be a class of functions from $\mathbb{R}^K$ to $\mathbb{R}^K$, such that every $\mathbf{f} \in \mathcal{F}_{\mathrm{id}}$ is $L_1$-Lipschitz with respect to $\|\cdot\|_{\infty}$. Define the composed function class
			\[
			\mathcal{F}' := \left\{ \mathbf{f} \circ \boldsymbol{\phi} : \mathbf{f} \in \mathcal{F}_{\mathrm{id}},\; \boldsymbol{\phi} \in \mathcal{T} \right\}.
			\]
			Then, for any empirical sample $\mathcal{Z}_T = \{ \mathbf{Z}^{(1)}, \ldots, \mathbf{Z}^{(T)} \} \subset [0,1]^d$, we have
			\[
			\mathcal{N}\left(\tfrac{\delta}{2K}, \mathcal{F}', \| \cdot \|_{\infty, \mathcal{Z}_T} \right)
			\leq 
			\mathcal{N}\left( \tfrac{\delta}{4K L_1}, \mathcal{T}, \| \cdot \|_{\infty, \mathcal{Z}_T} \right)
			\cdot 
			\mathcal{N}\left( \tfrac{\delta}{4K}, \mathcal{F}_{\mathrm{id}}, \| \cdot \|_{\infty} \right).
			\]
		\end{lemma}
		
		\begin{proof}
			Fix $\delta > 0$. Let $\widehat{\mathcal{T}}$ be a $\tfrac{\delta}{4K L_1}$-cover of $\mathcal{T}$ with respect to $\| \cdot \|_{\infty,\mathcal{Z}_T}$, and let $\widehat{\mathcal{F}}_{\mathrm{id}}$ be a $\tfrac{\delta}{4K}$-cover of $\mathcal{F}_{\mathrm{id}}$ with respect to $\| \cdot \|_{\infty}$. Define the composite approximation set
            \[
            \widehat{\mathcal{F}'} := \left\{ \hat{\mathbf{f}} \circ \hat{\boldsymbol{\phi}} : \hat{\mathbf{f}} \in \widehat{\mathcal{F}}_{\mathrm{id}},\; \hat{\boldsymbol{\phi}} \in \widehat{\mathcal{T}} \right\}.
            \]
            Next, we show that $\widehat{\mathcal{F}'}$ is a $\tfrac{\delta}{2K}$-cover of $\mathcal{F}'$ with respect to $\| \cdot \|_{\infty,\mathcal{Z}_T}$. Indeed, for any $\mathbf{f} \in \mathcal{F}_{\mathrm{id}}$ and $\boldsymbol{\phi} \in \mathcal{T}$, there exist $\hat{\mathbf{f}} \in \widehat{\mathcal{F}}_{\mathrm{id}}$ and $\hat{\boldsymbol{\phi}} \in \widehat{\mathcal{T}}$ such that
			Then for each $t \in \{1, \ldots, T\}$, we have
			\begin{align*}
				\| \mathbf{f}(\boldsymbol{\phi}(\mathbf{Z}^{(t)})) - \hat{\mathbf{f}}(\hat{\boldsymbol{\phi}}(\mathbf{Z}^{(t)})) \|_{\infty}
				&\leq \| \mathbf{f}(\boldsymbol{\phi}(\mathbf{Z}^{(t)})) - \mathbf{f}(\hat{\boldsymbol{\phi}}(\mathbf{Z}^{(t)})) \|_{\infty}
				+ \| \mathbf{f}(\hat{\boldsymbol{\phi}}(\mathbf{Z}^{(t)})) - \hat{\mathbf{f}}(\hat{\boldsymbol{\phi}}(\mathbf{Z}^{(t)})) \|_{\infty} \\
				&\leq L_1  \| \boldsymbol{\phi}(\mathbf{Z}^{(t)}) - \hat{\boldsymbol{\phi}}(\mathbf{Z}^{(t)}) \|_{\infty} + \tfrac{\delta}{4K} \\
				&\leq L_1 \tfrac{\delta}{4K L_1} + \tfrac{\delta}{4K} \\
                &= \tfrac{\delta}{2K}.
			\end{align*}
			Taking the maximum over $t = 1, \ldots, T$ yields
			\[
			\| \mathbf{f} \circ \boldsymbol{\phi} - \hat{\mathbf{f}} \circ \hat{\boldsymbol{\phi}} \|_{\infty, \mathcal{Z}_T} \leq \tfrac{\delta}{2K},
			\]
			showing that $\widehat{\mathcal{F}'}$ is indeed a $\tfrac{\delta}{2K}$-cover of $\mathcal{F}'$.
            Consequently, 
			\[
			\mathcal{N}\left(\tfrac{\delta}{2K}, \mathcal{F}', \| \cdot \|_{\infty, \mathcal{Z}_T} \right)
			\leq 
			|\widehat{\mathcal{T}}| \cdot |\widehat{\mathcal{F}}_{\mathrm{id}}|
			= 
			\mathcal{N}\left( \tfrac{\delta}{4K L_1}, \mathcal{T}, \| \cdot \|_{\infty, \mathcal{Z}_T} \right) 
			\cdot 
			\mathcal{N}\left( \tfrac{\delta}{4K}, \mathcal{F}_{\mathrm{id}}, \| \cdot \|_{\infty} \right).
			\]
		    This completes the proof of Lemma~\ref{lem:composition_covering}.	
		\end{proof}
		
		Finally, by Lemma~\ref{lem:deltaCover} and Lemma~\ref{lem:composition_covering}, we obtain 
		\begin{eqnarray*}
		 \mathcal{N}\left(\delta, \log(\mathcal{F}), \|\cdot\|_{\infty, 
		   \mathcal{Z}_T}\right)
		&\leq& \mathcal{N}\left(\tfrac{\delta}{2K}, \mathcal{F}', \|\cdot\|_{\infty, 
		   \mathcal{Z}_T}\right) \\
		&\leq&  \mathcal{N}\left( \tfrac{\delta}{4K L_1}, \mathcal{T}, \| \cdot 
		   \|_{\infty, \mathcal{Z}_T} \right) 
			\cdot \mathcal{N}\left( \tfrac{\delta}{4K}, \mathcal{F}_{\mathrm{id}}, \| \cdot \|_{\infty} \right).
		\end{eqnarray*}
		Applying the logarithm to both sides and taking the supremum over all empirical samples $\mathcal{Z}_T \subset [0,1]^d$, we obtain
		\[
		\mathcal{V}_\infty\left( \delta,\, \log \mathcal{F},\, T \right)
		\leq 
		\mathcal{V}_\infty\left( \tfrac{\delta}{4K L_1},\, \mathcal{T},\, T \right)
		+ \mathcal{V}\left( \tfrac{\delta}{4K},\, \mathcal{F}_{\mathrm{id}},\, \| \cdot \|_{\infty} \right).
		\]
		Thus, we complete the proof of Lemma~\ref{Covering_number} in the main text.
		\end{proof}

\section{Proof of Lemma~\ref{Covering_number_of_Transformer}}
\label{app:proof_transformer_covering}
Before presenting the detailed proof of Lemma~\ref{Covering_number_of_Transformer}, we first introduce the notion of pseudo-dimension which will be used in the proof of Lemma~\ref{Covering_number_of_Transformer}.

\begin{definition}{\bf (Pseudo-dimension)}\label{Pseudo-dimension}
Let $\mathcal{H}$ be a real-valued function class defined on $\mathbb{R}^d$. The  pseudo-dimension of $\mathcal{H}$, denoted by ${\rm Pdim}(\mathcal{H})$, is the largest integer $m$ for which there exist $\{\mathbf{z}_1, \dots, \mathbf{z}_m\} \subset \mathbb{R}^d $ and $\{y_1, \dots, y_m\} \subset \mathbb{R}$ 
such that 
$$ \left| \{(\operatorname{sgn}(g(\mathbf{z}_1) - y_1), \dots, \operatorname{sgn}(g(\mathbf{z}_m) - y_m)): g \in \mathcal{H} \} \right| = 2^m. $$
\end{definition}

\begin{proof}[Proof of Lemma \ref{Covering_number_of_Transformer}.]
Our proof for Lemma~\ref{Covering_number_of_Transformer} is similar to that of Lemma 30 in Jiao et al. (2024), Lemma 8 in \citet{gurevych2022rate}, and Theorem 6 in \citet{ bartlett2019nearly}. The main idea is to bound the pseudo-dimension of subsets of $\mathcal{T}$ and then utilize well-known results developed by \citet{ anthony2009neural} to bound the entropy by pseudo-dimensions. Here $\mathcal{T} = \mathcal{T}(M, H, D, W_1, B_1, S_1, \gamma)$ defined in~(\ref{eq:class_of_transformer}) in the main text. 

Recall that $\mathcal{T}$ consists of functions mapping from $\mathbb{R}^d$ to $\mathbb{R}^K$.
Our objective is to derive an upper bound on the pseudo-dimension of subsets of $\mathcal{T}$, in which the positions of the nonzero parameters are fixed. We first establish the bound for the spacial case $K=1$ and then generalize the result to arbitrary $K \geq 1$.
	
According to the definition of $\mathcal{T}$ with $K=1$, each function \( g \in \mathcal{T} \) has at most \( S_1 \) nonzero parameters. Fixing the positions of these nonzero parameters, we denote their values by \( \boldsymbol{\theta} \in \mathbb{R}^{S_1} \) and define the function class:
	\[
	\mathcal{G} = \{ g(\cdot, \boldsymbol{\theta}): \mathbb{R}^{d} \to \mathbb{R} \mid g \in \mathcal{T}, \boldsymbol{\theta} \in \mathbb{R}^{S_1} \},
	\]
To bound the pseudo-dimension of $\mathcal{G}$, consider a set of input-output pairs $\{(\mathbf{Z}^{(i)}, y_i)\}_{i=1}^{m}\subseteq \mathbb{R}^{d}\times\mathbb{R}$ such that
	\[
	\left|\left\{\left(\operatorname{sgn}(g(\mathbf{Z}^{(1)}, \boldsymbol{\theta})-y_1),\dots,\operatorname{sgn}(g(\mathbf{Z}^{(m)}, \boldsymbol{\theta})-y_m)\right): g\in \mathcal{G}\right\}\right| = 2^m.
	\]
Fixing the inputs $\{\mathbf{Z}^{(i)} \}_{i=1}^{m}$ and viewing the network parameters $\boldsymbol{\theta}$ as a set of $S_1$ real-valued variables, we aim to bound such an integer $m$, which gives an upper bound of $\operatorname{Pdim}(\mathcal{G})$.

Specifically, we define a sequence of partitions 
$ \mathcal{S}_0, \mathcal{S}_1, \mathcal{S}_2, \dots, \mathcal{S}_M, \mathcal{S}_{M+1} $
of the parameter domain $\mathbb{R}^{S_1}$, constructed iteratively through successive refinement, such that in the final partition, for any region \( P \in \mathcal{S}_{M+1} \), the functions
\[
  g(\mathbf{Z}^{(1)}, \boldsymbol{\theta}), \dots, g(\mathbf{Z}^{(m)}, \boldsymbol{\theta})
\]
are polynomials in \( \boldsymbol{\theta} \) of degree at most \( 10^{M+2} \) for all \( \boldsymbol{\theta} \in P \). 
Note that for the partition $\mathcal{S}_{M+1}$, we have the following bound
\begin{eqnarray}\label{m-bound}
	2^m 
	&=& \left|\left\{\left(\operatorname{sgn}(g(\mathbf{Z}^{(1)}, \boldsymbol{\theta})-y_1),\dots,\operatorname{sgn}(g(\mathbf{Z}^{(m)}, \boldsymbol{\theta})-y_m)\right): g\in \mathcal{G}\right\}\right| \nonumber \\
	&\leq& \sum_{P \in \mathcal{S}_{M+1}} \left| \left\{ \left(\operatorname{sgn}(g(\mathbf{Z}^{(1)}, \boldsymbol{\theta}) - y_1), \dots, \operatorname{sgn}(g(\mathbf{Z}^{(m)}, \boldsymbol{\theta}) - y_m) \right) : \boldsymbol{\theta} \in P \right\} \right|.
\end{eqnarray}
Lemma~\ref{lemma:gen_3} in the following ensures that each term in the sum of (\ref{m-bound}) can be effectively bounded, which provides an upper bound for $m$.

Next, we construct the sequence of partitions $ \mathcal{S}_0, \mathcal{S}_1, \mathcal{S}_2, \dots, \mathcal{S}_M, \mathcal{S}_{M+1} $ layer by layer. For the input layer, define \(\mathcal{S}_0 = \{\mathbb{R}^{S_1}\}\). 
Then all components of \( \mathbf{Z}_0 = \mathbf{Z} \) are polynomial as functions of \(\boldsymbol{\theta}\), each with degree $0$, for all \(\boldsymbol{\theta} \in \mathbb{R}^{S_1}\).
	
Consider \(r \in \{1, \ldots, M\}\). Suppose inductively that for each set \(S \in \mathcal{S}_{r-1}\), all components of \(Z_{r-1}\) are polynomial functions in \(\boldsymbol{\theta}\) of degree at most \(10^{r}\) for \(\boldsymbol{\theta} \in S\). Then any component of
	\[
	\boldsymbol{q}_{r-1, s, i} = \mathbf{W}_{Q, r, s}\boldsymbol{z}_{r-1, i}, \quad \boldsymbol{k}_{r-1, s, i} = \mathbf{W}_{K, r, s}\boldsymbol{z}_{r-1, i}
	\]
is a polynomial function in $\mathbf{\theta}$ of degree at most \(10^{r} + 1\) on each set \(S \in \mathcal{S}_{r-1}\).
Consequently, for each \(S \in \mathcal{S}_{r-1}\), the inner product
	\[
	\langle \boldsymbol{q}_{r-1, s, i}, \boldsymbol{k}_{r-1, s, j}\rangle
	\]
is a polynomial function in \(\boldsymbol{\theta}\) with degree at most \(2(10^{r} + 1) \). By applying Lemma \ref{lemma:gen_3}, the differences
	\[
	\left\{ \langle \boldsymbol{q}_{r-1, s, i}, \boldsymbol{k}_{r-1, s, j_1}\rangle - \langle \boldsymbol{q}_{r-1, s, i}, \boldsymbol{k}_{r-1, s, j_2}\rangle: s=1, \ldots, H, \ i, j_1, j_2 = 1, \ldots, N+1 \right\}
	\]
have at most
	\[
	\Delta_1 = 2 \left(\frac{2 e H (N+1)^3 (2 \cdot 10^{r} + 2)}{S_1}\right)^{S_1}
	\]
distinct sign patterns.     
	
	We refine each set in \(\mathcal{S}_{r-1}\) into at most \(\Delta_1\) subsets so that within each subset, these polynomial differences maintain consistent signs. Consequently, on each subset, every component of
	\[
	\mathbf{W}_{O, r, s} \left( \mathbf{W}_{V, r, s} \mathbf{Z}_{r-1} \right) \left[ (\mathbf{W}_{K, r, s} \mathbf{Z}_{r-1})^{\top}(\mathbf{W}_{Q, r, s} \mathbf{Z}_{r-1})
	\odot \sigma_H\left( (\mathbf{W}_{K, r, s} \mathbf{Z}_{r-1})^{\top}(\mathbf{W}_{Q, r, s} \mathbf{Z}_{r-1}) \right) \right]
	\]
	is a polynomial with degree at most \(3 \cdot 10^{r} + 4\). Therefore, all components of
	\[
	\mathbf{Y}_r = F^{(SA)}(\mathbf{Z}_{r-1})
	\]
	are polynomials with degree at most \(3 \cdot 10^{r} + 4\) within each refined region. 
	On each subset within the new partition, every component of
	\begin{align*}
		\mathbf{W}_{F_1} F^{(SA)}(\mathbf{Z}_{r-1})
	\end{align*}
	is represented by a polynomial with degree at most $3 \cdot 10^{r} + 5$. By invoking Lemma~\ref{lemma:gen_3} again, each subset of this partition can be further refined into
	\begin{align*}
		\Delta_2 = 2 \left(\frac{2 e W_1 \left(3 \cdot 10^{r} + 5\right)}{S_1}\right)^{S_1}
	\end{align*}
	subsets, ensuring all components in $ \mathbf{W}_{F_1} F^{(SA)}(\mathbf{Z}_{r-1}) $ maintain consistent sign patterns within this refined partition. The resulting partition, obtained by two successive refinements of $\mathcal{S}_{r-1}$, is denoted as $\mathcal{S}_r$. Since within each subset of $\mathcal{S}_r$, the signs of all components remain unchanged, we conclude that every component of
	\begin{align*}
		\sigma\left( \mathbf{W}_{F_1} F^{(SA)}(\mathbf{Z}_{r-1}) \right)
	\end{align*}
	is either identically zero or represented by a polynomial of degree at most $3 \cdot 10^{r} + 5$. Consequently, within each subset in $\mathcal{S}_r$, every component of the  output
	\begin{align*}
		\mathbf{Z}_r = F^{(\mathrm{FF})}\left( F^{(SA)}(\mathbf{Z}_{r-1}) \right) =  F^{(SA)}(\mathbf{Z}_{r-1}) + \mathbf{W}_{F_2} \sigma\left( \mathbf{W}_{F_1} F^{(SA)}(\mathbf{Z}_{r-1}) \right)
	\end{align*}
	is given by a polynomial whose degree at most 
	\begin{align*}
		3 \cdot 10^{r} + 6 \leq 10^{r+1}.
	\end{align*}
	Repeat the above steps, we obtain a partition $\mathcal{S}_{M}$ of $\mathbb{R}^{S_1}$ such that on each set $S \in \mathcal{S}_{M}$ all components of
	\begin{align*}
		\mathbf{Z}_M
	\end{align*}
	are polynomials as functions of $\boldsymbol{\theta} \in S$ of degree no more than $10^{M+1}$. Consequently, for all $k \in \{1, \ldots, m\}$
	\begin{align*}
		g(\mathbf{Z}^{(k)},  \boldsymbol{\theta}) - y_k
	\end{align*}
	are polynomials of degree at most $10^{M+1}+1 \leq 10^{M+2}$ in $\boldsymbol{\theta}$ for $\boldsymbol{\theta} \in S$.
	
	According to the above refinement process, we have
	\begin{align*}
		\left|\mathcal{S}_{M}\right| = \prod_{r=1}^M \frac{\left|\mathcal{S}_r\right|}{\left|\mathcal{S}_{r-1}\right|} 
		\leq \prod_{r=1}^M 2 \left(\frac{2e H (N+1)^3 \left(2 \cdot 10^{r} + 2\right)}{S_1}\right)^{S_1} 
		2 \left(\frac{2e W_1 \left(3 \cdot 10^{r} + 5\right)}{S_1}\right)^{S_1}.
	\end{align*}
    Note that
	\begin{align*}
		\begin{aligned}
			& \left|\left\{\left(\operatorname{sgn}\left(g(\mathbf{Z}^{(1)}, \boldsymbol{\theta}) - y_1\right), \ldots, \operatorname{sgn}\left(g(\mathbf{Z}^{(m)}, \boldsymbol{\theta}) - y_m\right)\right): g \in \mathcal{G} \right\}\right| \\
			& \leq \sum_{S \in \mathcal{S}_{M}} \left|\left\{\left(\operatorname{sgn}\left(g(\mathbf{Z}^{(1)}, \boldsymbol{\theta}) - y_1\right), \ldots, \operatorname{sgn}\left(g(\mathbf{Z}^{(m)}, \boldsymbol{\theta}) - y_m\right)\right): \boldsymbol{\theta} \in S \right\}\right|.
		\end{aligned}
	\end{align*}
    It follows from Lemma~\ref{lemma:gen_3} that
	\begin{align*}
		\begin{aligned}
			2^m & = \left|\left\{\left(\operatorname{sgn}\left(g(\mathbf{Z}^{(1)}, \boldsymbol{\theta}) - y_1\right), \ldots, \operatorname{sgn}\left(g(\mathbf{Z}^{(m)}, \boldsymbol{\theta}) - y_m\right)\right): g \in \mathcal{G} \right\}\right| \\
			& \leq 2 \left(\frac{2e m 10^{M+2}}{S_1}\right)^{S_1} |\mathcal{S}_{M}| \\
			& \leq 2 \left(\frac{2e m 10^{M+2}}{S_1}\right)^{S_1} 
			\prod_{r=1}^M 2 \left(\frac{2e H (N+1)^3 \left(2 \cdot 10^{r} + 2\right)}{S_1}\right)^{S_1} 
			2 \left(\frac{2e W_1 \left(3 \cdot 10^{r} + 5\right)}{S_1}\right)^{S_1} \\
			& \leq 2^{2M + 1} \left(\frac{m 6e (\max\{H(N+1), W_1\})^3 10^{M+2}}{(2M + 1) S_1}\right)^{(2M + 1) S_1}.
		\end{aligned}
	\end{align*}
	When $m \geq (2M + 1) S_1$, by Lemma~\ref{lemma:general-growth-bound} in the following, we have
	\begin{eqnarray*}
	m &\leq& (2M + 1) + (2M + 1) S_1 \log_2 \Big\{12e (\max\{H(N+1), W_1\})^3 
	         10^{M + 2} \\
	  &&     \times \log_2 \Big( 6e (\max\{H(N+1), W_1\})^3 10^{M+2} \Big) \Big\} \\
	  &\leq& C S_1 M^2 \log\big(\max\{H(N+1), W_1\} \big) \\ 
	  &\leq& C S_1 N M^2 \log\big(\max\{H, W_1\} \big), 
	\end{eqnarray*}
	where $C$ is an universal constant.
	When $m \leq (2M+1)S_1$, it is readily seen that $m \leq C S_1 N M^2 \log\big( \max\{H, W_1\} \big)$ also holds for some positive constant $C$. 
	This implies
	\[
	\operatorname{Pdim}(\mathcal{G}) \leq C S_1 N M^2 \log\big( \max\{H, W_1\} \big).
	\]
	Applying Lemma~\ref{lem:covering-number-pseudo-dimension}, the covering number of the function class \(\mathcal{G}\) with respect to $\|\cdot\|_{\infty, \mathcal{Z}_{T}}$-norm is bounded as
	\[
	\mathcal{N}(\delta_1, \mathcal{G}, \|\cdot\|_{\infty, \mathcal{Z}_T}) 
	\leq \left( \frac{e T B_1}{\delta_1 \operatorname{Pdim}(\mathcal{G})} \right)^{\operatorname{Pdim}(\mathcal{G})}, 
	\ \text{for} \ T \geq \operatorname{Pdim}(\mathcal{G}).
	\]
	
	Next, we consider the case $T \leq \operatorname{Pdim}(\mathcal{G})$. 
	Write $\mathcal{B}_{B_1} = \{ \mathbf{x} = (x_1, \dots, x_T)^{\top} \in \mathbb{R}^T: \|\mathbf{x}\|_{\infty} \leq B_1 \}$, where $\|\mathbf{x}\|_{\infty} = \max_{1 \leq i \leq T} |x_i|$. 
	Similar to the argument of Lemma 1.18 in \cite{rigollet2023high}, we have
	$$ \mathcal{N}(\delta_1, \mathcal{B}_{B_1}, \|\cdot\|_{\infty}) \leq 
	\left( \frac{3B_1}{\delta_1} \right)^T.
	$$
	Since $\sup_{\mathbf{Z}}|g(\mathbf{Z}, \mathbf{\theta})| \leq B_1$, it follows that
	\begin{equation*}
    \mathcal{N}(\delta_1, \mathcal{G}, \|\cdot\|_{\infty, \mathcal{Z}_T})
    \leq \mathcal{N}(\delta_1, \mathcal{B}_{B_1}, \|\cdot\|_{\infty})
    \leq \left( \tfrac{3B_1}{\delta_1} \right)^T
    \leq \left( \tfrac{3B_1}{\delta_1} \right)^{\mathrm{Pdim}(\mathcal{G})},
    \qquad \,T<\mathrm{Pdim}(\mathcal{G})\,.
    \end{equation*}
	Therefore, in all cases, we have
	$$ 
	\mathcal{N}(\delta_1, \mathcal{G}, \|\cdot\|_{\infty, \mathcal{Z}_T})
	\leq \left( \frac{3e T B_1}{\delta_1} \right)^{{\rm Pdim}(\mathcal{G})}.
	$$
	
	Note that the functions in \(\mathcal{T}\) depend on at most
	$ C M H^2 \left(\max\{D, W_1, N+1\}\right)^3 $
	parameters, of which at most \(S_1\) parameters can be nonzero. Thus, the number of ways to choose these parameters is bounded by
	\[
	\binom{C M H^2 \left(\max\{D, W_1, N+1\}\right)^3 }{S_1}
	\leq \left( C M H^2 \left(\max\{D, W_1, N+1\}\right)^3 \right)^{S_1}.
	\]
	Consequently,
	$$ 
	\mathcal{N}(\delta_1, \mathcal{T}, \|\cdot\|_{\infty, \mathcal{Z}_T})
	\leq \left( C M H^2 \left(\max\{D, W_1, N+1\}\right)^3 \right)^{S_1} \left( \frac{3e T B_1}{\delta_1} \right)^{{\rm Pdim}(\mathcal{G})}.
	$$
	Recall that $ \operatorname{Pdim}(\mathcal{G}) \leq C S_1 N M^2 \log\big( \max\{H, W_1\} \big)  $. It follows that 
	\begin{eqnarray*}
	&&    \log \mathcal{N}(\delta_1, \mathcal{T}, \|\cdot\|_{\infty, \mathcal{Z}_T}) \\
	&\leq& {\rm Pdim}(\mathcal{G}) \log\left( \frac{3e T B_1}{\delta_1} \right) 
	       + S_1 \log\left(  C M H^2 \left(\max\{D, W_1, N+1\}\right)^3 \right) \\
	&\leq& C S_1 N M^2 \log(\max\{H, W_1\}) \log\left( \frac{3 e T B_1}{\delta_1} \right)
	       + S_1 \log \left( C M H^2 (\max\{D, W_1, N+1\})^3 \right) \\
	&\lesssim& S_1 M^2 \log(\max\{M, H, W_1\}) 
		\log\left( \frac{T B_1}{\delta_1} \right).
    \end{eqnarray*}
	Consequently,
	$$
    \mathcal{V}_{\infty}(\delta_1, \mathcal{T}, T) \lesssim S_1 M^2 \log(\max\{M, H, W_1\}) \log\left( \frac{T B_1}{\delta_1} \right).
    $$
	Then we establish the upper bound of $\mathcal{V}_{\infty}(\delta_1, \mathcal{T}, T) $ for the special case $K=1$. 
	
	For arbitrary $K \geq 1$, note that the covering number of \(\mathcal{T}\) under the $ \|\cdot\|_{\infty, \mathcal{Z}_T} $-norm can be bounded by combining the covering numbers of the \(K\) scalar-valued function classes. Thus, we obtain
	\[
    \mathcal{V}_{\infty}(\delta_1, \mathcal{T}, T) 
    \lesssim K S_1 M^2 \log(\max\{M, H, W_1\}) \log\left(\frac{T B_1}{\delta_1}\right).
    \]
	This complete the proof of Lemma~\ref{Covering_number_of_Transformer}.
	\end{proof}

    In the end of this subsection, we present some lemmas that are used in the proof of Lemma~\ref{Covering_number_of_Transformer}.
    
    \begin{lemma}[Theorem 8.3 of \citet{anthony2009neural}]
	\label{lemma:gen_3}
	Suppose \(W \leq m\), and consider polynomials \(f_1, \ldots, f_m\) of degree at most \(D\) in \(W\) variables. Define the quantity
	\[
	K := \left| \left\{ \left(\operatorname{sgn} \left(f_1(\boldsymbol{x})\right), \dots, \operatorname{sgn} \left(f_m(\boldsymbol{x})\right) \right) : \boldsymbol{x} \in \mathbb{R}^W \right\} \right|.
	\]
	Then the following bound holds:
	\[
	K \leq 2 \left( \frac{2 e m D}{W} \right)^W.
	\]
    \end{lemma}

    \begin{lemma}[Lemma 16 of \citet{bartlett2019nearly}]
	\label{lemma:general-growth-bound}
	Suppose that \(2^m \leq 2^L \left(\frac{m R}{w}\right)^w\) for some \(R \geq 16\) and \(m \geq w \geq L \geq 0\). Then
	\[
	m \leq L + w \log_2 \left(2 R \log_2 R\right).
	\]
    \end{lemma}

    \begin{lemma}[Theorem 12.2 of \citet{anthony2009neural}]
    \label{lem:covering-number-pseudo-dimension}
    Let $\mathcal{F}$ denote a class of real-valued functions on a domain $\mathcal{X}$,
    with range in $[0,B]$. Let $\epsilon > 0$, and assume $\mathcal{F}$ has pseudo-dimension $d$.
    Then
    $$
    \mathcal{N}_\infty(\epsilon, \mathcal{F}, m) \leq \sum_{i=1}^{d} \binom{m}{i} \left(\frac{B}{\epsilon}\right)^i,
    $$
    which is further bounded by $\left(\frac{emB}{\epsilon d}\right)^d$ for $m \geq d$.
    \end{lemma}

\section{Proof of Lemma~\ref{Covering_number_of_MLP}}
\label{app:proof_mlp_covering}

\begin{proof}[Proof of Lemma~\ref{Covering_number_of_MLP}]
	For clarity in our proof, we restate the functional form of networks, i.e., any network $\mathbf{f} \in \mathcal{F}_{id}(L, W, S)$ is rewritten as
	\[
	\mathbf{f}(\mathbf{x}) = \left( W^{(L)}\sigma(\cdot) + b^{(L)} \right) \circ \dots \circ \left( W^{(1)}\mathbf{x} + b^{(1)} \right),
	\]
	where $\sigma(\cdot)$ denotes the ReLU activation function.
	Define the intermediate mappings
	\begin{eqnarray*}
	\mathcal{A}_k(\mathbf{f})(\mathbf{x}) 
	&=& \sigma \circ \left( W^{(k-1)} \sigma(\cdot)+b^{(k-1)} \right) \circ \dots 
	    \circ \left( W^{(1)}\mathbf{x} + b^{(1)} \right)  \quad \text{for } k=2, \dots, L, \\
	\mathcal{B}_k(\mathbf{f})(\mathbf{x}) 
	&=&  \left( W^{(L)}\sigma(\cdot)+b^{(L)} \right) \circ \dots \circ \left( W^{(k)}
	     \sigma(\mathbf{x}) +b^{(k)} \right) \quad \text{for } k=1, \dots, L.
	\end{eqnarray*}
	For completeness, set $\mathcal{A}_1(\mathbf{f})(\mathbf{x})=\mathbf{x}$ and $\mathcal{B}_{L+1}(\mathbf{f})(\mathbf{x})=\mathbf{x}$. Then we can express
	\[
	\mathbf{f}(\mathbf{x})=\mathcal{B}_{k+1}(\mathbf{f})\circ(W^{(k)}\cdot+b^{(k)})\circ\mathcal{A}_k(\mathbf{f})(\mathbf{x}), \quad \forall \ k=1, \dots, L.
	\]
	Let $\mathbf{f},\mathbf{g} \in \mathcal{F}_{id}(L,W,S)$ be two networks given by 
	\begin{eqnarray*}
	\mathbf{f}(\mathbf{x}) 
	&=& \left( W^{(L)}\sigma(\cdot) + b^{(L)} \right) \circ \dots \circ \left( 
	    W^{(1)}\mathbf{x} + b^{(1)} \right)  \\
	\mathbf{g}(\mathbf{x}) 
	&=& \left( W^{'(L)}\sigma(\cdot) + b^{'(L)} \right) \circ \dots \circ \left( 
	    W^{'(1)}\mathbf{x} + b^{'(1)} \right)
	\end{eqnarray*}
	with parameters satisfying $ \| W^{(k)}-{W^{'(k)}} \|_\infty \leq \delta_2 $ and $ \| b^{(k)}-{b^{'(k)}} \|_\infty \leq \delta_2$ 
	for all $1 \leq k \leq L$. 
	Note that 
	$$ \mathcal{B}_{k+1}(\mathbf{g})\circ(W^{(k)}\cdot+b^{(k)})\circ
	   \mathcal{A}_k(\mathbf{f})(\mathbf{x}) 
	  = \mathcal{B}_{k+2}(\mathbf{g})\circ({W^{'(k+1)}} \cdot+{b^{'(k+1)}})\circ\mathcal{A}_k(\mathbf{f})(\mathbf{x}),
	$$
	for $k = 1, \dots, L-1$. 
	This yields 
	\begin{eqnarray*}
	&&  \|\mathbf{f}(\mathbf{x})-\mathbf{g}(\mathbf{x})\|_\infty \\
	&=& \left\| \sum_{k=1}^{L} \mathcal{B}_{k+1}(\mathbf{g}) \circ 
	  (W^{(k)}\cdot+b^{(k)}) \circ \mathcal{A}_k(\mathbf{f}) (\mathbf{x}) -\mathcal{B}_{k+1} 
        (\mathbf{g}) \circ ({W^{'(k)}} \cdot +{b^{'(k)}}) \circ \mathcal{A}_k(\mathbf{f}) (\mathbf{x}) \right \|_\infty \\
	&\leq& \sum_{k=1}^{L} \| \mathcal{B}_{k+1}(\mathbf{g})\|_{Lip} \cdot \| (W^{(k)} \cdot   
        + b^{(k)}) \circ \mathcal{A}_k(\mathbf{f})(\mathbf{x}) - ({W^{'(k)}} \cdot 
	  + {b^{'(k)}}) \circ \mathcal{A}_k(\mathbf{f})(\mathbf{x}) \|_\infty,
	\end{eqnarray*}
    where $\| \mathcal{B}_{k+1}(\mathbf{g}) \|_{Lip}$ denotes the corresponding Lipschitz constant of $\mathcal{B}_{k+1}(\mathbf{g})$.
    Next, we explicitly derive the following uniform bound for each of the intermediate mappings. According to the  definition of $\mathcal{F}_{\rm id}(L, W, S)$ provided in (\ref{eq:MLP_class}), we obtain

	\begin{eqnarray*}
	  \|\mathcal{A}_k(\mathbf{f})\|_\infty 
	  &\leq& \max_j \|W^{(k-1)}_{j,:}\|_1\|\mathcal{A}_{k-1}(\mathbf{f})\|_\infty + \|b^{(k-1)}\|_\infty  \\
	  &\leq& W\|\mathcal{A}_{k-1}(\mathbf{f})\|_\infty + 1,
	\end{eqnarray*}
	where $W^{(k-1)}_{j,:}$ is the $j$-row of the matrix $ W^{(k-1)}$. Consequently,
    $$ 
	 \|\mathcal{A}_k(\mathbf{f})\|_\infty \leq (W+1)^{k-1}.
	$$
	Similarly, we can show that the Lipschitz constant of $\mathcal{B}_{k+1}(\mathbf{g})$ is bounded by
	\[
	\|\mathcal{B}_{k+1}(\mathbf{g})\|_{Lip}\leq W^{L-k}.
	\]
	Using these bounds, we obtain 
	\begin{align*}
	\|\mathbf{f}(\mathbf{y})-\mathbf{g}(\mathbf{y})\|_\infty
	\leq \sum_{k=1}^{L}W^{L-k}[W(W+1)^{k-1}+1] \delta_2 
	\leq L(W+1)^L \delta_2.
	\end{align*}
	Similar to the arguments for Lemma~\ref{Covering_number_of_Transformer} in Appendix~\ref{app:proof_transformer_covering}, we fix  the positions of these nonzero parameters of the network $F_{id}(L, W, S)$ and then its covering number with respect to the $\| \cdot\|_{\infty}$-norm is bounded by
	\[
	\left(\frac{3 \delta_2}{L(W+1)^L}\right)^{-S}.
	\]
	Recall that \(\mathcal{F}_{id}(L, W, S)\) depend on at most	$ (W+1)^L $ parameters, of which at most \(S\) parameters can be nonzero.
	Thus, the number of different chosen patterns of these nonzero parameters is $\binom{(W+1)^L}{S} \leq (W+1)^{LS}$.
	Therefore, we obtain
	\begin{align*}
	\mathcal{N}(\delta_2, \mathcal{F}_{id}(L,W,S), \|\cdot\|_\infty)
	\leq (W+1)^{LS}\left(\frac{3\delta_2}{L(W+1)^L}\right)^{-S} 
	=\left( \frac{L(W+1)^{2L}}{ 3\delta_2} \right)^S.
	\end{align*}
	Taking logarithms, we have
	$$ \log \mathcal{N}(\delta_2, \mathcal{F}_{id}(L,W,S), \|\cdot\|_\infty)
	\leq 2 LS \log \left( \frac{L(W+1)}{3\delta_2} \right).
	$$
	This completes the proof of Lemma~\ref{Covering_number_of_MLP}.
\end{proof}

\section{Proof of Lemma~\ref{lem:hellinger-klb-equivalence}}\label{app:hellinger-klb}
The Kullback-Leibler divergence and Hellinger distance are two classical measures of discrepancy between probability distributions; see \citet[Section~2.4]{tsybakov2008introduction} and \citet[Lemma 5]{birge1998minimum}. These two measures are closely related and satisfy several useful inequalities, which are fundamental in nonparametric estimation theory.
Recall that the squared Hellinger distance between \(P\) and \(Q\) is defined as
\(  H^2(P, Q) = 1 - \int \sqrt{dP\,dQ} \)
\citep[Section~7.3]{birge1998minimum}. 
We proceed to establish the inequality stated in Lemma~\ref{lem:hellinger-klb-equivalence}.

\begin{proof}{[Proof of Lemma~\ref{lem:hellinger-klb-equivalence}]}
We first show that $H^2(P, Q) \leq \frac{1}{2} \mathrm{KL}_2(P \| Q) \leq \frac{1}{2} \mathrm{KL}_B(P \| Q)$. As the second inequality follows from the definition of the truncated KL-divergence, we only show $H^2(P, Q) \leq \frac{1}{2} \mathrm{KL}_2(P \| Q)$. 
Let \(P = P^a + P^s\) be the Lebesgue decomposition of \(P\) with respect to \(Q\), where \(P^a \ll Q\) and \(P^s \perp Q\). Let \(\mu = \frac{1}{2}(P + Q)\) be the dominating measure and let \(p, p^a, p^s, q\) denote the corresponding densities of \(P, P^a, P^s, Q\) with respect to \(\mu\), respectively. Since \(p^s q = 0\), it follows that 
\[ H^2(P, Q) = 1- \int \sqrt{pq} d\mu = \int (p^a + p^s - \sqrt{p^a q}) \, d\mu.
\]
Since $-x \leq - \log(1+x)$ for $x > -1$, it follows that
\begin{eqnarray*}
\mathrm{KL}_2(P \| Q) 
&=& \int \left( 2 \wedge \log \frac{dP}{dQ} \right) dP  \\
&=& \int \left( 2 \wedge 2\log\sqrt{p/q}  \right) p d\mu \\
&=& 2 \int \left[ 1 \wedge \left( -\log \left( 1+\sqrt{p/q} -1 \right) 
    \right) \right] pd\mu \\
&\geq& 2 \int \left[ 1 \wedge \left(1-\sqrt{q/p} \right) \right] pd\mu \\
&=& 2 \int (p - \sqrt{pq}) d\mu \\
&=& 2H^2(P, Q).
\end{eqnarray*}
Therefore, we obtain $H^2(P, Q) \leq \frac{1}{2} KL_2(P \| Q)$.

	The proof of the upper bound on \(\mathrm{KL}_B(P \| Q)\) in terms of \(H^2(P, Q)\) follows a similar approach to that of Theorem 5 of \citet{wong1995probability} and Lemma 3.4 of \citet{bos2022convergence}. For completeness, we also provide a detailed derivation.
	Note that by second order Taylor expansion, we have 
	$ x \log(x) \leq x-1 +\frac{1}{2}(x-1)^2/(x \wedge 1) $ for $x >0$. 
	Consequently, 
	\[
	x \log x = 2\sqrt{x} [\sqrt{x}\log(\sqrt{x})] \leq 2(x - \sqrt{x}) + ( \sqrt{x} \vee 1) (\sqrt{x} - 1)^2, \quad \forall \ x > 0.
	\]
	Applying this inequality to \(x = p^a/q \), we obtain that on the set \(\{p^a/q \leq e^B\}\),
	\[
	\frac{p^a}{q} \log\left( \frac{p^a}{q} \right) \leq 2\left( \frac{p^a}{q} - \sqrt{\frac{p^a}{q}} \right) + e^{B/2} \left( \sqrt{\frac{p^a}{q}} - 1 \right)^2.
	\]
	This implies
	\begin{eqnarray*}
	\mathrm{KL}_B(P \| Q) 
	&=&    \int_{p^a/q \leq e^B} \frac{p^a}{q} \log\left( \frac{p^a}{q} \right) q d\mu + 
	       B \int_{\frac{dP}{dQ} > e^B} dP \\
	&\leq& 2 \int_{p^a/q \leq e^B} (p^a - \sqrt{p^a q}) d\mu + e^{B/2} \int_{p^a/q 
	       \leq e^B} (\sqrt{p^a} - \sqrt{q})^2 d\mu + B\int_{\frac{dP}{dQ} > e^B} dP \\
	&\leq& 2 \int_{p^a/q \leq e^B}(p^a - \sqrt{p^a q}) d\mu + 2 e^{B/2} H^2(P, Q) + B 
	       \int_{\frac{dP}{dQ} > e^B} dP,
	\end{eqnarray*}
	where the last inequality is due to $p^sq =0$ and $H^2(P, Q) = \int (1 - \sqrt{p q}) d\mu$.
	Next, we consider two cases to complete the proof. 
	
    {\it Case 1}: \(\int_{p^a/q \leq e^B} (p^a - \sqrt{p^a q}) d\mu \leq 0\). 
    Note that
		\[
		H^2(P, Q) = \frac{1}{2} \int (\sqrt{p} - \sqrt{q})^2 d\mu 
		\geq \frac{1}{2} \int_{p/q \geq e^B} p(1 - e^{-B/2})^2 d\mu 
		= \frac{1}{2} (1 - e^{-B/2})^2 \int_{\frac{dP}{dQ} > e^B} dP.
		\]
		Consequently, 
		\[
		\mathrm{KL}_B(P \| Q) \leq 2\left(e^{B/2} + B(1 - e^{-B/2})^{-2} \right) H^2(P, Q).
		\]

   {\it Case 2}: \(\int_{p^a/q \leq e^B} (p^a - \sqrt{p^a q}) d\mu > 0\). Some elementary calculations show that 
    $$ B \int_{\frac{dP}{dQ} > e^B} dP \leq B(1 - e^{-B/2})^{-1} \int_{p^a/q > e^B} (p-\sqrt{pq})d\mu. $$
   It follows that 
	\begin{align*}
	  \mathrm{KL}_B(P \| Q)
	  \leq&\  2 \int_{p^a/q \leq e^B}(p^a - \sqrt{p^a q}) d\mu + 2 e^{B/2} H^2(P, Q) + B 
	       \int_{\frac{dP}{dQ} > e^B} dP \\
	  \leq&\  2e^{B/2} H^2(P, Q) +  B(1 - e^{-B/2})^{-1} \int_{p^a/q \leq e^B}(p-\sqrt{p q}) d\mu\\
	  & \ + B(1 - e^{-B/2})^{-1} \int_{p^a/q > e^B} (p-\sqrt{pq})d\mu \\
	  =&\  2e^{B/2} H^2(P, Q) +  B(1 - e^{-B/2})^{-1} \int (p - \sqrt{p q}) d\mu \\
	  =&\ \left( 2e^{B/2} + B(1 - e^{-B/2})^{-1} \right) H^2(P, Q).
	\end{align*}
		
	In both cases, since \(B \geq 2\), it follows that 
	$$	B(1 - e^{-B/2})^{-1} \leq 2e^{B/2},\quad
	B(1 - e^{-B/2})^{-2} \leq 4e^{B/2}.
	$$
	Therefore, we obtain 
	$ \mathrm{KL}_B(P \| Q) \leq 10 e^{B/2} H^2(P, Q) $.
	Consequently,
	\[ \frac{1}{2} \mathrm{KL}_B(P \| Q) \leq 5 e^{B/2} H^2(P, Q). \]
	This completes the proof of Lemma~\ref{lem:hellinger-klb-equivalence}.
\end{proof}

\section{Proof of Lemma~\ref{lem:covering_number_bound}}\label{app:proof_covering_number_bound}
\begin{proof}[Proof of Lemma~\ref{lem:covering_number_bound}]
Recall that $ \mathcal{F}_{\rm mlp}(L,W,S) = \mathbf{\Phi}( \mathcal{F}_{\mathrm{id}}(L,W,S)) $, it follows from Lemma~\ref{lem:deltaCover} that
\[
\mathcal{N}\left( \frac{1}{T}, \log\left(\mathcal{F}_{\mathrm{mlp}}(L, W, S)\right), \|\cdot\|_{\infty, \mathcal{Z}_T} \right)
\leq \mathcal{N}\left( \frac{1}{2KT}, \mathcal{F}_{\mathrm{id}}(L, W, S), \|\cdot\|_{\infty, \mathcal{Z}_T} \right).
\]
Consider the function classes
\begin{eqnarray*}
{\mathcal{F}^k_{\mathrm{id}}(L,W,S)}^{+} 
&=& \{ \sigma(\mathbf{e}_k^{\top} \mathbf{G}): \mathbf{G} = (G_1, \dots, G_K)^{\top} \in 
    \mathcal{F}_{\mathrm{id}}(L,W,S) \}, \\
{\mathcal{F}^k_{\mathrm{id}}(L,W,S)}^{-} 
&=& \{ \sigma(-\mathbf{e}_k^{\top} \mathbf{G}): \mathbf{G} = (G_1, \dots, G_K)^{\top} \in 
    \mathcal{F}_{\mathrm{id}}(L,W,S) \},
\end{eqnarray*}
where $\sigma(x) = \max(x, 0)$ is the ReLU activation function and $\mathbf{e}_k=(0, \dots, 0, 1, 0, \dots, 0)^{\top} \in \mathbb{R}^K$ with $1$ lying in the $k$-th component of $\mathbf{e}_k$.  
It is readily seen that both ${\mathcal{F}^k_{\mathrm{id}}(L,W,S)}^{+}$ and ${\mathcal{F}^k_{\mathrm{id}}(L,W,S)}^{-}$ consist of ReLU networks with depth $L+1$, width $W$, and size $S+1$. 
Since $ \| \mathbf{G} \|_{\infty} \leq B_2 $ for any $\mathbf{G} \in  \mathcal{F}_{\mathrm{id}}(L,W,S)$, it follows that $|\sigma(\mathbf{e}_k^{\top} \mathbf{G})| \leq B_2$ and $| \sigma(-\mathbf{e}_k^{\top} \mathbf{G})| \leq B_2$. 
Therefore, by Lemma~\ref{lem:covering-number-pseudo-dimension} (Theorem~12.2 of~\citet{anthony2009neural}), we obtain
\begin{eqnarray*}
\mathcal{N} \left( \frac{1}{4KT}, {\mathcal{F}^k_{\mathrm{id}}(L, W, S)}^{+}, \|\cdot\|_{\infty, \mathcal{Z}_T} \right)
&\leq& \left( \frac{4eK B_2 T^2}{\mathrm{Pdim}( {\mathcal{F}^k_{\mathrm{id}}(L, W, S)}^{+})} 
       \right)^{\mathrm{Pdim}( {\mathcal{F}^k_{\mathrm{id}}(L, W, S)}^{+})},  
\end{eqnarray*}	
Furthermore, Theorems~3 and 6 of~\citet{bartlett2019nearly} imply that there exist universal constants \(\underline{c}, \overline{c} > 0\) such that
\[
\underline{c} (S+1)(L+1) \log \frac{S+1}{L+1} \leq \mathrm{Pdim}( {\mathcal{F}^k_{\mathrm{id}}(L, W, S)}^{+}) \leq \overline{c} (S+1)(L+1) \log(S+1).
\]
Consequently, 
$$
\mathcal{N} \left( \frac{1}{4KT}, {\mathcal{F}^k_{\mathrm{id}}(L, W, S)}^{+}, \|\cdot\|_{\infty, \mathcal{Z}_T} \right)
\leq \left( \frac{4eK B_2 T^2}{\underline{c} (S+1)(L+1) \log \frac{S+1}{L+1}} 
     \right)^{\overline{c} (S+1)(L+1) \log(S+1)},
$$
Similarly, we have 
$$
\mathcal{N} \left( \frac{1}{4KT}, {\mathcal{F}^k_{\mathrm{id}}(L, W, S)}^{-}, \|\cdot\|_{\infty, \mathcal{Z}_T} \right)
\leq \left( \frac{4eK B_2 T^2}{\underline{c} (S+1)(L+1) \log \frac{S+1}{L+1}} 
     \right)^{\overline{c} (S+1)(L+1) \log(S+1)},
$$
Also note that $G_k = \sigma(\mathbf{e}_k^{\top} \mathbf{G}) - \sigma(-\mathbf{e}_k^{\top} \mathbf{G})$, it follows that 
$$
\mathcal{N} \left( \frac{1}{2KT}, {\mathcal{F}_{\mathrm{id}}(L, W, S)}, \|\cdot\|_{\infty, \mathcal{Z}_T} \right)
\leq \left( \frac{4eK B_2 T^2}{\underline{c} (S+1)(L+1) \log \frac{S+1}{L+1}} 
     \right)^{2\overline{c} K (S+1)(L+1) \log(S+1)}.
$$
Taking the logarithm and the supremum over all $\mathcal{Z}_T \subset [0,1]^d$, we conclude that
$$
\mathcal{V}_\infty\left( \frac{1}{T}, \log\left(\mathcal{F}_{\mathrm{mlp}}(L, W, S)\right), T \right)
\leq C K S L \log(S) \log(B_2 KT),
$$
for some universal constant $C > 0$.
\end{proof}
\vskip 0.2in

\bibliography{references}

\end{document}